%% file: main.tex
\documentclass{article}
\usepackage[utf8]{inputenc}
\usepackage{amsthm,amsmath,amsfonts,amssymb,bm,dsfont,enumitem,natbib,graphicx}
\usepackage{algorithm,algorithmic}
\usepackage[usenames]{color}
\usepackage{geometry}
\usepackage[colorlinks,
            linkcolor=red,
            anchorcolor=blue,
            citecolor=blue
            ]{hyperref}
\usepackage{mystyle}

\input{mynotation}
\def\MWLE{{\sf MWLE}}

\newcommand{\indep}{\perp \!\!\! \perp}

\title{\textbf{Maximum Likelihood Estimation is All You Need for Well-Specified Covariate Shift}}

\author
{\normalsize
Jiawei Ge\thanks{equal contribution}~\thanks{Department of Operations Research and Financial Engineering, Princeton University; 
\texttt{\{jg5300,shangetang,\newline jqfan\}@princeton.edu}}
\qquad Shange Tang \footnotemark[1]~\footnotemark[2]
\qquad Jianqing Fan\footnotemark[2]
\qquad Cong Ma\thanks{Department of Statistics, University of Chicago;
\texttt{congm@uchicago.edu}}
\qquad Chi Jin\thanks{Department of Electrical and Computer Engineering, Princeton University; 
\texttt{chij@princeton.edu}}
}

\date{}

\begin{document}

\maketitle

\input{abstract}
\input{introduction}

\input{setup}

\input{wellspecified}
\input{application}
\input{misspecified}

\input{conclusion}

\newpage
\bibliography{reference}
\bibliographystyle{iclr2024_conference}

\newpage
\appendix
\input{pf_wellspecified}

\input{pf_application}
\input{pf_misspecified}
\input{pf_auxiliary}

\end{document}

%% file: mynotation.tex
\def\argmin{{\arg\min}}

\def\bbB{\mathbb{B}}

\def\bbE{\mathbb{E}}

\def\bbI{\mathbb{I}}

\def\bbP{\mathbb{P}}

\def\bbR{\mathbb{R}}

\def\cD{\mathcal{D}}

\def\cF{\mathcal{F}}

\def\cI{\mathcal{I}}

\def\cM{\mathcal{M}}
\def\cN{\mathcal{N}}
\def\cO{\mathcal{O}}

\def\cS{\mathcal{S}}

\def\cX{\mathcal{X}}
\def\cY{\mathcal{Y}}

\def\diag{{\sf diag}}

\def\exp{{\sf exp}}

\def\rank{{\sf rank}}

\def\MLE{{\sf MLE}}
\def\Tr{{\sf Tr}}

\def\uni{{\sf Uniform}}
\def\vec{{\sf vec}}

%% file: abstract.tex
\begin{abstract}

A key challenge of modern machine learning systems is to achieve Out-of-Distribution (OOD) generalization---generalizing to target data whose distribution differs from that of source data. Despite its significant importance, the fundamental question of ``what are the most effective algorithms for OOD generalization'' remains open even under the standard setting of covariate shift.
This paper addresses this fundamental question by proving that, surprisingly, classical Maximum Likelihood Estimation (MLE) purely using source data (without any modification) achieves the \emph{minimax} optimality for covariate shift under the \emph{well-specified} setting. That is, \emph{no} algorithm performs better than MLE in this setting (up to a constant factor), justifying MLE is all you need.
Our result holds for a very rich class of parametric models, and does not require any boundedness condition on the density ratio. We illustrate the wide applicability of our framework by instantiating it to three concrete examples---linear regression, logistic regression, and phase retrieval. This paper further complement the study by proving that, under the \emph{misspecified setting}, MLE is no longer the optimal choice, whereas Maximum Weighted Likelihood Estimator (MWLE) emerges as minimax optimal in certain scenarios.

\end{abstract}

%% file: introduction.tex
\section{Introduction}
Distribution shift, where the distribution of test data (target data) significantly differs from the distribution of training data (source data), is commonly encountered in practical machine learning scenarios \citep{zou2018unsupervised,ramponi2020neural,guan2021domain}. A central challenge of modern machine learning is to achieve Out-of-Distribution (OOD) generalization, where learned models maintain good performance in the target domain despite the presence of distribution shifts. To address this challenge, a variety of algorithms and techniques have been proposed, including vanilla empirical risk minimization (ERM) \citep{vapnik1999overview,gulrajani2020search}, importance weighting \citep{SHIMODAIRA2000227,huang2006correcting,cortes2010learning,cortes2014domain}, learning invariant representations \citep{ganin2016domain, arjovsky2019invariant,wu2019domain,rosenfeld2020risks}, distributionally robust optimization (DRO) \citep{sagawa2019distributionally}, etc. See the recent survey~\citep{shen2021towards} for more details. These results claim the effectiveness of the corresponding proposed algorithms in different regimes. This leads to a natural fundamental question:
\begin{center}
\textit{What are the most effective algorithms for OOD generalization?}    
\end{center}

This paper consider a widely-studied formulation of OOD-generalization---covariate shift. Under covariate shift, the marginal distributions of the input covariates $X$ vary between the source and target domains, while the conditional distribution of output given covariates $Y\mid X$ remains the same across domains. We consider learning a model from a \emph{known parametric model class} under \emph{well-specified} setting, where well-specification refers to the problems where the true conditional distribution of $Y \mid X$ lies in the given parametric model class. We argue that well-specified setting becomes increasingly more relevant in modern learning applications, because these applications typically use large-scale models with an enormous number of parameters, which are highly expressive and thus make the settings ``approximately'' well-specified.

Unfortunately, even under the basic setup of well-specified covariate shift, the aforementioned highlighted problem remains elusive --- while the seminar work \cite{SHIMODAIRA2000227} provides the first asymptotic guarantees for classical Maximum Likelihood Estimation (MLE) algorithm under this setup, and proves its optimality among a specific class of weighted likelihood estimators, his results leave two critical questions open: (1) Does MLE remain effective in the practical non-asymptotic scenario when the number of data is limited? (2) Do there exist smart algorithms beyond the class of weighted likelihood estimators that outperform MLE? This paper precisely addresses both critical questions and thus resolving the highlighted problem under well-specified covariate shift.

\paragraph{Our contributions.} Concretely, this paper makes following contributions:

\begin{enumerate}[leftmargin=*]
\item We prove that, for a large set of well-specified covariate shift problems, the classical Maximum Likelihood Estimation (MLE) --- which is computed purely based on source data without using any target data --- finds the optimal predictor on the target domain with prediction loss decreases as $\tilde{O}(\Tr(\cI_T\cI^{-1}_S)/n)$. Here $\Tr(\cdot)$ standards for trace, $\cI_S, \cI_T$ are the fisher information under source and target data distribution respectively, and $n$ is the number of source data. Our result does not require any boundedness condition on the density ratio, and is, to our best knowledge, the \emph{first} general, non-asymptotic, sharp result for MLE on a rich class of covariate shift problems.

\item We provide the \emph{first} minimax lower bound under well-specified covariate shift for \emph{any algorithm}, matching the error rate of MLE. This implies that MLE is minimax optimal, and no algorithm is better than MLE in this setting (up to a constant factor), justifying ``MLE is all you need''.

\item We instantiate our generic results by considering three representative examples with distinct problem structures: linear regression, logistic regression and phase retrieval. We verify preconditions, compute key quantities, and directly give covariate shift guarantees for these applications.


\item We further complement the study of this paper by considering the \emph{mis-specfied} setting where MLE ceases to work. We establish the \emph{first} general, non-asymptotic upper bound for the Maximum Weighted Likelihood Estimator (MWLE) provided bounded likelihood ratio. We prove that MWLE is minimax optimal under certain worst-case mis-specification.

\end{enumerate}

\paragraph{MLE versus MWLE.} This paper shows that importance weighting should not always be the go-to algorithm for covariate shift problems. Despite MWLE works under more general mis-specified setting given bounded density ratio, in the well-specified regime, MLE does not require bounded density ratio, and is provably more efficient than MWLE in terms of sample complexity. MLE is all you need for well-specified covariate shift problem.

\subsection{Related work}
\paragraph{Parametric covariate shift.} The statistical study of covariate shift under parametric models can be dated back to \cite{SHIMODAIRA2000227}, which established the asymptotic normality of MWLE and pointed out that vanilla MLE is asymptotically optimal among all the weighted likelihood estimators when the model is well-specified. However, no finite sample guarantees were provided, and the optimality of MLE is only proved within the restricted class of weighted likelihood estimators. In contrast, this paper establishes non-asymptotic results and proves the optimality of MLE among all possible estimators under well-specified models. \cite{NIPS2010_59c33016} studied the importance weighting under the statistical learning framework and gave a non-asymptotic upper bound for the generalization error of the weighted estimator. However, their rate scales as $\cO(1/\sqrt{n})$ compared to our rate $\cO(1/n)$, where $n$ is the sample size.
A recent line of work also provide non-asymptotic analyses for covariate shift under well-specified setting, however they focus on linear regression or a few specific models which are more restrictive than our setting: \cite{mousavi2020minimax} introduces a statistical minimax framework and provides lower bounds for OOD generalization in the context of linear and one-hidden layer neural network regression models. When applied to covariate shift, their lower bounds are loose and no longer minimax optimal. \cite{jason2021near} considers the minimax optimal estimator for linear regression under fixed design, the estimator they proposed is not MLE and is much more complicated in certain regimes. Finally, \cite{zhang2022class} considers covariate shift in linear regression where the learner can have access to a small number of target labels, this is beyond the scope of this paper, where we focus on the classical covariate shift setup in which target labels are not known.

\paragraph{Nonparametric covariate shift.} Another line of work focuses on well-specified nonparametric models under covariate shift. \cite{kpotufe2018marginal} presented minimax results for nonparametric classification problem, which was controlled by a transfer-exponent that measures the discrepancy between source and target.
Inspired by the aforementioned work, \cite{pathak2022new} studied nonparametric regression problem over the class of Hölder continuous functions with a more fine-grained similarity measure.
When considering reproducing
kernel Hilbert space (RKHS), \cite{ma2023optimally} showed  kernel ridge regression (KRR) estimator with a properly chosen penalty is minimax optimal for a large family of RKHS when the likelihood ratio is uniformly bounded, and a reweighted KRR using truncated likelihood ratios is minimax optimal when the likelihood ratio has a finite second moment. Later, \cite{wang2023pseudo} proposed a learning strategy based on pseudo-labels. When the likelihood ratio is bounded, their estimator enjoyed the optimality guarantees without prior knowledge about the amount of covariate shift. 
Although these works focused on covariate shift problems, they considered nonparametric setting, and hence are not directly comparable to our work. 
As an example, \cite{ma2023optimally} showed that MLE (empirical risk minimization in their language) is provably suboptimal for addressing covariate shift under nonparametric RKHS assumptions. In contrast, we show that MLE is optimal for covariate shift for a well-specified parametric model. 
We also highlight that our lower bound is instance dependent in the sense that it depends on the source and target distributions. This is in contrast to prior work (e.g. \cite{ma2023optimally}, \cite{kpotufe2018marginal}) that consider the worst-case scenario over certain classes of source-target pairs (e.g., bounded density ratios).



\paragraph{Maximum likelihood estimation.} A crucial part of this work is analyzing MLE, which is a dominant approach in statistical inference. There exists a variety of work studying the behavior of MLE under the standard no-distribution-shift setting. It is well known that MLE is asymptotically normal \citep{casella2021statistical} with the inverse of Fisher information as the asymptotic variance. \cite{Cramr1946MathematicalMO,Rao1992InformationAT} established the famous Cramer-Rao bound for unbiased estimators, which also showed that no consistent estimator has lower asymptotic mean squared error than the MLE. \cite{white1982maximum} gave the asymptotic distribution of MLE under the mis-specified setting. More recently, non-asymptotic behaviours of MLE are studied under certain models. \cite{10.1214/09-EJS521,ostrovskii2021finite} established the non-asymptotic error bound for MLE in logistic regression using self-concordance. This line of work does not consider covariate shift, which is an indispensable part of this paper.

\paragraph{Importance reweighting algorithms.} Lastly, importance reweighting (or importance sampling) is a classical method to use independent samples from a proposal distribution to approximate expectations w.r.t.~a target measure~\citep{agapiou2017importance}. \cite{chatterjee2018sample} studied the sample size (depending on the KL divergence between two distributions) required for importance sampling to approximate a single function. \cite{sanz2018importance} extended analysis to the case with general $f$-divergences. In addition to correcting covariate shift, importance reweighting has been central in offline reinforcement learning. For instance, \cite{ma2022minimax} showed a truncated version of importance reweighting is minimax optimal for estimation the value of a target policy using data from a behavior policy. For learning the optimal policy from the behavior data, \cite{swaminathan2015batch} presented upper bounds of an importance-reweighted estimator. This spurs a long line of work of using importance weighting in offline RL. See the recent work~\cite{gabbianelli2023importance} and the references therein. 


%% file: setup.tex
\section{Background and Problem Formulation}
In this section, we provide background on the problem of learning under covariate shift. We also review two widely adopted estimators: maximum likelihood estimator and maximum weighted likelihood estimator.

\paragraph{Notations.} 
Throughout the paper, we use $c$ to denote universal constants, which may vary from line to line.

\subsection{Covariate shift and excess risk}
Let $X\in \cX$ be the covariates and $Y\in\cY$ be the response variable that we aim to predict. In a general out-of-distribution (OOD) generalization problem, we have two domains of interest, namely a source domain $S$ and a target domain $T$. Each domain is associated with a data generating distribution over $(X,Y)$: $\bbP_S(X,Y)$ for the source domain and $\bbP_T(X,Y)$ for the target domain. Given 
$n$ i.i.d.~labeled samples $\{(x_i,y_i)\}^n_{i=1}\sim\bbP_S(X,Y)$ from the source domain, the goal of OOD generalization is to learn a prediction rule $X \to Y$ that performs well in the target domain.
In this paper, we focus on the covariate shift version of the OOD generalization problem, in which the marginal distributions $\bbP_S(X)$ and $\bbP_T(X)$ of the covariates could differ between the source and target domains, while the conditional distribution $Y\,|\,X$ is assumed to be the same on both domains.  

More precisely, we adopt the notion of excess risk to measure of the performance of an estimator under covariate shift. 
Let $\cF:=\{f(y\,|\,x;\beta) \mid \beta\in\bbR^d\}$ be a parameterized function class to model the conditional density function $p(y\,|\,x)$ of $Y\,|\,X$. A typical loss function is  defined using the negative log-likelihood function:
\begin{align*}
\ell(x,y,\beta):=-\log f(y\,|\, x;\beta).
\end{align*}
The excess risk at $\beta$ is then defined as
\begin{align}\label{risk}
R(\beta):= \bbE_{T}\left[\ell(x,y,\beta)\right]-\inf_{\beta}\bbE_{T}\left[\ell(x,y,\beta)\right],
\end{align}
where the expectation $\bbE_{T}$ is taken over $\bbP_T(X,Y)$. 
When the model is well-specified, i.e., when the true density $p(y\,|\,x)=f(y\,|\, x;\beta^{\star})$ for some $\beta^{\star}$, we have $\inf_{\beta}\bbE_{T}[\ell(x,y,\beta)]= \bbE_{T}[\ell(x,y,\beta^{\star})]$. As a result, we evaluate the loss at $\beta$ against the loss at the true parameter $\beta^{\star}$. 
In contrast, in the case of mis-specification, i.e., when $p(y\,|\,x) \notin \cF$, the loss at $\beta$ is compared against the loss of the best fit in the model class.

\subsection{Maximum likelihood estimation and its weighted version}
In the no-covariate-shift case, maximum likelihood estimation (MLE) is arguably the most popular approach. 
Let \begin{align}\label{def:emp_loss}
\ell_n(\beta):=\frac1n \sum^n_{i=1}\ell(x_i,y_i,\beta)
\end{align}
be the empirical negative log-likelihood using the samples $\{(x_i,y_i)\}^n_{i=1}$ from the source domain. 
The vanilla MLE is defined as
\begin{align}\label{eq:MLE}
\beta_{\MLE}:=\argmin_{\beta\in\bbR^d} \ell_n(\beta).
\end{align}

One potential ``criticism'' against MLE in the covariate shift setting is that the empirical negative log-likelihood is not a faithful estimate of the out-of-distribution generalization performance, i.e., 
$\bbE_{T}\left[\ell(x,y,\beta)\right]$. In light of this, a weighted version of MLE is proposed. Let $w(x):=d\bbP_T(x)/d\bbP_S(x)$ be the density ratio function and
\begin{align}\label{def:emp_loss_mis}
\ell^{w}_n(\beta):=\frac1n  \sum^n_{i=1} w(x_i) \ell(x_i,y_i,\beta).
\end{align}
be the weighed loss. Then the maximum weighted likelihood estimator is defined as
\begin{align}\label{eq:WMLE}
\beta_{\MWLE}:=\argmin_{\beta\in\bbR^d} \ell^{w}_n(\beta).
\end{align}
It is easy to see that the weighted loss is an unbiased estimate of $\bbE_{T}\left[\ell(x,y,\beta)\right]$.


To ease  presentations later, we would also recall the classical notion of Fisher information---an important quantity to measure the difficulty of parameter estimation. The Fisher information evaluated at $\beta$ on source and target is defined as
\begin{align*}
&\cI_S(\beta):=\bbE_{x\sim\bbP_S(X), y\,|\,x\sim f(y\,|\, x;\beta)}\left[\nabla^2 \ell(x,y,\beta)\right],\\
&\cI_T(\beta):=\bbE_{x\sim\bbP_T(X), y\,|\,x\sim f(y\,|\, x;\beta)}\left[\nabla^2 \ell(x,y,\beta)\right].
\end{align*}
Here, the gradient and Hessian are taken with respect to the parameter $\beta$. 



%% file: wellspecified.tex
\section{Well-Specified Parametric Model under Covariate Shift}\label{well-specified}
In this section, we focus on covariate shift with a well-specified model, that is, the true conditional distribution falls in our parametric function class. This setting aligns with the practice, since in modern machine learning we often deploy large models whose representation ability are so strong that every possible true data distribution almost falls in the function class. We assume there exists some $\beta^{\star}$ such that 
$p(y\,|\,x)=f(y\,|\,x;\beta^{\star})$, and
denote the excess risk evaluated at $\beta$ under true model parameter $\beta^{\star}$ as $R_{\beta^{\star}}(\beta)$, i.e.,
\begin{align}\label{risk2}
R_{\beta^{\star}}(\beta):= \bbE_{\substack{x\sim\bbP_T(X)\\y|x\sim f(y|x;\beta^{\star})}}\left[\ell(x,y,\beta)\right]-\bbE_{\substack{x\sim\bbP_T(X)\\y|x\sim f(y|x;\beta^{\star})}}\left[\ell(x,y,\beta^{\star})\right].
\end{align}

While the objective of MLE (cf.~\eqref{eq:MLE}) is not an unbiased estimate of the risk under the target domain, we will show in this section that MLE is in fact optimal for addressing covariate shift under well-specified models.

More specifically, in Section~\ref{sec:upper-MLE}, we provide the performance upper bound for MLE under generic assumptions on the parametric model. Then in Section~\ref{sec:lower-bound-well-specified}, we characterize the performance limit of any estimator in the presence of covariate shift. As we will see, MLE is minimax optimal as it matches the performance limit.

\subsection{Upper bound for MLE}\label{sec:upper-MLE}
In this subsection, we establish a non-asymptotic upper bound for MLE under generic assumptions on the model class.  



\begin{assumption}\label{assm:well_upper}
We make the following assumptions on the model class $\cF$:
\begin{enumerate}[label=A.\arabic*,leftmargin=*]
\item \label{assm1} 
  There exist $B_1, B_2$, $N(\delta)$, and absolute constants $c, \gamma$ such that for any fixed  matrix $A \in \mathbb{R}^{d \times d}$, any $\delta \in (0, 1)$, and any $n > N(\delta)$, with probability at least $1-\delta$:
\begin{align}\label{assm1:ineq1}
\left\|A\left(\nabla\ell_n(\beta^{\star})-\bbE[\nabla\ell_n(\beta^{\star})]\right)\right\|_{2} & \leq c \sqrt{\frac{ V \log \frac{d}{\delta}}{n}}+ B_{1} \|A\|_2 \log^\gamma\left(\frac{B_{1} \|A\|_2}{\sqrt{V}}\right) \frac{ \log \frac{d}{\delta}}{n},\\\label{assm1:ineq2}
\left\|\nabla^{2} \ell_{n}(\beta^{\star})-\bbE[\nabla^{2} \ell_{n}(\beta^{\star})]\right\|_2 &\leq B_2 \sqrt{\frac{\log \frac{d}{\delta}}{n}},
\end{align}
where $V = n \cdot \mathbb{E} \|A(\nabla\ell_n(\beta^{\star})-\bbE[\nabla\ell_n(\beta^{\star})])\|_2^2$ is the variance.

 \item \label{assm2} There exists some constant $B_3\geq 0$ such that $\|\nabla^3 \ell (x,y,\beta)\|_2\leq B_3$ for all $x\in\cX_S\cup\cX_T, y\in\cY, \beta\in\bbR^d$, where $\cX_S$ (resp.~$\cX_T$) is the support of $\bbP_S(X)$ (resp.~$\bbP_T(X)$). 

\item \label{assm3}The empirical loss $\ell_n(\cdot)$ defined in \eqref{def:emp_loss} has a unique local minimum in $\bbR^d$, which is also the global minimum.
\end{enumerate}
\end{assumption}

Several remarks on Assumption~\ref{assm:well_upper} are in order. 
Assumption \ref{assm1} is a general version of Bernstein inequality (when $\gamma=0$ it reduces to classical Bernstein inequality), which gives concentration on gradient and Hessian. This assumption is naturally satisfied when the gradient and Hessian are bounded (see Proposition \ref{prop:concentration} for details). Assumption \ref{assm2} requires the third order derivative of log-likelihood to be bounded, which is easy to satisfy (e.g., linear regression satisfies this assumption with $B_3=0$). Assumption \ref{assm3} ensures the MLE is unique, which is standard in the study of the behaviour of MLE. We can see that it naturally applies to traditional convex losses. It is worth  noting that our general theorem can also be applied under a relaxed version of Assumption \ref{assm3}, which will be shown in Theorem \ref{thm:phase}. In Section~\ref{application}, we will see that Assumption \ref{assm:well_upper} is mild and easily satisfied for a wide range of models.

Now we are ready to present the performance upper bound for MLE under covariate shift. 
\begin{theorem}\label{thm:well_upper}
Suppose that the model class $\cF$ satisfies Assumption \ref{assm:well_upper}. Let $\cI_T:=\cI_T(\beta^{\star})$ and $\cI_S:=\cI_S(\beta^{\star})$.
For any $\delta\in (0,1)$, if $n\geq c \max\{N^{\star}\log(d/\delta), N(\delta)\}$, then with probability at least $1-2\delta$, we have
\begin{align*}
R_{\beta^{\star}}(\beta_{\MLE}) 
\leq c\frac{\Tr\left(\cI_T\cI^{-1}_S\right)\log \frac{d}{\delta}}{n}
\end{align*}
for an absolute constant $c$. Here
$N^{\star}:=  \text{Poly} (d, B_1, B_2, B_3, \|\cI_S^{-1}\|_2, \|\cI_T^{\frac12} \cI_S^{-1}\cI_T^{\frac12}\|_2^{-1}).$
\end{theorem}
For an exact characterization of the threshold $N^{\star}$, one can refer to Theorem \ref{thm_exact:well_upper} in the appendix.
Theorem \ref{thm:well_upper} gives a non-asymptotic upper bound for the excess risk of MLE: when the sample size exceeds a certain threshold of $\max\{N^{\star}\log(d/\delta), N(\delta)\}$, MLE achieves an instance dependent risk bound $\Tr (\cI_T \cI_S^{-1})/n$. 
It is worth noting that our analysis does not require boundedness on the density ratios between the target and source distributions (as have been assumed in prior art~\citep{ma2023optimally}), which yields broader applicability. In Section \ref{application}, we will instantiate our generic analysis on three different examples: linear regression, logistric regression and phase retrieval.

\subsection{Minimax lower bound}\label{sec:lower-bound-well-specified}
In the previous section, we have established the upper bound for the vanilla MLE. 
Now we turn to the complementary question regarding the fundamental limit of covariate shift under well-specified models.
To establish the lower bound, we will need the following Assumption \ref{assm:well_lower} that is a slight variant of Assumption \ref{assm:well_upper}. Different from the upper bound, the lower bound is algorithm independent and involve a model class rather than a fixed ground truth. Hence, Assumption \ref{assm:well_lower} focuses on population properties of our model as opposed to Assumption \ref{assm:well_upper}, which is on the sample level.


\begin{assumption}\label{assm:well_lower}
Let $\beta_0\in\bbR^d$ and $B>0$.  We make the following assumptions on the model class~$\cF$:
\begin{enumerate}[label=\theassumption.\arabic*,leftmargin=*]
\item Assumption \ref{assm2} holds.

\item \label{assm4} There exist some constants $L_S, L_T\geq 0$ such that for any $\beta_1,\beta_2\in\bbB_{\beta_0}(B)$:
\begin{align*}
&\|\cI_S(\beta_1)-\cI_S(\beta_2)\|_2\leq L_S \|\beta_1-\beta_2\|_2,\\
&\|\cI_T(\beta_1)-\cI_T(\beta_2)\|_2\leq L_T \|\beta_1-\beta_2\|_2.
\end{align*}

\item \label{assm5} For any $\beta^{\star}\in\bbB_{\beta_0}(B)$, the excess risk $R_{\beta^{\star}}(\beta)$ defined in \eqref{risk2} is convex in $\beta\in\bbR^d$.

\item \label{assm7} We assume $\cI_S(\beta)$ and $\cI_T(\beta)$ are positive definite for all $\beta\in\bbB_{\beta_0}(B)$.
\end{enumerate}
\end{assumption}

Assumption \ref{assm4} essentially requires the Fisher information will not vary drastically in a small neighbourhood of $\beta_0$. This assumption is easy to hold when the fisher information has certain smoothness (e.g., in linear regression, the fisher information does not change when $\beta$ varies). Since Assumption \ref{assm:well_lower} is a slight variant of Assumption \ref{assm:well_upper}, both assumptions are often satisfied simultaneously for a wide range of models, as we will show in Section \ref{application}.


\begin{theorem}\label{thm:well_lower}
Suppose the model class $\cF$ satisfies Assumption \ref{assm:well_lower}. As long as $n \geq N_0$, we have
\begin{align*}
\inf_{\hat\beta}\sup_{\beta^{\star}\in\bbB_{\beta_0}(B)}\Tr\left(\cI_T(\beta^{\star})\cI^{-1}_S(\beta^{\star})\right)^{-1}\bbE_{\substack{x_i\sim\bbP_S(X)\\y_i|x_i\sim f(y|x;\beta^{\star})}}\left[R_{\beta^{\star}}(\hat\beta)\right]\geq \frac{1}{50n},
\end{align*}
where $N_0:= \text{Poly} (d,B^{-1}, B_3, L_S, L_T, \|\cI_S(\beta_0)\|_2, \|\cI_T(\beta_0)\|_2, \|\cI_S(\beta_0)^{-1}\|_2, \|\cI_T(\beta_0)^{-1}\|_2)$.
\end{theorem}

For an exact characterization of the threshold $N_0$, one can refer to Theorem \ref{thm_exact:well_lower} in the appendix.

Comparing Theorem \ref{thm:well_upper} and \ref{thm:well_lower}, we can see that, under\footnote{It is worthy to point out that, it is not hard for Assumptions \ref{assm:well_upper} and \ref{assm:well_lower} to be satisfied simultaneously. These assumptions will hold naturally when the domain is bounded and the log-likelihood is of certain convexity and smoothness, as we will show in the next section by several concrete examples. } Assumptions \ref{assm:well_upper} and \ref{assm:well_lower}, then for large enough sample size $n$, $\Tr\left(\cI_T(\beta^{\star})\cI^{-1}_S(\beta^{\star})\right)/n$ exactly characterizes the fundamental hardness of covariate shift under well-specified parametric models. It also reveals that vanilla MLE is minimax optimal under this scenario.
To gain some intuitions, $\cI_S^{-1}$ captures the variance of the parameter estimation, and $\cI_T$ measures how the excess risk on the target depends on the estimation accuracy of the parameter. 
Therefore what really affects the excess risk (on target) is the accuracy of estimating the parameter, and vanilla MLE is naturally the most efficient choice.

We also highlight that our lower bound is instance dependent in the sense that it depends on the source and target distributions. This is in contrast to prior work (e.g. \cite{ma2023optimally}, \cite{kpotufe2018marginal}) that consider the worst-case scenario over certain classes of source-target pairs (e.g., bounded density ratios).


%% file: application.tex
\section{Applications}\label{application}
In this section, we illustrate the broad applicability of our framework by delving into three distinct statistical models, namely linear regression, logistic regression and phase retrieval.
For each model, we will demonstrate the validity of the assumptions, and give the explicit non-asymptotic upper bound on the vanilla MLE obtained by our framework as well as the threshold of sample size needed to obtain the upper bound.



\subsection{Linear regression}
In linear regression, we have $Y = X^{T}\beta^{\star}+\varepsilon$, where $\varepsilon\sim\cN(0,1)$ and $\varepsilon\indep X$. The corresponding negative log-likelihood function (i.e. the loss function) is given by
\begin{align*}
\ell(x,y,\beta):=\frac{1}{2}(y-x^{T}\beta)^2.
\end{align*}
We assume $X\sim\cN(0,I_d)$ on the source domain and $X\sim\cN(\alpha,\sigma^2 I_d)$ on the target domain.

\begin{proposition}\label{prop:linear}
The aforementioned linear regression model satisfies Assumption \ref{assm:well_upper} and \ref{assm:well_lower} with $\gamma=1$, $N(\delta)=d\log (1/\delta)$, $B_1=c\sqrt{d}$, $B_2=c\sqrt{d}$, $B_3=0$ and $L_S=L_T=0$. Moreover, we have $\Tr(\cI_T \cI_S^{-1})=\|\alpha\|_2^2+\sigma^2 d$.
\end{proposition}
By Theorem \ref{thm:well_upper} and Theorem \ref{thm:well_lower}, since Assumption \ref{assm:well_upper} and \ref{assm:well_lower} are satisfied, we immediately demonstrate the optimality of MLE under linear regression. The following theorem gives the explicit form of excess risk bound by applying Theorem \ref{thm:well_upper}:

\begin{theorem}\label{thm:linear}
For any $\delta\in (0,1)$, if $n\geq \cO (N\log \frac{d}{\delta})$, then with probability at least $1-2\delta$, we have
\begin{align*}
&R_{\beta^{\star}}(\beta_{\MLE}) 
\leq c\frac{\left(\|\alpha\|^2_2+\sigma^2 d\right)\log \frac{d}{\delta}}{n},
\end{align*}
where
$
N:=d\left(1+\frac{\|\alpha\|^2_2d+\sigma^2 d}{\|\alpha\|^2_2+\sigma^2 d}\right)^2 .
$
\end{theorem}



\begin{remark}[Excess risk]
Regarding the upper bound of the excess risk, we categorize it into two scenarios: large shift and small shift. 
In the small shift scenarios (i.e., $\|\alpha\|_2^2\leq \sigma^2 d$), the result is the same as that in scenarios without any mean shift, with a rate of $\sigma^2 d/n$. 
On the other hand, in the large shift scenarios (i.e., $\|\alpha\|_2^2\geq \sigma^2 d$), the upper bound of the excess risk increases with the mean shift at a rate of $\|\alpha\|_2^2/n$.
\end{remark}

\begin{remark}[Threshold $N$]
For a minor mean shift, specifically when $\|\alpha\|_2=c\sigma$ for a given constant $c$, the threshold is $N=d$. This aligns with the results from linear regression without any covariate shift. On the other hand, as the mean shift increases (i.e., $|\alpha|_2=\sigma d^{k}$ for some $0< k< 1/2$), the threshold becomes $N=d^{4k+1}$, increasing with the growth of $k$. In scenarios where the mean shift significantly surpasses the scaling shift, denoted as $\alpha\geq\sigma\sqrt{d}$, the threshold reaches $N=d^3$.
\end{remark}

\subsection{Logistic regression}\label{sec:logistic}
In the logistic regression, the response variable $Y\in\{0,1\}$ obeys
\begin{align*}
    \bbP(Y=1\,|\,X=x)=\frac{1}{1+e^{x^{T}\beta^{\star}}},\,\, \bbP(Y=0\,|\,X=x)=\frac{1}{1+e^{-x^{T}\beta^{\star}}}.
\end{align*}
The corresponding negative log-likelihood function (i.e. the loss function) is given by
\begin{align*}
    \ell(x,y,\beta):=\log(1+e^{x^{T}\beta})-y(x^T\beta).
\end{align*}
We assume $X\sim\uni (\cS^{d-1}(\sqrt{d}))$ on the source domain and $X\sim\uni(\cS^{d-1}(\sqrt{d}))+v$ on the target domain, where $\cS^{d-1}(\sqrt{d}):=\{x\in\bbR^d\mid \|x\|_2=\sqrt{d}\}$. In the following, we will give the upper bound of the excess risk for MLE when $v= r\beta_{\perp}^{\star}$, where $\beta_{\perp}^{\star}$ represents a vector perpendicular to $\beta^{\star}$ (i.e., $\beta_{\perp}^{\star T}\beta^{\star}$=0). Without loss of generality, we assume $\|\beta^{\star}\|_2=\|\beta_{\perp}^{\star }\|_2=1$.


\begin{proposition}\label{prop:logistic}
The aforementioned logistic regression model satisfies Assumption \ref{assm:well_upper} and \ref{assm:well_lower} with $\gamma=0$, $N(\delta)=0$, $B_1=c\sqrt{d}$, $B_2=cd$, $B_3=(\sqrt{d}+r)^3$, $L_S=d^{1.5}$ and $L_T=(\sqrt{d}+r)^3$. Moreover, we have $\Tr(\cI_T \cI_S^{-1}) \asymp d + r^{2}$.
\end{proposition}
By Theorem \ref{thm:well_upper} and Theorem \ref{thm:well_lower}, since Assumption \ref{assm:well_upper} and \ref{assm:well_lower} are satisfied, we immediately demonstrate the optimality of MLE under logistic regression. The following theorem gives the explicit form of excess risk bound by applying Theorem \ref{thm:well_upper}:
\begin{theorem}\label{thm:logistic}
For any $\delta\in (0,1)$, if $n\geq \cO (N\log \frac{d}{\delta})$, then with probability at least $1-2\delta$, we have
\begin{align*}
R_{\beta^{\star}}(\beta_{\MLE}) 
\leq c\frac{(d+r^2)\log \frac{d}{\delta}}{n},
\end{align*}
where $N:=d^4(1+r^6)$.

\end{theorem}

\begin{remark}[Excess risk]
The bound on the excess risk incorporates a $r^2$ term, which is a measurement of the mean shift. This is due to the fact that the MLE does not utilize the information that $v^{T}\beta^{\star}=0$.
Therefore, $v^{T}\beta_{\MLE}$ is not necessarily zero, which will lead to an additional bias. Similar to linear regression, we can categorize the upper bound of the excess risk into two scenarios: large shift ($r \geq \sqrt{d}$) and small shift ($r \leq \sqrt{d}$). 
\end{remark}

\begin{remark}[Threshold $N$]
We admit that the $N$ here may not be tight, as we lean on a general framework designed for a variety of models rather than a specific one. 
\end{remark}

\subsection{Phase retrieval}
As we have mentioned, our generic framework can also be applied to the scenarios where some of the assumptions are relaxed. In this subsection, we will further illustrate this point by delving into the phase retrieval model. 

In the phase retrieval, the response variable $Y = (X^{T}\beta^{\star})^2+\varepsilon$, where $\varepsilon\sim\cN(0,1)$ and $\varepsilon\indep X$. 
We assume $\bbP_S(X)$ and $\bbP_T(X)$ follow the same distribution as that in the logistic regression model (i.e., Section \ref{sec:logistic}). 
Note that both the phase retrieval model and the logistic regression model belong to generalized linear model (GLM), thus they are expected to have similar properties.
However, given the loss function $\ell(x,y,\beta):=\frac{1}{2}\left(y-(x^{T}\beta)^2\right)^2$, it is obvious that Assumption \ref{assm3} is not satisfied, since if $\beta$ is a global minimum of $\ell_{n}$, $-\beta$ is also a global minimum. The following theorem shows that we can still obtain results similar to logistic regression though Assumption \ref{assm3} fails to hold.

\begin{theorem}\label{thm:phase}
For any $\delta \in (0,1)$, if $n\geq \cO(N\log \frac{d}{\delta})$, then with probability at least $1-2\delta$, we have
\begin{align*}
R_{\beta^{\star}}(\beta_{\MLE}) 
\leq c\frac{(d+r^2)\log \frac{d}{\delta}}{n},
\end{align*}
where $N:=d^8(1+r^8)$.
\end{theorem}

%% file: misspecified.tex
\section{Mis-Specified Parametric Model under Covariate Shift}\label{mis-specified}

In the case of model mis-specification, we still employ a parameterized function class $\cF:=\{f(y\,|\,x;\beta)\,|\,\beta\in\bbR^d\}$ to model the conditional density function of $Y\,|\,X$. However, the true density $p(y\,|\,x)$ might not be in $\cF$.
As we previously showed, under a well-specified parametric model, the vanilla MLE is minimax optimal up to constants. However, when the model is mis-specified, the classical MLE may not necessarily provide a good estimator.
\begin{proposition}\label{prop:mle_consistent}
There exist certain mis-specified scenarios such that classical MLE is not consistent, whereas MWLE is.
\end{proposition}
Proposition \ref{prop:mle_consistent} illustrates the necessity of adaptation under model mis-specification since the classical MLE asymptotically gives the wrong estimator. In this section, we study the non-asymptotic property of MWLE. 
Let $\cM$ be the model class of the ground truth $Y\,|\,X$, and $M \in \cM$ be the ground truth model for $Y\,|\,X$. 


We denote the optimal fit on target as
\begin{align*}
\beta^{\star}(M):=\argmin_{\beta}\bbE_{\substack{x\sim\bbP_T(X)\\y|x\sim M}}[\ell(x,y,\beta)].
\end{align*}
The excess risk evaluated at $\beta$ is then given by
\begin{align}
R_{M}(\beta) 
&= \bbE_{\substack{x\sim\bbP_T(X)\\y|x\sim M}}\left[\ell(x,y,\beta)\right]-\bbE_{\substack{x\sim\bbP_T(X)\\y|x\sim M}}\left[\ell(x,y,\beta^{\star}(M))\right].
\end{align}

\subsection{Upper bound for MWLE}
In this subsection, we establish the non-asymptotic upper bound for MWLE, as an analog to Theorem \ref{thm:well_upper}. We make the following assumption which is a modification of Assumption \ref{assm:well_upper}. 


\begin{assumption}\label{assm:mis_upper}
We assume the function class $\cF$ satisfies the follows:
\begin{enumerate}[label=\theassumption.\arabic*,leftmargin=*]

\item \label{assm6}There exists some constant $W>1$ such that the density ratio $w(x)\leq W$ for all $x\in\cX_S\cup\cX_T$.

\item \label{assm1_mis} There exist $B_1, B_2$ and $N(\delta)$, and absolute constants $c, \gamma$ such that for any fixed  matrix $A \in \mathbb{R}^{d \times d}$, any $\delta \in (0, 1)$, and any $n > N(\delta)$, with probability at least $1-\delta$:
\begin{align*}
\left\|A\left(\nabla\ell^{w}_n(\beta^{\star}(M))-\bbE[\nabla\ell^{w}_n(\beta^{\star}(M))]\right)\right\|_{2} & \leq c \sqrt{\frac{ V \log \frac{d}{\delta}}{n}}+ WB_{1} \|A\|_2 \log^\gamma\left(\frac{WB_{1} \|A\|_2}{\sqrt{V}}\right) \frac{ \log \frac{d}{\delta}}{n},\\
\left\|\nabla^{2} \ell^{w}_{n}(\beta^{\star}(M))-\bbE[\nabla^{2} \ell^{w}_{n}(\beta^{\star}(M))]\right\|_2 &\leq WB_2 \sqrt{\frac{\log \frac{d}{\delta}}{n}},
\end{align*}
where $V = n \cdot \mathbb{E} \|A(\nabla\ell^{w}_n(\beta^{\star}(M))-\bbE[\nabla\ell^{w}_n(\beta^{\star}(M))])\|_2^2$ is the variance.

\item \label{assm2_mis} Assumption \ref{assm2} holds.

\item \label{assm3_mis} There exists $N'(\delta)$ such that for any $\delta\in (0,1)$ and any $n\geq N'(\delta)$, with probability at least $1-\delta$, the empirical loss $\ell^{w}_n(\cdot)$ defined in \eqref{def:emp_loss_mis} has a unique local minimum in $\bbR^d$, which is also the global minimum. 
\end{enumerate}
\end{assumption}

Assumption \ref{assm6} is a density ratio upper bound (not required for analyzing MLE), which is essential for the analysis of MWLE. Assumption \ref{assm1_mis} is an analog of Assumption \ref{assm1}, in the sense that the empirical loss $\ell_n$ is replaced by its weighted version $\ell_n^w$. Assumption \ref{assm3_mis} is a weaker version of Assumption \ref{assm3} in the sense that it only requires $\ell_n^w$ has a unique local minimum with high probability. This is due to the nature of reweighting: when applying MWLE, $w(x_i)$ can sometimes be zero, which lead to the degeneration of $\ell_n^w$ (with a small probability). Therefore we only require the uniqueness of local minimum holds with high probability.

To state our non-asymptotic upper bound for MWLE, we define the following ``weighted version'' of Fisher information:
\begin{align*}
&G_w(M):=\bbE_{\substack{x\sim\bbP_S(X)\\y|x\sim M}}\left[w(x)^2\nabla\ell(x,y,\beta^{\star}(M))\nabla\ell(x,y,\beta^{\star}(M))^{T}\right],\\
&H_w(M):=\bbE_{\substack{x\sim\bbP_S(X)\\y|x\sim M}}\left[w(x)\nabla^2\ell(x,y,\beta^{\star}(M))\right]=\bbE_{\substack{x\sim\bbP_T(X)\\y|x\sim M}}\left[\nabla^2\ell(x,y,\beta^{\star}(M))\right].
\end{align*}

\begin{theorem}\label{thm:mis_upper}
Suppose the function class $\cF$ satisfies Assumption \ref{assm:mis_upper}. Let $G_w := G_w(M)$ and $H_w := H_w(M)$. For any $\delta\in (0,1)$, if $n\geq c \max\{N^{\star}\log(d/\delta), N(\delta), N'(\delta)\}$, then with probability at least $1-3\delta$, we have
\begin{align*}
&R_{M}(\beta_{\MWLE}) 
\leq c\frac{\Tr\left(G_w H^{-1}_w\right)\log \frac{d}{\delta}}{n}
\end{align*}
for an absolute constant $c$. Here $N^{\star}
:= \text{Poly}(W,B_1,B_2,B_3,\|H_w^{-1}\|_2, \Tr(G_w H_w^{-2}), \Tr(G_w H_w^{-2})^{-1})$.
\end{theorem}
For an exact characterization of the threshold $N^{\star}$, one can refer to Theorem \ref{thm_exact:mis_upper} in the appendix.

Compared with Theorem \ref{thm:well_upper}, Theorem \ref{thm:mis_upper} does not require well-specification of the model, demonstrating the wide applicability of MWLE. The excess risk upper bound can be explained as follows: note that $\Tr(G_w H^{-1}_w)$ can be expanded as $\Tr(H_w H^{-1}_w G_w H^{-1}_w)$. As shown by \cite{SHIMODAIRA2000227}, the term $\sqrt{n}(\beta_{\MWLE}-\beta^{\star})$ converges asymptotically to a normal distribution, denoted as $\cN(0,H_w^{-1}G_wH_w^{-1})$. Thus, the component $H^{-1}_w G_w H^{-1}_w$ characterizes the variance of the estimator, corresponding to the $\cI^{-1}_S$ term in Theorem \ref{thm:well_upper}. 
Additionally, the excess risk's dependence on the parameter estimation is captured by $H_w$ as a counterpart of $\cI_T$ in Theorem \ref{thm:well_upper}.

However, to establish Theorem \ref{thm:mis_upper}, it is necessary to assume the bounded density ratio, which does not appear in Theorem \ref{thm:well_upper}.
Moreover, when the model is well-specified, by Cauchy-Schwarz ineqaulity, we have $\Tr(G_w H_w^{-1})\geq \Tr(\cI_T\cI_S^{-1})$, which implies the upper bound for MWLE is larger than the vanilla MLE. This observation aligns with the results presented in \citet{SHIMODAIRA2000227}, which point out that when the model is well specified, MLE is more efficient than MWLE in terms of the asymptotic variance.



\subsection{Optimality of MWLE}
To understand the optimality of MWLE, it is necessary to establish a matching lower bound.
However, deriving a lower bound similar to Theorem \ref{thm:well_lower}, which holds for any model classes that satisfies certain mild conditions, is challenging due to hardness of capturing the difference between $\cM$ and $\cF$.
As a solution, we present a lower bound tailored for certain model classes and data distributions in the following.

\begin{theorem}\label{thm:mis_optimal}
There exist $\bbP_S(X)\neq \bbP_T(X)$, a model class $\cM$  and a prediction class $\cF$ satisfying Assumption \ref{assm:mis_upper} such that when $n$ is sufficiently large, we have
\begin{align}\label{thm:mis_optimal_lower}
\inf_{\hat\beta}\sup_{M \in\cM}\Tr\left(G_w(M) H^{-1}_w(M)\right)^{-1}\bbE_{\substack{x_i\sim\bbP_S(X)\\y_i|x_i\sim M}}\left[R_M(\hat\beta)\right] \gtrsim \frac{1}{n}.
\end{align}
\end{theorem}

By Theorem \ref{thm:mis_upper}, the excess risk of MWLE is upper bounded by ${\Tr (G_w H_w^{-1})}/{n}$. Therefore, Theorem \ref{thm:mis_optimal} shows that there exists a non-trivial scenario where MWLE is minimax optimal. 

Notice that Theorem \ref{thm:mis_optimal} presents a weaker lower bound compared to  Theorem \ref{thm:well_lower}. The lower bound presented in Theorem \ref{thm:mis_optimal} holds only for certain meticulously chosen $\bbP_S(X), \bbP_T(X)$, model class $\cM$ and prediction class $\cF$. In contrast, the lower bound in Theorem \ref{thm:well_lower} applies to any $\bbP_S(X), \bbP_T(X)$, and class $\cF$ that meet the required assumptions.


%% file: conclusion.tex
\section{Conclusion and Discussion}
To conclude, we prove that MLE achieves the minimax optimality for covariate shift under a well-specified parametric model. Along the way, we demonstrate that the term $\Tr (\cI_T \cI_S^{-1})$ characterizes the foundamental hardness of covariate shift, where $\cI_S$ and $\cI_T$ are the Fisher information on the source domain and the target domain, respectively. To complement the study, we also consider the misspecified setting and show that Maximum Weighted Likelihood Estimator (MWLE) emerges as minimax optimal in specific scenarios, outperforming MLE.



Our work opens up several interesting avenues for future study. 
First, it is of great interest to extend our analysis to other types of OOD generalization problems, e.g., imbalanced data, posterior shift, etc. Second, our analyses relies on standard regularity assumptions, such as the positive definiteness of the Fisher information (which implies certain identifiability of the parameter) and the uniqueness of the minimum of the loss function. Addressing covariate shift without these assumptions is also important future directions.


%% file: pf_wellspecified.tex
\section{Proofs for Section \ref{well-specified}}

\subsection{Proofs for Theorem \ref{thm:well_upper}}
The detailed version of Theorem \ref{thm:well_upper} is stated as the following.
\begin{theorem} \label{thm_exact:well_upper}
Suppose that the model class $\cF$ satisfies Assumption \ref{assm:well_upper}. Let $\cI_T:=\cI_T(\beta^{\star})$ and $\cI_S:=\cI_S(\beta^{\star})$.
For any $\delta\in (0,1)$, if $n\geq c \max\{N^{\star}\log(d/\delta), N(\delta)\}$, then with probability at least $1-2\delta$, we have
\begin{align*}
R_{\beta^{\star}}(\beta_{\MLE}) 
\leq c\frac{\Tr\left(\cI_T\cI^{-1}_S\right)\log \frac{d}{\delta}}{n}
\end{align*}
for an absolute constant $c$. Here
\begin{align*}
N^{\star}:=  (1 + \tilde{\kappa}/\kappa)^2 \cdot \max\left\{\tilde{\kappa}^{-1}\alpha_1^2\log^{2\gamma} \left((1+\tilde{\kappa}/\kappa)\tilde{\kappa}^{-1}\alpha^2_1\right),\  \alpha_2^2, \  
\tilde{\kappa}(1 + \|\cI_T^{\frac12} \cI_S^{-1}\cI_T^{\frac12}\|_2^{-2})\alpha_3^2 \right\},
\end{align*}
where $\alpha_1 := B_1 \|\cI_S^{-1}\|_2^{1/2}$, $\alpha_2 := B_2 \|\cI_S^{-1}\|_2$, $\alpha_3 := B_3 \|\cI_S^{-1}\|_2^{3/2}$,
\begin{align*}
    \kappa:=\frac{\Tr(\cI_{T} \cI_{S}^{-1}) }{\|\cI_T^{\frac12} \cI_S^{-1}\cI_T^{\frac12}\|_2},\, \, \tilde{\kappa} := \frac{\Tr(\cI_S^{-1})}{\|\cI_S^{-1}\|_2}.
\end{align*}
\end{theorem}

For proving Theorem \ref{thm_exact:well_upper}, we first state two main lemmas. Informally speaking, Lemma \ref{claim1} and Lemma \ref{claim2} capture the distance between $\beta_{\MLE}$ and $\beta^{\star}$ under different measurements. 

\begin{lemma}\label{claim1}
Suppose Assumption \ref{assm:well_upper} holds. For any $\delta\in (0,1)$ and any $n\geq c\max\{N_1\log(d/\delta), N(\delta)\}$, with probability at least $1-\delta$, we have $\beta_{\MLE}\in\bbB_{\beta^{\star}}(c\sqrt{\frac{\Tr(\cI_S^{-1})\log \frac{d}{\delta}}{n}})$ for some absolute constant $c$.
Here 
\begin{align*}
N_1:=\max\bigg\{
&B^2_2\|\cI_{S}^{-1}\|^2_{2}, B^2_3\|\cI_{S}^{-1}\|^2_{2}\Tr(\cI_{S}^{-1}),
\left(\frac{B^2_1B_2\|\cI_{S}^{-1}\|_{2}^3\log^{2\gamma} (\tilde{\kappa}^{-1/2}\alpha_1)}{\Tr(\cI_{S}^{-1})}\right)^{\frac23},\notag\\
&\left(\frac{B^3_1B_3\|\cI_{S}^{-1}\|_{2}^4\log^{3\gamma} (\tilde{\kappa}^{-1/2}\alpha_1)}{\Tr(\cI_{S}^{-1})}\right)^{\frac12},
\frac{B^2_1\|\cI_{S}^{-1}\|_{2}^2\log^{2\gamma} (\tilde{\kappa}^{-1/2}\alpha_1)}{\Tr(\cI_{S}^{-1})}
\bigg\} .   
\end{align*}
\end{lemma}

\begin{lemma}\label{claim2}
Suppose Assumption \ref{assm:well_upper} holds. For any $\delta\in (0,1)$ and any $n\geq c\max\{N_1\log(d/\delta), N_2\log(d/\delta), N(\delta)\}$, with probability at least $1-2\delta$, we have
\begin{align*}
\|\cI_{T}^{\frac{1}{2}}(\beta_{\MLE}-\beta^{\star})\|_{2}^{2}\leq c\frac{\Tr(\cI_{T} \cI_{S}^{-1}) \log \frac{d}{\delta}}{n}
\end{align*}
for some absolute constant $c$. Here $N_1$ is defined in Lemma \ref{claim1} and
\begin{align*}
N_2:=\max\bigg\{
&\left(\frac{B_2\|\cI_{T}^{\frac{1}{2}}\cI_{S}^{-\frac{1}{2}}\|_{2}^{2}\Tr(\cI_{S}^{-1})}{\Tr(\cI_{T} \cI_{S}^{-1}) }\right)^2,
\left(\frac{B_3\|\cI_{T}^{\frac{1}{2}}\cI_{S}^{-\frac{1}{2}}\|_{2}^{2}\Tr(\cI_{S}^{-1})^{1.5}}{\Tr(\cI_{T} \cI_{S}^{-1}) }\right)^2,\notag\\
&\left(\frac{B^2_1B_2\|\cI_{T}^{\frac{1}{2}}\cI_{S}^{-\frac{1}{2}}\|_{2}^{2}\|\cI_{S}^{-1}\|_{2}^2\log^{2\gamma} (\tilde{\kappa}^{-1/2}\alpha_1)}{\Tr(\cI_{T} \cI_{S}^{-1}) }\right)^{\frac23},
\left(\frac{B^3_1B_3\|\cI_{T}^{\frac{1}{2}}\cI_{S}^{-\frac{1}{2}}\|_{2}^{2}\|\cI_{S}^{-1}\|_{2}^3\log^{3\gamma} (\tilde{\kappa}^{-1/2}\alpha_1)}{\Tr(\cI_{T} \cI_{S}^{-1}) }\right)^{\frac12},\notag\\
&\frac{B^2_{1}\|\cI_{T}^{\frac{1}{2}}\cI_{S}^{-\frac{1}{2}}\|_{2}^{2}\|\cI_{S}^{-1}\|_{2}\log^{2\gamma} ({\kappa}^{-1/2}\alpha_1)}{\Tr(\cI_{T} \cI_{S}^{-1}) }
\bigg\}.
\end{align*}
\end{lemma}
The proofs for Lemma \ref{claim1} and \ref{claim2} are delayed to the end of this subsection. With these two lemmas, we can now state the proof for Theorem \ref{thm_exact:well_upper}.
\begin{proof}[Proof of Theorem \ref{thm_exact:well_upper}]
By Assumption \ref{assm2}, we can do Taylor expansion w.r.t. $\beta$ as the following:
\begin{align*}
R_{\beta^{\star}}(\beta_{\MLE})&=  
\bbE_{\substack{x\sim\bbP_T(X)\\y|x\sim f(y|x;\beta^{\star})}}\left[\ell(x,y,\beta_{\MLE})-\ell(x,y,\beta^{\star})\right] \notag \\ &\leq  \bbE_{\substack{x\sim\bbP_T(X)\\y|x\sim f(y|x;\beta^{\star})}} [\nabla \ell(x,y,\beta^{\star})]^{T}(\beta_{\MLE}-\beta^{\star}) \notag \\ 
&+ \frac{1}{2}(\beta_{\MLE}-\beta^{\star})^{T}\cI_{T}(\beta_{\MLE}-\beta^{\star}) +\frac{B_{3}}{6}\|\beta_{\MLE}-\beta^{\star}\|_{2}^{3}. 
\end{align*}
Applying Lemma \ref{claim1} and \ref{claim2}, we know for any $\delta$ and any $n\geq c\max\{N_1\log(d/\delta), N_2\log(d/\delta), N(\delta)\}$, with probability at least $1-2\delta$, we have
\begin{align*}
(\beta_{\MLE}-\beta^{\star})^{T}\cI_{T}(\beta_{\MLE}-\beta^{\star}) \leq c\frac{\Tr(\cI_{T} \cI_{S}^{-1}) \log \frac{d}{\delta}}{n}
\end{align*} and
\begin{align*}
   \|\beta_{\MLE}-\beta^{\star}\|_{2} \leq c\sqrt{\frac{ \Tr(\cI_S^{-1})\log \frac{d}{\delta}}{n}}.
\end{align*}
Also notice that, $\bbE_{\substack{x\sim\bbP_T(X)\\y|x\sim f(y|x;\beta^{\star})}} [\nabla \ell(x,y,\beta^{\star})]=0$. Therefore, with probability at least $1-2\delta$, we have
\begin{align*} 
R_{\beta^{\star}}(\beta_{\MLE})  \leq \frac{c}{2}\frac{\Tr(\cI_{T} \cI_{S}^{-1}) \log \frac{d}{\delta}}{n} + \frac{c^3}{6}B_3\Tr(\cI_S^{-1})^{1.5}(\frac{\log \frac{d}{\delta}}{n})^{1.5}
\end{align*}
for any $\delta$ and any $n\geq c\max\{N_1\log(d/\delta), N_2\log(d/\delta), N(\delta)\}$. If we further assume $n\geq c(\frac{B_3\Tr(\cI_S^{-1})^{1.5}}{\Tr(\cI_{T} \cI_{S}^{-1}) })^{2}\log(d/\delta)$, it then holds that
\begin{align*} 
R_{\beta^{\star}}(\beta_{\MLE})  \leq c \frac{\Tr(\cI_{T} \cI_{S}^{-1}) \log \frac{d}{\delta}}{n}.
\end{align*}
Note that
\begin{align*}
&\max\bigg\{N_1,N_2,
\left(\frac{B_3\Tr(\cI_S^{-1})^{1.5}}{\Tr(\cI_{T} \cI_{S}^{-1}) }\right)^{2}
\bigg\}\\
&=\max\bigg\{
B^2_2\|\cI_{S}^{-1}\|^2_{2}, B^2_3\|\cI_{S}^{-1}\|^2_{2}\Tr(\cI_{S}^{-1}),
\left(\frac{B^2_1B_2\|\cI_{S}^{-1}\|_{2}^3\log^{2\gamma} (\tilde{\kappa}^{-1/2}\alpha_1)}{\Tr(\cI_{S}^{-1})}\right)^{\frac23},
\left(\frac{B^3_1B_3\|\cI_{S}^{-1}\|_{2}^4\log^{3\gamma} (\tilde{\kappa}^{-1/2}\alpha_1)}{\Tr(\cI_{S}^{-1})}\right)^{\frac12},\notag\\
&\quad\frac{B^2_1\|\cI_{S}^{-1}\|_{2}^2\log^{2\gamma} (\tilde{\kappa}^{-1/2}\alpha_1)}{\Tr(\cI_{S}^{-1})},
\left(\frac{B_2\|\cI_{T}^{\frac{1}{2}}\cI_{S}^{-\frac{1}{2}}\|_{2}^{2}\Tr(\cI_{S}^{-1})}{\Tr(\cI_{T} \cI_{S}^{-1}) }\right)^2,
\left(\frac{B_3\|\cI_{T}^{\frac{1}{2}}\cI_{S}^{-\frac{1}{2}}\|_{2}^{2}\Tr(\cI_{S}^{-1})^{1.5}}{\Tr(\cI_{T} \cI_{S}^{-1}) }\right)^2,\notag\\
&\quad\left(\frac{B^2_1B_2\|\cI_{T}^{\frac{1}{2}}\cI_{S}^{-\frac{1}{2}}\|_{2}^{2}\|\cI_{S}^{-1}\|_{2}^2\log^{2\gamma} (\tilde{\kappa}^{-1/2}\alpha_1)}{\Tr(\cI_{T} \cI_{S}^{-1}) }\right)^{\frac23},
\left(\frac{B^3_1B_3\|\cI_{T}^{\frac{1}{2}}\cI_{S}^{-\frac{1}{2}}\|_{2}^{2}\|\cI_{S}^{-1}\|_{2}^3\log^{3\gamma} (\tilde{\kappa}^{-1/2}\alpha_1)}{\Tr(\cI_{T} \cI_{S}^{-1}) }\right)^{\frac12},\notag\\
&\quad\frac{B^2_{1}\|\cI_{T}^{\frac{1}{2}}\cI_{S}^{-\frac{1}{2}}\|_{2}^{2}\|\cI_{S}^{-1}\|_{2}\log^{2\gamma} ({\kappa}^{-1/2}\alpha_1)}{\Tr(\cI_{T} \cI_{S}^{-1}) },
\left(\frac{B_3\Tr(\cI_S^{-1})^{1.5}}{\Tr(\cI_{T} \cI_{S}^{-1}) }\right)^{2}
\bigg\}\notag\\
&=\max\bigg\{
\alpha_2^2, \tilde{\kappa} \alpha^2_3,
\alpha_1^{4/3}\alpha_2^{2/3} \tilde{\kappa} ^{-2/3}\log^{4\gamma/3} (\tilde{\kappa}^{-1/2}\alpha_1), 
\alpha_1^{3/2}\alpha_3^{1/2} \tilde{\kappa} ^{-1/2}\log^{3\gamma/2} (\tilde{\kappa}^{-1/2}\alpha_1),
\alpha_1^2\tilde{\kappa}^{-1}\log^{2\gamma} (\tilde{\kappa}^{-1/2}\alpha_1),
\notag\\
& 
\quad\quad\quad\quad
\alpha_2^2 (\tilde{\kappa}/\kappa)^2, \alpha_3^2 \tilde{\kappa}^3/\kappa^2, 
\alpha_1^{4/3}\alpha_2^{2/3}\kappa ^{-2/3}\log^{4\gamma/3} (\tilde{\kappa}^{-1/2}\alpha_1), 
\alpha_1^{3/2}\alpha_3^{1/2} \kappa ^{-1/2}\log^{3\gamma/2} (\tilde{\kappa}^{-1/2}\alpha_1),
\notag\\
& 
\quad\quad\quad\quad
\alpha_1^2\kappa^{-1}\log^{2\gamma} ({\kappa}^{-1/2}\alpha_1),
\alpha_3^2 \tilde{\kappa}^3 \kappa^{-2} \|\cI_T^{\frac12} \cI_S^{-1}\cI_T^{\frac12}\|_2^{-2}
\bigg\}\notag\\
&\leq \max\bigg\{
\tilde{\kappa}^{-1}\alpha_1^2\log^{2\gamma} \left((1+\tilde{\kappa}/\kappa)\tilde{\kappa}^{-1}\alpha^2_1\right),
\kappa^{-1}\alpha_1^2\log^{2\gamma} \left((1+\tilde{\kappa}/\kappa)\tilde{\kappa}^{-1}\alpha^2_1\right),
\alpha_2^2, 
(\tilde{\kappa}/\kappa)^2\alpha_2^2 ,\notag\\
&\quad\quad\quad\quad
\tilde{\kappa} \alpha^2_3,
(\tilde{\kappa}^3/\kappa^2)\alpha_3^2 ,
\tilde{\kappa}^3 \kappa^{-2} \|\cI_T^{\frac12} \cI_S^{-1}\cI_T^{\frac12}\|_2^{-2}\alpha_3^2 
\bigg\}\notag\\
&\leq (1 + \tilde{\kappa}/\kappa)^2 \cdot \max\{\tilde{\kappa}^{-1}\alpha_1^2\log^{2\gamma} \left((1+\tilde{\kappa}/\kappa)\tilde{\kappa}^{-1}\alpha^2_1\right), \alpha_2^2,  \tilde{\kappa}(1 + \|\cI_T^{\frac12} \cI_S^{-1}\cI_T^{\frac12}\|_2^{-2})\alpha_3^2 \} \notag\\
&=:N^{\star}.
\end{align*}
To summarize, for any $\delta$, any $n\geq c\max\{N^{\star}\log(d/\delta), N(\delta)\}$, with probability at least $1-2\delta$, we have
\begin{align*} 
R_{\beta^{\star}}(\beta_{\MLE})  \leq c \frac{\Tr(\cI_{T} \cI_{S}^{-1}) \log \frac{d}{\delta}}{n}.
\end{align*}

\end{proof}

In the following, we prove Lemma \ref{claim1} and \ref{claim2}.

\paragraph{Proof of Lemma \ref{claim1}}
\begin{proof}[Proof of Lemma \ref{claim1}]
For notation simplicity, we denote $g:=\nabla\ell_n(\beta^{\star})-\bbE[\nabla\ell_n(\beta^{\star})]$.
Note that 
\begin{align*}
V &
= n \cdot \mathbb{E} [\|A(\nabla\ell_n(\beta^{\star})-\bbE[\nabla\ell_n(\beta^{\star})])\|_2^2]\notag\\
&=n\cdot\bbE[\nabla\ell_n(\beta^{\star})^{T}A^TA\nabla\ell_n(\beta^{\star})]\notag\\
&=n\cdot\bbE[\Tr(A\nabla\ell_n(\beta^{\star})\nabla\ell_n(\beta^{\star})^{T}A^T)]\notag\\
&=\Tr(A\cI_SA^{T}).
\end{align*}
By taking $A= \cI_{S}^{-1}$ in Assumption \ref{assm1}, for any $\delta$, any $n > N(\delta)$, we have with probability at least $1-\delta$:
\begin{align}
\label{ineq:pf:concentration1}
&\|\cI_{S}^{-1}g\|_{2} 
\leq c\sqrt{\frac{\Tr(\cI_{S}^{-1}) \log \frac{d}{\delta}}{n}}+ B_{1} \|\cI_{S}^{-1}\|_2 \log^\gamma\left(\frac{B_{1} \|\cI_{S}^{-1}\|_2}{\sqrt{\Tr(\cI_{S}^{-1})}}\right) \frac{ \log \frac{d}{\delta}}{n}\notag\\
&\qquad\qquad =c\sqrt{\frac{\Tr(\cI_{S}^{-1}) \log \frac{d}{\delta}}{n}}+ B_{1} \|\cI_{S}^{-1}\|_2 \log^{\gamma} (\tilde{\kappa}^{-1/2}\alpha_1)\frac{ \log \frac{d}{\delta}}{n},\\
\label{ineq:pf:concentration3}
&\left\|\nabla^{2} \ell_{n}(\beta^{\star})-\bbE[\nabla^{2} \ell_{n}(\beta^{\star})]\right\|_2 
\leq B_2 \sqrt{\frac{\log \frac{d}{\delta}}{n}}.
\end{align}
Let event $A:=\{\eqref{ineq:pf:concentration1},\eqref{ineq:pf:concentration3}\text{ holds}\}$. 
Under the event $A$, we have the following Taylor expansion:
\begin{align} \label{ineq:pf:taylor1}
\ell_{n}(\beta) - \ell_{n}(\beta^{\star}) 
&\stackrel{\text{by Assumption \ref{assm2}}}{\leq}  (\beta - \beta^{\star})^{T} \nabla \ell_{n} (\beta^{\star}) +  \frac{1}{2} (\beta - \beta^{\star})^{T} \nabla^{2} \ell_{n} (\beta^{\star}) (\beta - \beta^{\star}) + \frac{B_{3}}{6} \|\beta-\beta^{\star}\|_{2}^{3} \notag \\
& \stackrel{\nabla \ell(\beta^{\star})=0}{=} (\beta - \beta^{\star})^{T} g  +  \frac{1}{2} (\beta - \beta^{\star})^{T} \nabla^{2} \ell_{n} (\beta^{\star}) (\beta - \beta^{\star}) + \frac{B_{3}}{6} \|\beta-\beta^{\star}\|_{2}^{3} \notag \\
& \stackrel{\text{by (\ref{ineq:pf:concentration3})}}{\leq} (\beta - \beta^{\star})^{T} g + \frac{1}{2} (\beta - \beta^{\star})^{T} \cI_{S} (\beta - \beta^{\star}) + B_{2}\sqrt{\frac{\log \frac{d}{\delta}}{n}} \|\beta-\beta^{\star}\|_{2}^{2} + \frac{B_{3}}{6} \|\beta-\beta^{\star}\|_{2}^{3}  \notag \\
& \stackrel{\Delta_{\beta}:=\beta-\beta^{\star}}{=} \Delta_{\beta}^{T}g + \frac{1}{2}\Delta_{\beta}^{T} \cI_{S} \Delta_{\beta} + B_{2}\sqrt{\frac{\log \frac{d}{\delta}}{n}} \|\Delta_{\beta}\|_{2}^{2} + \frac{B_{3}}{6} \|\Delta_{\beta}\|_{2}^{3} \notag \\
&= \frac{1}{2}(\Delta_{\beta}-z)^{T} \cI_{S} (\Delta_{\beta}-z) - \frac{1}{2}z^{T}\cI_{S}z + B_{2}\sqrt{\frac{\log \frac{d}{\delta}}{n}} \|\Delta_{\beta}\|_{2}^{2}+ \frac{B_{3}}{6} \|\Delta_{\beta}\|_{2}^{3}
\end{align} 
where $z:=-\cI_{S}^{-1}g$. Similarly
\begin{align} \label{ineq:pf:taylor2}
\ell_{n}(\beta) - \ell_{n}(\beta^{\star}) \geq \frac{1}{2}(\Delta_{\beta}-z)^{T} \cI_{S} (\Delta_{\beta}-z) - \frac{1}{2}z^{T}\cI_{S}z - B_{2}\sqrt{\frac{\log \frac{d}{\delta}}{n}} \|\Delta_{\beta}\|_{2}^{2} - \frac{B_{3}}{6} \|\Delta_{\beta}\|_{2}^{3}.
\end{align}
Notice that $\Delta_{\beta^{\star}+z} = z$, by (\ref{ineq:pf:concentration1}) and (\ref{ineq:pf:taylor1}), we have
\begin{align} \label{ineq:pf:taylor3}
\ell_{n}(\beta^{\star}+z)- \ell_{n}(\beta^{\star}) 
&\leq - \frac{1}{2}z^{T}\cI_{S}z+B_{2}\sqrt{\frac{\log \frac{d}{\delta}}{n}}\left(c\sqrt{\frac{\Tr(\cI_{S}^{-1}) \log \frac{d}{\delta}}{n}}+ B_{1} \|\cI_{S}^{-1}\|_2 \log^{\gamma} (\tilde{\kappa}^{-1/2}\alpha_1)\frac{ \log \frac{d}{\delta}}{n}\right)^{2} \notag\\
&\quad+ \frac{B_{3}}{6}\left(c\sqrt{\frac{\Tr(\cI_{S}^{-1}) \log \frac{d}{\delta}}{n}}+ B_{1} \|\cI_{S}^{-1}\|_2 \log^{\gamma} (\tilde{\kappa}^{-1/2}\alpha_1)\frac{ \log \frac{d}{\delta}}{n}\right)^{3}\notag\\
&\leq - \frac{1}{2}z^{T}\cI_{S}z+2c^2B_2 \Tr(\cI_{S}^{-1})(\frac{\log \frac{d}{\delta}}{n})^{1.5}+2B^2_1B_2\|\cI_{S}^{-1}\|_{2}^2\log^{2\gamma} (\tilde{\kappa}^{-1/2}\alpha_1)(\frac{\log \frac{d}{\delta}}{n})^{2.5}\notag\\
&\quad\quad + \frac{2}{3}c^3B_3 \Tr(\cI_{S}^{-1})^{1.5}(\frac{\log \frac{d}{\delta}}{n})^{1.5}+\frac{2}{3}B^3_1B_3\|\cI_{S}^{-1}\|_{2}^3\log^{3\gamma} (\tilde{\kappa}^{-1/2}\alpha_1)(\frac{\log \frac{d}{\delta}}{n})^{3},
\end{align}
where we use the fact that $(a+b)^n\leq 2^{n-1}(a^n+b^n)$ in the last inequality. For any $\beta \in \bbB_{\beta^{\star}}(3c\sqrt{\frac{\Tr(\cI_{S}^{-1})\log\frac{d}{\delta}}{n}})$, by (\ref{ineq:pf:taylor2}), we have
\begin{align} \label{ineq:pf:taylor4}
\ell_{n}(\beta)-\ell_{n}(\beta^{\star}) &\geq \frac{1}{2}(\Delta_{\beta}-z)^{T} \cI_{S} (\Delta_{\beta}-z) - \frac{1}{2}z^{T}\cI_{S}z\notag\\
&\quad -9c^2B_2\Tr(\cI_{S}^{-1})(\frac{\log \frac{d}{\delta}}{n})^{1.5}-\frac{9}{2}c^3B_3\Tr(\cI_{S}^{-1})^{1.5}(\frac{\log \frac{d}{\delta}}{n})^{1.5}.
\end{align}
(\ref{ineq:pf:taylor4}) - (\ref{ineq:pf:taylor3}) gives
\begin{align} \label{ineq:pf:taylor5}
\ell_{n}(\beta)-\ell_{n}(\beta^{\star}+z)  
&\geq \frac{1}{2}(\Delta_{\beta}-z)^{T} \cI_{S} (\Delta_{\beta}-z) \notag \\
&- \bigg(9c^2 B_2\Tr(\cI_{S}^{-1})(\frac{\log \frac{d}{\delta}}{n})^{1.5}+\frac{9}{2}c^3B_3\Tr(\cI_{S}^{-1})^{1.5}(\frac{\log \frac{d}{\delta}}{n})^{1.5}\notag\\
&\quad+2c^2B_2 \Tr(\cI_{S}^{-1})(\frac{\log \frac{d}{\delta}}{n})^{1.5}+2B^2_1B_2\|\cI_{S}^{-1}\|_{2}^2\log^{2\gamma} (\tilde{\kappa}^{-1/2}\alpha_1) (\frac{\log \frac{d}{\delta}}{n})^{2.5}\notag\\
&\quad\quad + \frac{2}{3}c^3 B_3 \Tr(\cI_{S}^{-1})^{1.5}(\frac{\log \frac{d}{\delta}}{n})^{1.5}+\frac{2}{3}B^3_1B_3\|\cI_{S}^{-1}\|_{2}^3\log^{3\gamma} (\tilde{\kappa}^{-1/2}\alpha_1)(\frac{\log \frac{d}{\delta}}{n})^{3}
\bigg)\notag\\
& = \frac{1}{2}(\Delta_{\beta}-z)^{T} \cI_{S} (\Delta_{\beta}-z) \notag \\
& - \bigg(11c^2B_2\Tr(\cI_{S}^{-1})(\frac{\log \frac{d}{\delta}}{n})^{1.5}+\frac{31}{6}c^3B_3\Tr(\cI_{S}^{-1})^{1.5}(\frac{\log \frac{d}{\delta}}{n})^{1.5}\notag\\
&\quad+2B^2_1B_2\|\cI_{S}^{-1}\|_{2}^2\log^{2\gamma} (\tilde{\kappa}^{-1/2}\alpha_1)(\frac{\log \frac{d}{\delta}}{n})^{2.5}+ \frac{2}{3}B^3_1B_3\|\cI_{S}^{-1}\|_{2}^3\log^{3\gamma} (\tilde{\kappa}^{-1/2}\alpha_1)(\frac{\log \frac{d}{\delta}}{n})^{3}
\bigg)
\end{align}
Consider the ellipsoid 
\begin{align*}
\cD:=\bigg\{ \beta \in \bbR^{d} \,\bigg| \,&\frac{1}{2}(\Delta_{\beta}-z)^{T} \cI_{S} (\Delta_{\beta}-z) \notag \\
&\leq 11c^2 B_2\Tr(\cI_{S}^{-1})(\frac{\log \frac{d}{\delta}}{n})^{1.5}+\frac{31}{6}c^3B_3\Tr(\cI_{S}^{-1})^{1.5}(\frac{\log \frac{d}{\delta}}{n})^{1.5}\notag\\
&\quad+2B^2_1B_2\|\cI_{S}^{-1}\|_{2}^2\log^{2\gamma} (\tilde{\kappa}^{-1/2}\alpha_1)(\frac{\log \frac{d}{\delta}}{n})^{2.5}+\frac{2}{3} B^3_1B_3\|\cI_{S}^{-1}\|_{2}^3\log^{3\gamma} (\tilde{\kappa}^{-1/2}\alpha_1)(\frac{\log \frac{d}{\delta}}{n})^{3}
\bigg\}    .
\end{align*}
Then by (\ref{ineq:pf:taylor5}), for any $\beta \in \bbB_{\beta^{\star}}(3c\sqrt{\frac{\Tr(\cI_{S}^{-1})\log\frac{d}{\delta}}{n}}) \cap \cD^{C}$, 
\begin{align} \label{ineq:pf:ellipsoid}
\ell_{n}(\beta)-\ell_{n}(\beta^{\star}+z) > 0.    
\end{align}
Notice that by the definition of $\cD$, using $\lambda_{\min}^{-1}(\cI_{S})= \|\cI_{S}^{-1}\|_{2}$, we have for any $\beta \in \cD$,
\begin{align*}
\|\Delta_{\beta}-z\|_{2}^{2} &\leq 22c^2B_2\|\cI_{S}^{-1}\|_{2}\Tr(\cI_{S}^{-1})(\frac{\log \frac{d}{\delta}}{n})^{1.5}+\frac{31}{3}c^3B_3\|\cI_{S}^{-1}\|_{2}\Tr(\cI_{S}^{-1})^{1.5}(\frac{\log \frac{d}{\delta}}{n})^{1.5}\notag\\
&\quad+4B^2_1B_2\|\cI_{S}^{-1}\|_{2}^3\log^{2\gamma} (\tilde{\kappa}^{-1/2}\alpha_1) (\frac{\log \frac{d}{\delta}}{n})^{2.5}+ \frac{4}{3}B^3_1B_3\|\cI_{S}^{-1}\|_{2}^4\log^{3\gamma} (\tilde{\kappa}^{-1/2}\alpha_1)(\frac{\log \frac{d}{\delta}}{n})^{3}.
\end{align*}
Thus for any $\beta \in \cD$, we have
\begin{align*}
\|\Delta_{\beta}\|_{2}^{2} &\leq 2(\|\Delta_{\beta}-z\|_{2}^{2}+\|z\|_{2}^{2})    \\
& \stackrel{\text{by} (\ref{ineq:pf:concentration1})}{\leq}
44c^2B_2\|\cI_{S}^{-1}\|_{2}\Tr(\cI_{S}^{-1})(\frac{\log \frac{d}{\delta}}{n})^{1.5}+\frac{62}{3}c^3B_3\|\cI_{S}^{-1}\|_{2}\Tr(\cI_{S}^{-1})^{1.5}(\frac{\log \frac{d}{\delta}}{n})^{1.5}\notag\\
&\quad+8B^2_1B_2\|\cI_{S}^{-1}\|_{2}^3\log^{2\gamma} (\tilde{\kappa}^{-1/2}\alpha_1)(\frac{\log \frac{d}{\delta}}{n})^{2.5}+ \frac{8}{3} B^3_1B_3\|\cI_{S}^{-1}\|_{2}^4\log^{3\gamma} (\tilde{\kappa}^{-1/2}\alpha_1)(\frac{\log \frac{d}{\delta}}{n})^{3}\notag\\
&\quad+4c^2\frac{\Tr(\cI_{S}^{-1})\log \frac{d}{\delta}}{n}+4B^2_1\|\cI_{S}^{-1}\|_{2}^2\log^{2\gamma} (\tilde{\kappa}^{-1/2}\alpha_1)(\frac{\log \frac{d}{\delta}}{n})^2
.
\end{align*}
To guarantee $\frac{\Tr(\cI_{S}^{-1})\log \frac{d}{\delta}}{n}$ is the leading term, we only need $\frac{\Tr(\cI_{S}^{-1})\log \frac{d}{\delta}}{n}$ to dominate the rest of the terms. Hence, if we further have $n\geq c N_1\log(d/\delta)$, it then holds that
\begin{align*}
\|\Delta_{\beta}\|_{2}^{2} \leq 9c^2\frac{\Tr(\cI_{S}^{-1})\log \frac{d}{\delta}}{n},
\end{align*}
i.e., $\beta \in \bbB_{\beta^{\star}}(3c\sqrt{\frac{\Tr(\cI_{S}^{-1})\log \frac{d}{\delta}}{n}})$.
Here 
\begin{align*}
N_1:=\max\bigg\{
&B^2_2\|\cI_{S}^{-1}\|^2_{2}, B^2_3\|\cI_{S}^{-1}\|^2_{2}\Tr(\cI_{S}^{-1}),
\left(\frac{B^2_1B_2\|\cI_{S}^{-1}\|_{2}^3\log^{2\gamma} (\tilde{\kappa}^{-1/2}\alpha_1)}{\Tr(\cI_{S}^{-1})}\right)^{\frac23},\notag\\
&\left(\frac{B^3_1B_3\|\cI_{S}^{-1}\|_{2}^4\log^{3\gamma} (\tilde{\kappa}^{-1/2}\alpha_1)}{\Tr(\cI_{S}^{-1})}\right)^{\frac12},
\frac{B^2_1\|\cI_{S}^{-1}\|_{2}^2\log^{2\gamma} (\tilde{\kappa}^{-1/2}\alpha_1)}{\Tr(\cI_{S}^{-1})}
\bigg\} .   
\end{align*}
In other words, we show that $\cD \subset \bbB_{\beta^{\star}}(3c\sqrt{\frac{\Tr(\cI_{S}^{-1})\log \frac{d}{\delta}}{n}})$ when $n\geq c \max\{N_1\log(d/\delta),N(\delta)\}$. 
Recall that by (\ref{ineq:pf:ellipsoid}), we know that for any $\beta \in \bbB_{\beta^{\star}}(3c\sqrt{\frac{\Tr(\cI_{S}^{-1})\log \frac{d}{\delta}}{n}}) \cap \cD^{C}$, 
\begin{align*} 
\ell_{n}(\beta)-\ell_{n}(\beta^{\star}+z) > 0.
\end{align*}
Note that $\beta^{\star}+z\in\cD$. Hence there is a local minimum of $\ell_{n}(\beta)$ in $\cD$. By Assumption \ref{assm3}, we know that the global minimum of $\ell_{n}(\beta)$ is in $\cD$, i.e., 
\begin{align*}
    \beta_{\MLE} \in \cD \subset \bbB_{\beta^{\star}}(3c\sqrt{\frac{\Tr(\cI_{S}^{-1})\log \frac{d}{\delta}}{n}}).
\end{align*}
\end{proof}

\paragraph{Proof of Lemma \ref{claim2}}
\begin{proof}[Proof of Lemma \ref{claim2}]
Let $E:= \{ \beta_{\MLE} \in \cD \subset \bbB_{\beta^{\star}}(c\sqrt{\frac{\Tr(\cI_{S}^{-1})\log \frac{d}{\delta}}{n}}) \}$. For any $\delta$ and any $n\geq c \max\{N_1\log(d/\delta),N(\delta)\}$, by the proof of Lemma \ref{claim1}, we have $\bbP(E) \geq 1-\delta$.

By taking $A= \cI_{T}^{\frac{1}{2}}\cI_{S}^{-1}$ in Assumption \ref{assm1}, for any $\delta$, any $n > N(\delta)$, we have with probability at least $1-\delta$:
\begin{align} \label{ineq:pf:concentration2}
    \|\cI_{T}^{\frac{1}{2}}\cI_{S}^{-1}g\|_{2} 
    &\leq c\sqrt{\frac{\Tr(\cI_{S}^{-1}\cI_{T})\log \frac{d}{\delta}}{n}}+  B_{1} \|\cI_{T}^{\frac{1}{2}}\cI_{S}^{-1}\|_2 \log^\gamma\left(\frac{B_{1} \|\cI_{T}^{\frac{1}{2}}\cI_{S}^{-1}\|_2}{\sqrt{\Tr(\cI_{S}^{-1}\cI_{T})}}\right) \frac{ \log \frac{d}{\delta}}{n}\notag\\
    &\leq c\sqrt{\frac{\Tr(\cI_{S}^{-1}\cI_{T})\log \frac{d}{\delta}}{n}}+  B_{1} \|\cI_{T}^{\frac{1}{2}}\cI_{S}^{-1}\|_2 \log^\gamma(\kappa^{-1/2}\alpha_1) \frac{ \log \frac{d}{\delta}}{n}.
\end{align}
We denote $E':=\{\eqref{ineq:pf:concentration2} \text{ holds}\}$. For any $\delta$ and any $n\geq c \max\{N_1\log(d/\delta),N(\delta)\}$, we have $\bbP(E \cap E') \geq 1-2\delta$.

Under $E \cap E'$, $\beta_{\MLE} \in \cD$, i.e., 
\begin{align*}
&\frac{1}{2}(\Delta_{\beta_{\MLE}}-z)^{T}\cI_{S}(\Delta_{\beta_{\MLE}}-z) \notag\\
&\leq 11c^2B_2\Tr(\cI_{S}^{-1})(\frac{\log \frac{d}{\delta}}{n})^{1.5}+\frac{31}{6}c^3B_3\Tr(\cI_{S}^{-1})^{1.5}(\frac{\log \frac{d}{\delta}}{n})^{1.5}\notag\\
&\quad+2B^2_1B_2\|\cI_{S}^{-1}\|_{2}^2\log^{2\gamma} (\tilde{\kappa}^{-1/2}\alpha_1)(\frac{\log \frac{d}{\delta}}{n})^{2.5}+ \frac{2}{3}B^3_1B_3\|\cI_{S}^{-1}\|_{2}^3\log^{3\gamma} (\tilde{\kappa}^{-1/2}\alpha_1)(\frac{\log \frac{d}{\delta}}{n})^{3}.
\end{align*}
In other words,
\begin{align} \label{ineq:pf:claim2}
&\|\cI_{S}^{\frac{1}{2}}(\Delta_{\beta_{\MLE}}-z)\|_{2}^{2}\notag\\ 
&\leq 22c^2 B_2\Tr(\cI_{S}^{-1})(\frac{\log \frac{d}{\delta}}{n})^{1.5}+\frac{31}{3}c^3B_3\Tr(\cI_{S}^{-1})^{1.5}(\frac{\log \frac{d}{\delta}}{n})^{1.5}\notag\\
&\quad+4B^2_1B_2\|\cI_{S}^{-1}\|_{2}^2\log^{2\gamma} (\tilde{\kappa}^{-1/2}\alpha_1)(\frac{\log \frac{d}{\delta}}{n})^{2.5}+ \frac{4}{3}B^3_1B_3\|\cI_{S}^{-1}\|_{2}^3\log^{3\gamma} (\tilde{\kappa}^{-1/2}\alpha_1)(\frac{\log \frac{d}{\delta}}{n})^{3}
\end{align}
Thus we have
\begin{align*}
&\|\cI_{T}^{\frac{1}{2}}(\beta_{\MLE}-\beta^{\star})\|_{2}^{2}\notag\\ 
& = \|\cI_{T}^{\frac{1}{2}}\Delta_{\beta_{\MLE}}\|_{2}^{2} \notag \\
& = \|\cI_{T}^{\frac{1}{2}}(\Delta_{\beta_{\MLE}}-z) +\cI_{T}^{\frac{1}{2}}z \|_{2}^{2} \notag \\
& \leq 2\|\cI_{T}^{\frac{1}{2}}(\Delta_{\beta_{\MLE}}-z)\|_{2}^{2} + 2\|\cI_{T}^{\frac{1}{2}}z \|_{2}^{2} \notag \\
& = 2\|\cI_{T}^{\frac{1}{2}}\cI_{S}^{-\frac{1}{2}} (\cI_{S}^{\frac{1}{2}}
(\Delta_{\beta_{\MLE}}-z))\|_{2}^{2} + 2\|\cI_{T}^{\frac{1}{2}}\cI_{S}^{-1}g \|_{2}^{2} \notag \\
& \leq 2\|\cI_{T}^{\frac{1}{2}}\cI_{S}^{-\frac{1}{2}}\|_{2}^{2}\|\cI_{S}^{\frac{1}{2}}
(\Delta_{\beta_{\MLE}}-z)\|_{2}^{2} + 2\|\cI_{T}^{\frac{1}{2}}\cI_{S}^{-1}g \|_{2}^{2} \notag \\
& \stackrel{\text{by} (\ref{ineq:pf:claim2}) \text{and} (\ref{ineq:pf:concentration2})}{\leq} 4c^2\frac{\Tr(\cI_{T} \cI_{S}^{-1}) \log \frac{d}{\delta}}{n} \notag \\
&+44c^2B_2\|\cI_{T}^{\frac{1}{2}}\cI_{S}^{-\frac{1}{2}}\|_{2}^{2}\Tr(\cI_{S}^{-1})(\frac{\log \frac{d}{\delta}}{n})^{1.5}
+\frac{62}{3}c^3B_3\|\cI_{T}^{\frac{1}{2}}\cI_{S}^{-\frac{1}{2}}\|_{2}^{2}\Tr(\cI_{S}^{-1})^{1.5}(\frac{\log \frac{d}{\delta}}{n})^{1.5}\notag\\
&+8B^2_1B_2\|\cI_{T}^{\frac{1}{2}}\cI_{S}^{-\frac{1}{2}}\|_{2}^{2}\|\cI_{S}^{-1}\|_{2}^2\log^{2\gamma} (\tilde{\kappa}^{-1/2}\alpha_1)(\frac{\log \frac{d}{\delta}}{n})^{2.5}
+\frac{8}{3} B^3_1B_3\|\cI_{T}^{\frac{1}{2}}\cI_{S}^{-\frac{1}{2}}\|_{2}^{2}\|\cI_{S}^{-1}\|_{2}^3\log^{3\gamma} (\tilde{\kappa}^{-1/2}\alpha_1)(\frac{\log \frac{d}{\delta}}{n})^{3}\notag\\
&\quad\quad +4B^2_{1}\|\cI_{T}^{\frac{1}{2}}\cI_{S}^{-\frac{1}{2}}\|_{2}^{2}\|\cI_{S}^{-1}\|_{2}\log^{2\gamma} ({\kappa}^{-1/2}\alpha_1) (\frac{\log \frac{d}{\delta}}{n})^2
\end{align*}
To guarantee $\frac{\Tr(\cI_{T} \cI_{S}^{-1}) \log \frac{d}{\delta}}{n} $ is the leading term, we only need $\frac{\Tr(\cI_{T} \cI_{S}^{-1}) \log \frac{d}{\delta}}{n}$ to dominate the rest of the terms. Hence, if we further have $n\geq cN_2\log(d/\delta)$, we have
\begin{align*}
\|\cI_{T}^{\frac{1}{2}}(\beta_{\MLE}-\beta^{\star})\|_{2}^{2}\leq 9c^2\frac{\Tr(\cI_{T} \cI_{S}^{-1}) \log \frac{d}{\delta}}{n}.
\end{align*}
Here 
\begin{align*}
N_2:=\max\bigg\{
&\left(\frac{B_2\|\cI_{T}^{\frac{1}{2}}\cI_{S}^{-\frac{1}{2}}\|_{2}^{2}\Tr(\cI_{S}^{-1})}{\Tr(\cI_{T} \cI_{S}^{-1}) }\right)^2,
\left(\frac{B_3\|\cI_{T}^{\frac{1}{2}}\cI_{S}^{-\frac{1}{2}}\|_{2}^{2}\Tr(\cI_{S}^{-1})^{1.5}}{\Tr(\cI_{T} \cI_{S}^{-1}) }\right)^2,\notag\\
&\left(\frac{B^2_1B_2\|\cI_{T}^{\frac{1}{2}}\cI_{S}^{-\frac{1}{2}}\|_{2}^{2}\|\cI_{S}^{-1}\|_{2}^2\log^{2\gamma} (\tilde{\kappa}^{-1/2}\alpha_1)}{\Tr(\cI_{T} \cI_{S}^{-1}) }\right)^{\frac23},
\left(\frac{B^3_1B_3\|\cI_{T}^{\frac{1}{2}}\cI_{S}^{-\frac{1}{2}}\|_{2}^{2}\|\cI_{S}^{-1}\|_{2}^3\log^{3\gamma} (\tilde{\kappa}^{-1/2}\alpha_1)}{\Tr(\cI_{T} \cI_{S}^{-1}) }\right)^{\frac12},\notag\\
&\frac{B^2_{1}\|\cI_{T}^{\frac{1}{2}}\cI_{S}^{-\frac{1}{2}}\|_{2}^{2}\|\cI_{S}^{-1}\|_{2}\log^{2\gamma} ({\kappa}^{-1/2}\alpha_1)}{\Tr(\cI_{T} \cI_{S}^{-1}) }
\bigg\}.
\end{align*}
To summarize, we show that for any $\delta\in (0,1)$ and any $n\geq c\max\{N_1\log(d/\delta), N_2\log(d/\delta), N(\delta)\}$, with probability at least $1-2\delta$, we have
\begin{align*}
\|\cI_{T}^{\frac{1}{2}}(\beta_{\MLE}-\beta^{\star})\|_{2}^{2}\leq 9c^2\frac{\Tr(\cI_{T} \cI_{S}^{-1}) \log \frac{d}{\delta}}{n}.
\end{align*}
\end{proof}

\subsection{Proofs for Theorem \ref{thm:well_lower}}
The detailed version of Theorem \ref{thm:well_lower} is stated as the following.
\begin{theorem}\label{thm_exact:well_lower}
Suppose the model class $\cF$ satisfies Assumption \ref{assm:well_lower}. Then we have
\begin{align*}
&\inf_{\hat\beta}\sup_{\beta^{\star}\in\bbB_{\beta_0}(B)}\Tr\left(\cI_T(\beta^{\star})\cI^{-1}_S(\beta^{\star})\right)^{-1}\bbE_{\substack{x_i\sim\bbP_S(X)\\y_i|x_i\sim f(y|x;\beta^{\star})}}\left[R_{\beta^{\star}}(\hat\beta)\right]\\
&\quad\geq \frac{1}{16}\cdot\frac{1}{2n+\frac{\pi^2d}{R^2_1}\Tr\left(\cI_T(\beta_0)\cI^{-2}_S(\beta_0)\right)\Tr\left(\cI_T(\beta_0)\cI^{-1}_S(\beta_0)\right)^{-1}},
\end{align*}
where
\begin{align*}
R_1:=\frac14\sqrt{\frac{\lambda_{\min}(\cI_T(\beta_0))}{\lambda_{\max}(\cI_T(\beta_0))}}\cdot\min\left\{\frac{\lambda^2_{\min}(\cI_S(\beta_0))}{4L_S\lambda_{\max}(\cI_S(\beta_0))},\frac{\lambda_{\min}(\cI_T(\beta_0))}{4B_3+2L_T},B\right\}.
\end{align*}
\end{theorem}
We first present some useful lemmas that will be used in the proof of Theorem \ref{thm_exact:well_lower}.

\begin{lemma}\label{lem:fisher_inform}
Under Assumptions \ref{assm2}, \ref{assm4} and \ref{assm5}, we can choose $R_0\leq B$ such that for any $\beta,\beta^{\star}\in\bbB_{\beta_0}(R_0)$:
\begin{align}
&\frac12\cdot \cI^{-1}_S(\beta_0)\preceq\cI^{-1}_S(\beta)\preceq 2\cdot\cI^{-1}_S(\beta_0),\label{ineq:fisher_souce}\\
&\frac12 \cdot\cI_T(\beta_0)\preceq\bbE_{\substack{x\sim\bbP_T(X)\\y|x\sim f(y|x;\beta^{\star})}}\left[\nabla^2\ell(x,y,\beta)\right]\preceq 2\cdot\cI_T(\beta_0).\label{ineq:fisher_target}
\end{align}
We can further choose $R_1\leq R_0$ such that for any $\beta^{\star}\in\bbB_{\beta_0}(R_1),\beta\notin\bbB_{\beta_0}(R_0)$: $R_{\beta^{\star}}(\beta)\geq R_{\beta^{\star}}(\beta_0)$.
\end{lemma}

Taking $\beta^{\star}=\beta$, Lemma \ref{lem:fisher_inform} \eqref{ineq:fisher_target} implies for any $\beta\in\bbB_{\beta_0}(R_0)$:
\begin{align}\label{ineq:fisher_target2}
\frac12 \cdot\cI_T(\beta_0)\preceq\cI_T(\beta)\preceq 2\cdot\cI_T(\beta_0).
\end{align}

\begin{lemma}\label{lem:van_trees}
Let $
C_{\beta_0}(B):=\{\beta\in\bbR^d\,|\,\beta-\beta_0\in[-B,B]^d\}$ be a cube around $\beta_0$.
For any $\beta_0\in\bbR^d$ and $B>0$, there exists a prior density $\lambda(\beta)$ supported on $C_{\beta_0}(B)$ such that for any estimator $\hat\beta$, we have
\begin{align*}
&\bbE_{\beta^{\star}\sim\lambda(\beta)}\bbE_{\substack{x_i\sim\bbP_S(X)\\y_i|x_i\sim f(y|x;\beta^{\star})}}\left[(\hat\beta-\beta^{\star})^{T}\cI_T(\beta_0)(\hat\beta-\beta^{\star})\right]\\
&\quad\geq \frac{\Tr\left(\cI_T(\beta_0)\cI^{-1}_S(\beta_0)\right)^2}{n\bbE_{\beta^{\star}\sim\lambda(\beta)}\left[\Tr\left(\cI^{-1}_S(\beta_0)\cI_S(\beta^{\star})\cI^{-1}_S(\beta_0)\cI_T(\beta_0)\right)\right]+\frac{\pi^2}{B^2}\Tr\left(\cI_T(\beta_0)\cI^{-2}_S(\beta_0)\right)}
\end{align*}
\end{lemma}

The proofs for the above lemmas are delivered to the end of this subsection.
With Lemma \ref{lem:fisher_inform} and Lemma \ref{lem:van_trees} in hand, we are now ready to prove Theorem \ref{thm_exact:well_lower}.

\begin{proof}[Proof of Theorem \ref{thm_exact:well_lower}]
For any estimator $\hat\beta$, we define
\begin{align*}
\hat\beta^{p}:=
\begin{cases} 
      \hat\beta & \hat\beta\in\bbB_{\beta_0}(R_0) \\
      \beta_0 & \hat\beta\notin\bbB_{\beta_0}(R_0).
   \end{cases}
\end{align*}
By Lemma \ref{lem:fisher_inform}, for any $\beta^{\star}\in\bbB_{\beta_0}(R_1)$, we have $R_{\beta^{\star}}(\hat\beta)\geq R_{\beta^{\star}}(\hat\beta^p)$.
We then have
\begin{align}\label{ineq:inf_sup1}
&\inf_{\hat\beta}\sup_{\beta^{\star}\in\bbB_{\beta_0}(B)}\Tr\left(\cI_T(\beta^{\star})\cI^{-1}_S(\beta^{\star})\right)^{-1}\bbE_{\substack{x_i\sim\bbP_S(X)\\y_i|x_i\sim f(y|x;\beta^{\star})}}\left[R_{\beta^{\star}}(\hat\beta)\right]\notag\\
&\geq\inf_{\hat\beta}\sup_{\beta^{\star}\in\bbB_{\beta_0}(R_1)}\Tr\left(\cI_T(\beta^{\star})\cI^{-1}_S(\beta^{\star})\right)^{-1}\bbE_{\substack{x_i\sim\bbP_S(X)\\y_i|x_i\sim f(y|x;\beta^{\star})}}\left[R_{\beta^{\star}}(\hat\beta)\right]\notag\\
&\geq \inf_{\hat\beta^p}\sup_{\beta^{\star}\in\bbB_{\beta_0}(R_1)}\Tr\left(\cI_T(\beta^{\star})\cI^{-1}_S(\beta^{\star})\right)^{-1}\bbE_{\substack{x_i\sim\bbP_S(X)\\y_i|x_i\sim f(y|x;\beta^{\star})}}\left[R_{\beta^{\star}}(\hat\beta^p)\right]\notag\\
&\geq \inf_{\hat\beta\in\bbB_{\beta_0}(R_0)}\sup_{\beta^{\star}\in\bbB_{\beta_0}(R_1)}\Tr\left(\cI_T(\beta^{\star})\cI^{-1}_S(\beta^{\star})\right)^{-1}\bbE_{\substack{x_i\sim\bbP_S(X)\\y_i|x_i\sim f(y|x;\beta^{\star})}}\left[R_{\beta^{\star}}(\hat\beta)\right],
\end{align}
where the first inequality follows from the fact that $R_1\leq R_0\leq B$, the second inequality follows from $R_{\beta^{\star}}(\hat\beta)\geq R_{\beta^{\star}}(\hat\beta^p)$, and the third inequality follows from $\hat\beta^p\in\bbB_{\beta_0}(R_0)$.
For any $\beta^{\star}\in\bbB_{\beta_0}(R_1)\subseteq\bbB_{\beta_0}(R_0)$, by \eqref{ineq:fisher_souce} and \eqref{ineq:fisher_target2}, we have
\begin{align*}
\cI_T(\beta^{\star})\preceq 2\cI_T(\beta_0),\quad \cI^{-1}_S(\beta^{\star})\preceq 2\cI^{-1}_S(\beta_0),
\end{align*}
which implies
\begin{align}\label{ineq:trace}
\Tr\left(\cI_T(\beta^{\star})\cI^{-1}_S(\beta^{\star})\right)^{-1} \geq \frac14 \Tr\left(\cI_T(\beta_0)\cI^{-1}_S(\beta_0)\right)^{-1} . 
\end{align}
Combine \eqref{ineq:inf_sup1} and \eqref{ineq:trace}, we have
\begin{align}\label{ineq:inf_sup2}
&\inf_{\hat\beta}\sup_{\beta^{\star}\in\bbB_{\beta_0}(B)}\Tr\left(\cI_T(\beta^{\star})\cI^{-1}_S(\beta^{\star})\right)^{-1}\bbE_{\substack{x_i\sim\bbP_S(X)\\y_i|x_i\sim f(y|x;\beta^{\star})}}\left[R_{\beta^{\star}}(\hat\beta)\right]\notag\\
&\geq \frac14 \Tr\left(\cI_T(\beta_0)\cI^{-1}_S(\beta_0)\right)^{-1}\inf_{\hat\beta\in\bbB_{\beta_0}(R_0)}\sup_{\beta^{\star}\in\bbB_{\beta_0}(R_1)}\bbE_{\substack{x_i\sim\bbP_S(X)\\y_i|x_i\sim f(y|x;\beta^{\star})}}\left[R_{\beta^{\star}}(\hat\beta)\right].
\end{align}
By Taylor expansion, for any $\hat\beta\in\bbB_{\beta_0}(R_0),\beta^{\star}\in\bbB_{\beta_0}(R_1)$, we have
\begin{align*}
R_{\beta^{\star}}(\hat\beta)&=R_{\beta^{\star}}(\beta^{\star})+(\hat\beta-\beta^{\star})^T\bbE_{\substack{x\sim\bbP_T(X)\\y|x\sim f(y|x;\beta^{\star})}}\left[\nabla\ell(x,y,\beta^{\star})\right]\\
&\quad\quad+\frac12 (\hat\beta-\beta^{\star})^T\bbE_{\substack{x\sim\bbP_T(X)\\y|x\sim f(y|x;\beta^{\star})}}\left[\nabla^2\ell(x,y,\tilde\beta)\right](\hat\beta-\beta^{\star})\\
&=\frac12 (\hat\beta-\beta^{\star})^T\bbE_{\substack{x\sim\bbP_T(X)\\y|x\sim f(y|x;\beta^{\star})}}\left[\nabla^2\ell(x,y,\tilde\beta)\right](\hat\beta-\beta^{\star})
\end{align*}
for some $\tilde\beta\in\bbB_{\beta_0}(R_0)$. By Lemma \ref{lem:fisher_inform} \eqref{ineq:fisher_target}, it then holds that
\begin{align}\label{ineq:risk}
 R_{\beta^{\star}}(\hat\beta)\geq \frac14  (\hat\beta-\beta^{\star})^T\cI_T(\beta_0)(\hat\beta-\beta^{\star}) .
\end{align}
By \eqref{ineq:inf_sup2} and \eqref{ineq:risk}, we then have
\begin{align}\label{ineq:inf_sup3}
&\inf_{\hat\beta}\sup_{\beta^{\star}\in\bbB_{\beta_0}(B)}\Tr\left(\cI_T(\beta^{\star})\cI^{-1}_S(\beta^{\star})\right)^{-1}\bbE_{\substack{x_i\sim\bbP_S(X)\\y_i|x_i\sim f(y|x;\beta^{\star})}}\left[R_{\beta^{\star}}(\hat\beta)\right]\notag\\  &\geq \frac{1}{16} \Tr\left(\cI_T(\beta_0)\cI^{-1}_S(\beta_0)\right)^{-1}\inf_{\hat\beta\in\bbB_{\beta_0}(R_0)}\sup_{\beta^{\star}\in\bbB_{\beta_0}(R_1)}\bbE_{\substack{x_i\sim\bbP_S(X)\\y_i|x_i\sim f(y|x;\beta^{\star})}}\left[(\hat\beta-\beta^{\star})^T\cI_T(\beta_0)(\hat\beta-\beta^{\star})\right]\notag\\
&\geq  \frac{1}{16} \Tr\left(\cI_T(\beta_0)\cI^{-1}_S(\beta_0)\right)^{-1}\inf_{\hat\beta\in\bbB_{\beta_0}(R_0)}\sup_{\beta^{\star}\in C_{\beta_0}(\frac{R_1}{\sqrt{d}})}\bbE_{\substack{x_i\sim\bbP_S(X)\\y_i|x_i\sim f(y|x;\beta^{\star})}}\left[(\hat\beta-\beta^{\star})^T\cI_T(\beta_0)(\hat\beta-\beta^{\star})\right],
\end{align}
where the last inequality follows from the fact that $C_{\beta_0}(\frac{R_1}{\sqrt{d}})\subseteq\bbB_{\beta_0}(R_1)$. By Lemma \ref{lem:van_trees}, there exists a prior density $\lambda(\beta)$ supported on $C_{\beta_0}(\frac{R_1}{\sqrt{d}})$ such that for any estimator $\hat\beta$, we have
\begin{align*}
&\bbE_{\beta^{\star}\sim\lambda(\beta)}\bbE_{\substack{x_i\sim\bbP_S(X)\\y_i|x_i\sim f(y|x;\beta^{\star})}}\left[(\hat\beta-\beta^{\star})^{T}\cI_T(\beta_0)(\hat\beta-\beta^{\star})\right]\\
&\geq \frac{\Tr\left(\cI_T(\beta_0)\cI^{-1}_S(\beta_0)\right)^2}{n\bbE_{\beta^{\star}\sim\lambda(\beta)}\left[\Tr\left(\cI^{-1}_S(\beta_0)\cI_S(\beta^{\star})\cI^{-1}_S(\beta_0)\cI_T(\beta_0)\right)\right]+\frac{\pi^2d}{R^2_1}\Tr\left(\cI_T(\beta_0)\cI^{-2}_S(\beta_0)\right)}\\
&\geq \frac{\Tr\left(\cI_T(\beta_0)\cI^{-1}_S(\beta_0)\right)^2}{2n\Tr\left(\cI_T(\beta_0)\cI^{-1}_S(\beta_0)\right)+\frac{\pi^2d}{R^2_1}\Tr\left(\cI_T(\beta_0)\cI^{-2}_S(\beta_0)\right)}.
\end{align*}
Here the last inequality uses the fact that for any $\beta^{\star}\in C_{\beta_0}(\frac{R_1}{\sqrt{d}})\subseteq\bbB_{\beta_0}(R_0)$, by Lemma \ref{lem:fisher_inform} \eqref{ineq:fisher_souce}, we have $\cI^{-1}_S(\beta_0)\preceq 2\cI^{-1}_S(\beta^{\star})$, which implies
\begin{align*}
\bbE_{\beta^{\star}\sim\lambda(\beta)}\left[\Tr\left(\cI^{-1}_S(\beta_0)\cI_S(\beta^{\star})\cI^{-1}_S(\beta_0)\cI_T(\beta_0)\right)\right] &\leq \bbE_{\beta^{\star}\sim\lambda(\beta)}\left[\Tr\left(2\cI^{-1}_S(\beta^{\star})\cI_S(\beta^{\star})\cI^{-1}_S(\beta_0)\cI_T(\beta_0)\right)\right] \\
&=2\Tr\left(\cI_T(\beta_0)\cI^{-1}_S(\beta_0)\right).
\end{align*}
We then conclude for any estimator $\hat\beta$
\begin{align}\label{ineq:van_trees}
&\sup_{\beta^{\star}\in C_{\beta_0}(\frac{R_1}{\sqrt{d}})}\bbE_{\substack{x_i\sim\bbP_S(X)\\y_i|x_i\sim f(y|x;\beta^{\star})}}\left[(\hat\beta-\beta^{\star})^T\cI_T(\beta_0)(\hat\beta-\beta^{\star})\right]\notag\\
&\geq \bbE_{\beta^{\star}\sim\lambda(\beta)}\bbE_{\substack{x_i\sim\bbP_S(X)\\y_i|x_i\sim f(y|x;\beta^{\star})}}\left[(\hat\beta-\beta^{\star})^{T}\cI_T(\beta_0)(\hat\beta-\beta^{\star})\right]\notag\\
&\geq \frac{\Tr\left(\cI_T(\beta_0)\cI^{-1}_S(\beta_0)\right)^2}{2n\Tr\left(\cI_T(\beta_0)\cI^{-1}_S(\beta_0)\right)+\frac{\pi^2d}{R^2_1}\Tr\left(\cI_T(\beta_0)\cI^{-2}_S(\beta_0)\right)}.
\end{align}
Combine \eqref{ineq:inf_sup3} and \eqref{ineq:van_trees}, we have
\begin{align*}
&\inf_{\hat\beta}\sup_{\beta^{\star}\in\bbB_{\beta_0}(B)}\Tr\left(\cI_T(\beta^{\star})\cI^{-1}_S(\beta^{\star})\right)^{-1}\bbE_{\substack{x_i\sim\bbP_S(X)\\y_i|x_i\sim f(y|x;\beta^{\star})}}\left[R_{\beta^{\star}}(\hat\beta)\right]\notag\\
&\geq \frac{1}{16} \Tr\left(\cI_T(\beta_0)\cI^{-1}_S(\beta_0)\right)^{-1}\cdot\frac{\Tr\left(\cI_T(\beta_0)\cI^{-1}_S(\beta_0)\right)^2}{2n\Tr\left(\cI_T(\beta_0)\cI^{-1}_S(\beta_0)\right)+\frac{\pi^2d}{R^2_1}\Tr\left(\cI_T(\beta_0)\cI^{-2}_S(\beta_0)\right)}\\
&=\frac{1}{16}\cdot\frac{1}{2n+\frac{\pi^2d}{R^2_1}\Tr\left(\cI_T(\beta_0)\cI^{-2}_S(\beta_0)\right)\Tr\left(\cI_T(\beta_0)\cI^{-1}_S(\beta_0)\right)^{-1}}.
\end{align*}
Thus we prove Theorem \ref{thm_exact:well_lower}.
\end{proof}

In the following, we prove Lemma \ref{lem:fisher_inform} and Lemma \ref{lem:van_trees}.

\paragraph{Proofs for Lemma \ref{lem:fisher_inform}}
\begin{proof}[Proof of Lemma \ref{lem:fisher_inform}]
We choose 
\begin{align*}
R_0 := \min\left\{\frac{\lambda^2_{\min}(\cI_S(\beta_0))}{4L_S\lambda_{\max}(\cI_S(\beta_0))},\frac{\lambda_{\min}(\cI_T(\beta_0))}{4B_3+2L_T},B\right\},\,\, R_1:=\frac14\sqrt{\frac{\lambda_{\min}(\cI_T(\beta_0))}{\lambda_{\max}(\cI_T(\beta_0))}}\cdot R_0.
\end{align*}
In the sequel, we will show the aforementioned choices of $R_0$ and $R_1$ satisfy the conditions outlined in Lemma \ref{lem:fisher_inform}.

First of all, we show \eqref{ineq:fisher_souce} holds. Fix any $\beta\in\bbB_{\beta_0}(R_0)$. By Assumption \ref{assm4}, we have
\begin{align*}
\|\cI_S(\beta)-\cI_S(\beta_0)\|_2\leq L_S\|\beta-\beta_0\|_2\leq L_S R_0,
\end{align*}
which implies
\begin{align*}
\|\cI^{-1}_S(\beta)-\cI^{-1}_S(\beta_0)\|_2
\leq \|\cI^{-1}_S(\beta_0)\|_2\cdot\|\cI_S(\beta)-\cI_S(\beta_0)\|_2\cdot\|\cI^{-1}_S(\beta)\|_2
\leq \frac{L_S R_0}{\lambda_{\min}(\cI_S(\beta_0))\lambda_{\min}(\cI_S(\beta))}.
\end{align*}
By Weyl's inequality (Lemma 2.2 in \cite{chen2021spectral}), we have
\begin{align*}
\left|\lambda_{\min}(\cI_S(\beta))-\lambda_{\min}(\cI_S(\beta_0))\right|
\leq \|\cI_S(\beta)-\cI_S(\beta_0)\|_2
\leq L_SR_0.
\end{align*}
Note that
\begin{align*}
R_0
\leq \frac{\lambda^2_{\min}(\cI_S(\beta_0))}{4L_S\lambda_{\max}(\cI_S(\beta_0))}
\leq \frac{\lambda_{\min}(\cI_S(\beta_0))}{2L_S}.
\end{align*}
Thus we have
\begin{align*}
\lambda_{\min}(\cI_S(\beta))
\geq \lambda_{\min}(\cI_S(\beta_0))-L_SR_0
\geq \frac12 \lambda_{\min}(\cI_S(\beta_0)),
\end{align*}
which implies 
\begin{align}\label{ineq:pf:lem_fisher1}
\|\cI^{-1}_S(\beta)-\cI^{-1}_S(\beta_0)\|_2
\leq \frac{L_S R_0}{\lambda_{\min}(\cI_S(\beta_0))\lambda_{\min}(\cI_S(\beta))}
\leq \frac{2L_S R_0}{\lambda^2_{\min}(\cI_S(\beta_0))}
\leq \frac{1}{2\lambda_{\max}(\cI_S(\beta_0))}.
\end{align}
Then for any $x\in\bbR^d$, we have
\begin{align*}
x^T\left(\cI^{-1}_S(\beta)-\frac12 \cI^{-1}_S(\beta_0)\right)x
&=\frac12 x^{T}\cI^{-1}_S(\beta_0)x+x^T\left(\cI^{-1}_S(\beta)-\cI^{-1}_S(\beta_0)\right)x\\
&\geq \frac{\|x\|^2_2}{2\lambda_{\max}(\cI_S(\beta_0))}-\|x\|^2_2\cdot\|\cI^{-1}_S(\beta)-\cI^{-1}_S(\beta_0)\|_2\\
&=\|x\|^2_2\left( \frac{1}{2\lambda_{\max}(\cI_S(\beta_0))}-\|\cI^{-1}_S(\beta)-\cI^{-1}_S(\beta_0)\|_2\right)\\
&\geq 0,
\end{align*}
where the last inequality follows from \eqref{ineq:pf:lem_fisher1}. 
Thus we conclude $\cI^{-1}_S(\beta)\succeq\frac12 \cI^{-1}_S(\beta_0)$. Similarly, we can show that $\cI^{-1}_S(\beta)\preceq 2\cI^{-1}_S(\beta_0)$. 
As a result, we show that \eqref{ineq:fisher_souce} holds.

Next, we show \eqref{ineq:fisher_target} holds. Fix any $\beta^{\star}, \beta\in\bbB_{\beta_0}(R_0)$. By Assumption \ref{assm2}, for any $x\in\cX, y\in\cY$, we have
\begin{align*}
\|\nabla^2\ell(x,y,\beta)-\nabla^2\ell(x,y,\beta^{\star})\|_2
\leq B_3\|\beta-\beta^{\star}\|_2
\leq 2B_3R_0,
\end{align*}
which implies 
\begin{align}\label{ineq:pf:lem_fisher2}
&\left\|\bbE_{\substack{x\sim\bbP_T(X)\\y|x\sim f(y|x;\beta^{\star})}}[\nabla^2\ell(x,y,\beta)]-\bbE_{\substack{x\sim\bbP_T(X)\\y|x\sim f(y|x;\beta^{\star})}}[\nabla^2\ell(x,y,\beta^{\star})]\right\|_2\notag\\
&\quad\quad\leq \bbE_{\substack{x\sim\bbP_T(X)\\y|x\sim f(y|x;\beta^{\star})}}[\|\nabla^2\ell(x,y,\beta)-\nabla^2\ell(x,y,\beta^{\star})\|_2]
\leq 2B_3R_0.
\end{align}
By Assumption \ref{assm4}, we have
\begin{align}\label{ineq:pf:lem_fisher3}
  \|\cI_T(\beta^{\star}) - \cI_T(\beta_0)\|_2
  \leq L_T\|\beta^{\star}-\beta_0\|_2
  \leq L_T R_0
\end{align}
Thus, by \eqref{ineq:pf:lem_fisher2} and \eqref{ineq:pf:lem_fisher3}, we have
\begin{align*}
&\left\|\bbE_{\substack{x\sim\bbP_T(X)\\y|x\sim f(y|x;\beta^{\star})}}[\nabla^2\ell(x,y,\beta)]-\cI_T(\beta_0)\right\|_2\\
&\leq \left\|\bbE_{\substack{x\sim\bbP_T(X)\\y|x\sim f(y|x;\beta^{\star})}}[\nabla^2\ell(x,y,\beta)]-\bbE_{\substack{x\sim\bbP_T(X)\\y|x\sim f(y|x;\beta^{\star})}}[\nabla^2\ell(x,y,\beta^{\star})]\right\|_2+\|\cI_T(\beta^{\star}) - \cI_T(\beta_0)\|_2\\
&\leq (2B_3+L_T)R_0\\
&\leq \frac12 \lambda_{\min}(\cI_T(\beta_0)),
\end{align*}
where the last inequality follows from the choice of $R_0$. Consequently, for any $x\in\bbR^{d}$, we have
\begin{align*}
&x^{T}\left(\bbE_{\substack{x\sim\bbP_T(X)\\y|x\sim f(y|x;\beta^{\star})}}[\nabla^2\ell(x,y,\beta)]-\frac12 \cI_T(\beta_0)\right)x\\
&=\frac12 x^T\cI_T(\beta_0) x+x^{T}\left(\bbE_{\substack{x\sim\bbP_T(X)\\y|x\sim f(y|x;\beta^{\star})}}[\nabla^2\ell(x,y,\beta)]-\cI_T(\beta_0)\right)x\\
&\geq \frac12 \|x\|^2_2\lambda_{\min}(\cI_T(\beta_0))-\|x\|^2_2\left\|\bbE_{\substack{x\sim\bbP_T(X)\\y|x\sim f(y|x;\beta^{\star})}}[\nabla^2\ell(x,y,\beta)]-\cI_T(\beta_0)\right\|_2\\
&\geq \frac12 \|x\|^2_2\lambda_{\min}(\cI_T(\beta_0))-\frac12 \|x\|^2_2\lambda_{\min}(\cI_T(\beta_0))=0.
\end{align*}
We then conclude $\bbE_{\substack{x\sim\bbP_T(X)\\y|x\sim f(y|x;\beta^{\star})}}[\nabla^2\ell(x,y,\beta)]\succeq\frac12 \cI_T(\beta_0)$. Similarly, we can show that 
$\bbE_{\substack{x\sim\bbP_T(X)\\y|x\sim f(y|x;\beta^{\star})}}[\nabla^2\ell(x,y,\beta)]\preceq 2\cI_T(\beta_0)$. Thus we show that \eqref{ineq:fisher_target} holds.

Finally, we need to show that for any $\beta^{\star}\in\bbB_{\beta_0}(R_1),\beta\notin\bbB_{\beta_0}(R_0)$: $R_{\beta^{\star}}(\beta)\geq R_{\beta^{\star}}(\beta_0)$. Fix any $\beta^{\star}\in\bbB_{\beta_0}(R_1),\beta\notin\bbB_{\beta_0}(R_0)$. We denote 
\begin{align*}
    \beta':=\left\{\lambda\beta+(1-\lambda)\beta^{\star}\,|\,\lambda\in [0,1]\right\}\cap\left\{\beta'\,|\,\|\beta'-\beta_0\|_2=R_0\right\}.
\end{align*}
By the choice of $R_1$, we know that $R_1\leq R_0/2$, which implies
\begin{align}\label{ineq:pf:lem_fisher4}
\|\beta'-\beta^{\star}\|_2
\geq \|\beta'-\beta_0\|_2-\|\beta_0-\beta^{\star}\|_2
\geq R_0-R_1
\geq \frac{R_0}{2}.
\end{align}
By convexity of $R_{\beta^{\star}}(\cdot)$ assumed in Assumption \ref{assm5} and $R_{\beta^{\star}}(\beta)\geq R_{\beta^{\star}}(\beta^{\star})$, we have $R_{\beta^{\star}}(\beta)\geq R_{\beta^{\star}}(\beta')$. Thus, we obtain
\begin{align}\label{ineq:pf:lem_fisher5}
R_{\beta^{\star}}(\beta)- R_{\beta^{\star}}(\beta^{\star})
&\geq R_{\beta^{\star}}(\beta')- R_{\beta^{\star}}(\beta^{\star})\notag\\
&\stackrel{\text{Taylor}}{=}\frac12 (\beta'-\beta^{\star})^{T}\bbE_{\substack{x\sim\bbP_T(X)\\y|x\sim f(y|x;\beta^{\star})}}[\nabla^2\ell(x,y,\tilde\beta)](\beta'-\beta^{\star})\notag\\
&\stackrel{\text{by \eqref{ineq:fisher_target}}}{\geq}\frac14 (\beta'-\beta^{\star})^{T}\cI_T(\beta_0)(\beta'-\beta^{\star})\notag\\
&\geq \frac14 \lambda_{\min}(\cI_T(\beta_0))\|\beta'-\beta^{\star}\|^2_2\notag\\
&\stackrel{\text{by \eqref{ineq:pf:lem_fisher4}}}{\geq}\frac{R^2_0}{16}\lambda_{\min}(\cI_T(\beta_0)).
\end{align}
Note that
\begin{align}\label{ineq:pf:lem_fisher6}
R_{\beta^{\star}}(\beta_0)- R_{\beta^{\star}}(\beta^{\star})
&\stackrel{\text{Taylor}}{=}\frac12 (\beta_0-\beta^{\star})^{T}\bbE_{\substack{x\sim\bbP_T(X)\\y|x\sim f(y|x;\beta^{\star})}}[\nabla^2\ell(x,y,\tilde\beta)](\beta_0-\beta^{\star})\notag\\
&\stackrel{\text{by \eqref{ineq:fisher_target}}}{\leq}(\beta_0-\beta^{\star})^{T}\cI_T(\beta_0)(\beta_0-\beta^{\star})\notag\\
&\leq \lambda_{\max}(\cI_T(\beta_0))\|\beta_0-\beta^{\star}\|^2_2\notag\\
&\leq R^2_1\lambda_{\max}(\cI_T(\beta_0))\notag\\
&=\frac{R^2_0}{16}\lambda_{\min}(\cI_T(\beta_0)),
\end{align}
where the last equation follows from the choice of $R_1$. By \eqref{ineq:pf:lem_fisher5} and \eqref{ineq:pf:lem_fisher6}, we obtain $R_{\beta^{\star}}(\beta)\geq R_{\beta^{\star}}(\beta_0)$. Thus, we finish the proof of Lemma \ref{lem:fisher_inform}.
\end{proof}

\paragraph{Proofs for Lemma \ref{lem:van_trees}}
\begin{proof}[Proof of Lemma \ref{lem:van_trees}]
Let $\beta_0=[\beta_{0,1},\ldots,\beta_{0,d}]^{T}$, $\beta=[\beta_{1},\ldots,\beta_{d}]^{T}$ and
\begin{align*}
    f_i(x) := \frac{\pi}{4B}\cos\left(\frac{\pi}{2B}(x-\beta_{0,i})\right),\, i=1,\ldots,d.
\end{align*}
We define the prior density as
\begin{align*}
\lambda(\beta):=
\begin{cases} 
      \Pi^d_{i=1}f_i(\beta_i) & \beta\in C_{\beta_0}(B)\\
      0 & \beta\notin C_{\beta_0}(B)
\end{cases},
\end{align*}
which is supported on $C_{\beta_0}(B)$. In the sequel, we will show this prior density satisfies the condition outlined in Lemma \ref{lem:van_trees}.

For notation simplicity, we denote 
\begin{align*}
A = (A_{ij}):=\cI^{-1}_T(\beta_0), \,\, C = (C_{ij}):=\cI_T(\beta_0)\cI^{-1}_S(\beta_0).
\end{align*}
By multivariate van Trees inequality (Theorem 1 in \cite{gill1995applications}), for any estimator $\hat\beta$, we have
\begin{align}\label{ineq:pf:lem_van1}
&\bbE_{\beta^{\star}\sim\lambda(\beta)}\bbE_{\substack{x_i\sim\bbP_S(X)\\y_i|x_i\sim f(y|x;\beta^{\star})}}\left[(\hat\beta-\beta^{\star})^{T}\cI_T(\beta_0)(\hat\beta-\beta^{\star})\right]\notag\\
&\quad\geq \frac{\Tr\left(\cI_T(\beta_0)\cI^{-1}_S(\beta_0)\right)^2}{n\bbE_{\beta^{\star}\sim\lambda(\beta)}\left[\Tr\left(\cI^{-1}_S(\beta_0)\cI_S(\beta^{\star})\cI^{-1}_S(\beta_0)\cI_T(\beta_0)\right)\right]+\tilde{\cI}(\lambda)},
\end{align}
where 
\begin{align*}
\tilde{\cI}(\lambda) = \int_{C_{\beta_0}(B)}\left(\sum_{i,j,k,\ell}A_{ij}C_{ik}C_{j\ell}\frac{\partial}{\partial\beta_k}\lambda(\beta)\frac{\partial}{\partial\beta_{\ell}}\lambda(\beta)\right)  \frac{1}{\lambda(\beta)} d\beta.
\end{align*}
By the choice of $\lambda(\beta)$, we have
\begin{align*}
&\int_{C_{\beta_0}(B)}\left(\sum_{\substack{i,j,k,\ell\\k\neq\ell}}A_{ij}C_{ik}C_{j\ell}\frac{\partial}{\partial\beta_k}\lambda(\beta)\frac{\partial}{\partial\beta_{\ell}}\lambda(\beta)\right)  \frac{1}{\lambda(\beta)} d\beta\\
&=\int_{C_{\beta_0}(B)}\sum_{\substack{i,j,k,\ell\\k\neq\ell}}A_{ij}C_{ik}C_{j\ell}f'_k(\beta_k)f'_{\ell}(\beta_{\ell})\Pi_{i\neq k,\ell}f_i(\beta_i)d\beta\\
&=\sum_{\substack{i,j,k,\ell\\k\neq\ell}}A_{ij}C_{ik}C_{j\ell}\int_{C_{\beta_0}(B)}f'_k(\beta_k)f'_{\ell}(\beta_{\ell})\Pi_{i\neq k,\ell}f_i(\beta_i)d\beta\\
&=0.
\end{align*}
Here the last equation follows from the fact
\begin{align*}
 \int^{\beta_{0,k}+B}_{\beta_{0,k}-B}f'_k(\beta_k)d\beta_k=\int^{\beta_{0,\ell}+B}_{\beta_{0,\ell}-B}f'_{\ell}(\beta_{\ell})d\beta_{\ell}=0.
\end{align*}

Note that
\begin{align*}
&\int_{C_{\beta_0}(B)}\left(\sum_{\substack{i,j,k,\ell\\k=\ell}}A_{ij}C_{ik}C_{j\ell}\frac{\partial}{\partial\beta_k}\lambda(\beta)\frac{\partial}{\partial\beta_{\ell}}\lambda(\beta)\right)  \frac{1}{\lambda(\beta)} d\beta\\    
&=\sum_{\substack{i,j,k}}A_{ij}C_{ik}C_{jk}\int_{C_{\beta_0}(B)}\frac{(f'_k(\beta_k))^2}{f_k(\beta_k)}\Pi_{i\neq k}f_i(\beta_i)d\beta\\
&=\sum_{\substack{i,j,k}}A_{ij}C_{ik}C_{jk}\int^{\beta_{0,k}+B}_{\beta_{0,k}-B}\frac{(f'_k(\beta_k))^2}{f_k(\beta_k)}d\beta_k\\
&=\frac{\pi^2}{B^2}\sum_{\substack{i,j,k}}A_{ij}C_{ik}C_{jk}\\
&=\frac{\pi^2}{B^2}\Tr(ACC^{T}).
\end{align*}
Thus, we have
\begin{align}\label{ineq:pf:lem_van2}
 \tilde{\cI}(\lambda) &=   \int_{C_{\beta_0}(B)}\left(\sum_{\substack{i,j,k,\ell\\k\neq\ell}}A_{ij}C_{ik}C_{j\ell}\frac{\partial}{\partial\beta_k}\lambda(\beta)\frac{\partial}{\partial\beta_{\ell}}\lambda(\beta)\right)  \frac{1}{\lambda(\beta)} d\beta\notag\\
 &\quad+ \int_{C_{\beta_0}(B)}\left(\sum_{\substack{i,j,k,\ell\\k=\ell}}A_{ij}C_{ik}C_{j\ell}\frac{\partial}{\partial\beta_k}\lambda(\beta)\frac{\partial}{\partial\beta_{\ell}}\lambda(\beta)\right)  \frac{1}{\lambda(\beta)} d\beta\notag\\
 &=\frac{\pi^2}{B^2}\Tr(ACC^{T})\notag\\
 &=\frac{\pi^2}{B^2}\Tr\left(\cI_T(\beta_0)\cI^{-2}_S(\beta_0)\right).
\end{align}
Combine \eqref{ineq:pf:lem_van1} and \eqref{ineq:pf:lem_van2}, we prove Lemma \ref{lem:van_trees}.
    
\end{proof}

%% file: pf_application.tex
\section{Proofs for Section \ref{application}}

\subsection{Proofs for Proposition \ref{prop:linear} and Theorem \ref{thm:linear}}\label{pf:thm_linear}
\begin{proof}
For our linear regression model, 
\begin{align*}
\ell(x,y,\beta) = \frac{1}{2}\log (2 \pi) + \frac{1}{2} (y-x^{T}\beta)^{2}.    
\end{align*}
The convexity of $\ell$ in $\beta$ immediately implies Assumption \ref{assm5}. 
We then have
\begin{align*}
&\nabla\ell(x,y,\beta)
=-x(y-x^{T}\beta),\\
&\nabla^2\ell(x,y,\beta)
=xx^T,\\
&\nabla^3\ell(x,y,\beta)
=0, \\
& \cI_S=\bbE_{x\sim\bbP_S(X)}[xx^T]=I_d , \\
& \cI_T=\bbE_{x\sim\bbP_T(X)}[xx^T] =\alpha\alpha^{T}+\sigma^2 I_d.
\end{align*}
Therefore Assumption \ref{assm4} is satisfied with $L_S=L_T=0$ and Assumption \ref{assm7} trivially holds.
Note that $\nabla\ell(x_i,y_i,\beta^{\star})=-x_i\varepsilon_i$. Since $\|x_i\|_2$ is $\sqrt{d}$-subgaussian and $|\varepsilon_i|$ is $1$-subgaussian, by Lemma 2.7.7 in \cite{vershynin2018high}, it holds that $\|x_i\|_2|\varepsilon_i|$ is $\sqrt{d}$-subexponential random variable. Thus $\|A\nabla\ell(x_i,y_i,\beta^{\star})\|_2$ is $\|A\|_2\sqrt{d}$-subexponential random variable. 

Then, by Lemma \ref{lem:concentration_vec} with $u_i=A(\nabla\ell(x_i,y_i,\beta^{\star})-\bbE[\nabla\ell(x_i,y_i,\beta^{\star})])=A\nabla\ell(x_i,y_i,\beta^{\star})$, $V=\bbE[\|u_i\|^2_2]=n \cdot \mathbb{E} \|A(\nabla\ell_n(\beta^{\star})-\bbE[\nabla\ell_n(\beta^{\star})])\|_2^2$, $\alpha=1$ and $B_u^{(\alpha)}=c\sqrt{d}\|A\|_2$, we have for any matrix $A\in\bbR^{d\times d}$, and any $\delta\in (0,1)$, with probability at least $1-\delta$:
\begin{align*}
     \left\|A\left(\nabla\ell_n(\beta^{\star})-\bbE[\nabla\ell_n(\beta^{\star})]\right)\right\|_{2} \leq c\left(\sqrt{\frac{V\log \frac{d}{\delta}}{n}}+\sqrt{d}\|A\|_2\log (\frac{\sqrt{d}\|A\|_2}{\sqrt{V}})\frac{\log \frac{d}{\delta}}{n}\right),
\end{align*}
which satisfies the gradient concentration in Assumption \ref{assm1} with $B_1=c\sqrt{d}$ and $\gamma=1$.

Note that $x_i\sim\cN(0,I_d)$. Thus, by Theorem 13.3 in \cite{Rinaldo2018}, for any $\delta\in (0,1)$, with probability at least $1-\delta$, we have
\begin{align*}
    \|\nabla^2\ell_n(\beta^{\star})-\bbE[\nabla^2\ell_n(\beta^{\star}]\|_2
    &=\left\|\frac1n\sum^n_{i=1}x_ix_i^{T}-I_d\right\|_2\notag\\
    &\leq c\left(\sqrt{\frac{d\log (1/\delta)}{n}}+\frac{d\log (1/\delta)}{n}\right)\notag\\
    &\leq 2c\sqrt{\frac{d\log (1/\delta)}{n}},
\end{align*}
where the last inequality holds if $n\geq \cO(d\log \frac{1}{\delta})$. Hence linear regression model satisfies the matrix concentration in Assumption \ref{assm1} with $B_2=c\sqrt{d}$, $N(\delta)= d \log \frac{1}{\delta}$.
Since $\nabla^3\ell\equiv 0$, we know Assumption \ref{assm2} holds with $B_3=0$.

Note that 
\begin{align*}
\nabla^2\ell_n(\beta)=\frac1n\sum^n_{i=1}x_ix_i^{T}=\frac1n X^{T}X,
\end{align*}
where $X:=[x_1,\ldots,x_n]^T$. Given that $\{x_i\}^n_{i=1}$ are i.i.d $\cN(0,I_d)$, it follows that $X$ is almost surely full rank when $n\geq d$. Hence, when $n\geq d$, we have
\begin{align*}
\nabla^2\ell_n(\beta)=\frac1n\sum^n_{i=1}x_ix_i^{T}=\frac1n X^{T}X\succ 0.
\end{align*}
Consequently, $\ell_n(\cdot)$ is strictly convex and thus satisfies Assumption \ref{assm3}. Finally, Theorem \ref{thm:linear} follows directly from Theorem \ref{thm:well_upper} with $\gamma=1$, $B_1=c\sqrt{d}, B_2=c\sqrt{d}, B_3=0$, $N(\delta)= d \log \frac{1}{\delta}$, $\cI_S=I_d$ and $\cI_T=\alpha\alpha^{T}+\sigma^2 I_d$.

\end{proof}

\subsection{Proofs for Proposition \ref{prop:logistic} and Theorem \ref{thm:logistic}}

\begin{proof}
In the following, we will show the logistic regression model satisfies Assumptions \ref{assm:well_upper} and \ref{assm:well_lower}. For logistic regression, the loss function is defined as 
\begin{align*}
\ell(x,y,\beta)=\log(1+e^{x^T\beta})-y(x^T\beta).
\end{align*}
We then have
\begin{align*}
&\nabla\ell(x,y,\beta)
=\frac{x}{1+e^{-x^T\beta}}-xy,\\
&\nabla^2\ell(x,y,\beta)
=\frac{xx^T}{2+e^{-x^T\beta}+e^{x^T\beta}},\\
&\nabla^3\ell(x,y,\beta)
=\frac{e^{-x^T\beta}-e^{x^T\beta}}{(2+e^{-x^T\beta}+e^{x^T\beta})^2}\cdot x\otimes x \otimes x.
\end{align*}
Here $\otimes$ represents the tensor product and $x\otimes x \otimes x\in\bbR^{d\times d\times d}$ with $(x\otimes x \otimes x)_{ijk}=x_i x_j x_k$. The convexity of $\ell$ in $\beta$ immediately implies Assumption \ref{assm5}; Assumption \ref{assm7} trivially holds. 
Note that on source domain $\|x\|_2 = \sqrt{d} $ and $|y|\leq 1$. Hence we have for any $(x,y)$ on source domain:
\begin{align*}
&\|\nabla\ell(x,y,\beta^{\star})\|_2
=\left\|\frac{x}{1+e^{-x^T\beta^{\star}}}-xy\right\|_2
\leq \left\|\frac{x}{1+e^{-x^T\beta^{\star}}}\right\|_2+\|xy\|_2
\leq \|x\|_2+\|x\|_2=2\sqrt{d},\\
&\|\nabla^2\ell(x,y,\beta^{\star})\|_2
=\left\|\frac{xx^T}{2+e^{-x^T\beta^{\star}}+e^{x^T\beta^{\star}}}\right\|_2
\leq \|xx^T\|_2\leq \|x\|^2_2\leq d.
\end{align*}
By Lemma \ref{lem:concentration_vec} with $u_i=A(\nabla\ell(x_i,y_i,\beta^{\star})-\bbE[\nabla\ell(x_i,y_i,\beta^{\star})])=A\nabla\ell(x_i,y_i,\beta^{\star})$, $V=\bbE[\|u_i\|^2_2]$, $\alpha=+\infty$, $B_{u}^{(\alpha)}=2\sqrt{d}\|A\|_2$,
we have for any matrix $A\in\bbR^{d\times d}$, and any $\delta\in (0,1)$, with probability at least $1-\delta$:
\begin{align*}
     \left\|A\left(\nabla\ell_n(\beta^{\star})-\bbE[\nabla\ell_n(\beta^{\star})]\right)\right\|_{2} \leq c\left(\sqrt{\frac{V\log \frac{d}{\delta}}{n}}+\frac{\sqrt{d}\|A\|_2\log \frac{d}{\delta}}{n}\right),
\end{align*}
which satisfies the gradient concentration in Assumption \ref{assm1} with $B_1=c\sqrt{d}$ and $\gamma=0$.
By matrix Hoeffding inequality, logistic regression model satisfies the matrix concentration in Assumption \ref{assm1} with $B_2=cd$. We conclude that logistic regression model satisfies Assumption \ref{assm1} with $N(\delta)=0$, $B_1=c\sqrt{d}$, $\gamma=0$, $B_2=cd$.

Note that for $x$ on source domain, we have $\|x\|_2\leq \sqrt{d}$; for $x$ on target domain, we have $\|x\|_2\leq \sqrt{d}+r$. Thus, it holds that
\begin{align*}
\|\nabla^3\ell(x,y,\beta)\|_2
=\left\|\frac{e^{-x^T\beta}-e^{x^T\beta}}{(2+e^{-x^T\beta}+e^{x^T\beta})^2}\cdot x\otimes x \otimes x\right\|_2
\leq_{(i)}\|x\otimes x \otimes x\|_2
\leq \|x\|^3_2\leq (\sqrt{d}+r)^3.
\end{align*}
Here $(i)$ uses the fact that
\begin{align*}
\left|\frac{e^{-x^T\beta}-e^{x^T\beta}}{(2+e^{-x^T\beta}+e^{x^T\beta})^2}\right|
\leq \frac{e^{-x^T\beta}+e^{x^T\beta}}{(2+e^{-x^T\beta}+e^{x^T\beta})^2}
\leq \frac{1}{2+e^{-x^T\beta}+e^{x^T\beta}}
\leq 1.
\end{align*}
Hence logistic regression satisfies Assumptions \ref{assm2} with $B_3=(\sqrt{d}+r)^3$. Notice that this also implies Assumption \ref{assm4}:
By definition,
\begin{align*}
   \cI_S(\beta) := \bbE_{x\sim\bbP_S(X)}[\nabla^{2} \ell (x,y,\beta)],    
\end{align*}
therefore 
\begin{align*}
\|\cI_S(\beta_1)-\cI_S(\beta_2)\|&=\|   \bbE_{x\sim\bbP_S(X)}[\nabla^{2} \ell (x,y,\beta_1)-\nabla^{2} \ell (x,y,\beta_2)] \| \\
& \leq \bbE_{x\sim\bbP_S(X)}[\|\nabla^{2} \ell (x,y,\beta_1)-\nabla^{2} \ell (x,y,\beta_2)\|] \\
& \leq (\sqrt{d})^3 \|\beta_1-\beta_2\|.
\end{align*}
Similarly
\begin{align*}
\|\cI_T(\beta_1)-\cI_T(\beta_2)\|\leq (\sqrt{d}+r)^3 \|\beta_1-\beta_2\|.
\end{align*}
These inequlities shows that logistic regression model satisfies Assumption \ref{assm4} with $L_S=d^{1.5}$ and $L_T=(\sqrt{d}+r)^3$.
Note that 
\begin{align*}
\nabla^2\ell_n(\beta)
=\frac{1}{n}\sum^n_{i=1}\nabla^2\ell(x_i,y_i,\beta) 
=\frac{1}{n}\sum^n_{i=1}\frac{x_ix^{T}_i}{2+e^{-x^T_i\beta}+e^{x^T_i\beta}}
=\frac1n X^{T}AX,
\end{align*}
where $X:=[x_1,\ldots,x_n]^{T}\in\bbR^{n\times d}$ and $A:=\diag(1/(2+e^{-x^T_i\beta}+e^{x^T_i\beta}))\succ 0$.
When $n \geq  d$, $X$ is full rank (i.e., $\rank(X)=d$) almost surely, consequently, $\ell_n(\cdot)$ is strictly convex and thus satisfies Assumption \ref{assm3}. 

By Theorem \ref{thm:well_upper}, we have when $n\geq \cO(N^{\star}\log \frac{d}{\delta})$,
\begin{align*}
R_{\beta^{\star}}(\beta_{\MLE}) 
\lesssim \frac{\Tr\left(\cI_T\cI^{-1}_S\right)\log \frac{d}{\delta}}{n}.
\end{align*}
Here 
\begin{align*}
N^{\star}:= (1 + \tilde{\kappa}/\kappa)^2 \cdot \max\left\{\tilde{\kappa}^{-1}\alpha_1^2\log^{2\gamma} \left((1+\tilde{\kappa}/\kappa)\tilde{\kappa}^{-1}\alpha^2_1\right),\  \alpha_2^2, \  
\tilde{\kappa}(1 + \|\cI_T^{\frac12} \cI_S^{-1}\cI_T^{\frac12}\|_2^{-2})\alpha_3^2 \right\},
\end{align*}
where $\alpha_1 := B_1 \|\cI_S^{-1}\|_2^{0.5}$, $\alpha_2 := B_2 \|\cI_S^{-1}\|_2$, $\alpha_3 := B_3 \|\cI_S^{-1}\|_2^{1.5}$,
\begin{align*}
    \kappa:=\frac{\Tr(\cI_{T} \cI_{S}^{-1}) }{\|\cI_T^{\frac12} \cI_S^{-1}\cI_T^{\frac12}\|_2},\, \, \tilde{\kappa} := \frac{\Tr(\cI_S^{-1})}{\|\cI_S^{-1}\|_2}.
\end{align*}

Now it remains to calculate the quantities $N^{\star}$ and $\Tr\left(\cI_T\cI^{-1}_S\right)$ for this instance, where the crucial part is to identify what are $\cI_S$ and $\cI_T$. The following two lemmas give the characterization of $\cI_S$ and $\cI_T$.

\begin{lemma} \label{logistic_lemma1}
Under the conditions of Theorem \ref{thm:logistic}, we have $\cI_S=U\diag(\lambda_1,\lambda_2,\ldots,\lambda_2)U^{T}$ and $\cI_T=U\diag(\lambda_1,\lambda_2+r^2\lambda_3,\lambda_2,\ldots,\lambda_2)U^{T}$ for an orthonormal matrix $U$. Where 
\begin{align*}
    &\lambda_{1}:=\bbE_{x\sim\uni (\cS^{d-1}(\sqrt{d}))}[\frac{(\beta^{\star T}x)^{2}}{2+\exp (\beta^{\star T}x) +\exp (-\beta^{\star T}x)}],\notag\\
    &\lambda_{2}:=\bbE_{x\sim\uni (\cS^{d-1}(\sqrt{d}))}[\frac{(\beta_{\perp}^{\star T}x)^{2}}{2+\exp (\beta^{\star T}x) +\exp (-\beta^{\star T}x)}],\notag\\
    &\lambda_{3}:=\bbE_{x\sim\uni (\cS^{d-1}(\sqrt{d}))}[\frac{1}{2+\exp (\beta^{\star T}x) +\exp (-\beta^{\star T}x)}].
\end{align*}
\end{lemma}

\begin{lemma} \label{logistic_lemma2}
Under the conditions of Theorem \ref{thm:logistic}, there exist absolute constants $c,C, c'>0$ such that $c<\lambda_1, \lambda_2, \lambda_3 <C$, for $d \geq c'$.
\end{lemma}
The proofs for these two lemmas are in the next section. With Lemma \ref{logistic_lemma1}, we have $\cI_T \cI_S^{-1} = U\diag(1,1+ r^{2}\frac{\lambda_3}{\lambda_2},\ldots,1)U^{T}$, $\cI_S^{-1} = U\diag(\frac{1}{\lambda_1},\frac{1}{\lambda_2},\ldots,\frac{1}{\lambda_2})U^{T}$. By Lemma \ref{logistic_lemma2}, since $\lambda_1, \lambda_2, \lambda_3 = O(1)$, we have $\Tr (\cI_T \cI_S^{-1}) = d + r^{2} \frac{\lambda_3}{\lambda_2} \asymp d + r^{2}$, $\|\cI_T \cI_S^{-1}\|_2 = 1 + r^{2} \frac{\lambda_3}{\lambda_2} \asymp 1 + r^{2}$. Similarly  $\Tr (\cI_S^{-1}) = \lambda_1^{-1} + (d-1) \lambda_2^{-1} \asymp d $, $\| \cI_S^{-1}\|_2 = \max\{\lambda_1^{-1},\lambda_2^{-1}\} \asymp 1$. Also recall that $B_1=\sqrt{d}, B_2=d, B_3=(\sqrt{d}+r)^{3}$, plug in all those quantities we have $\kappa = \frac{\Tr (\cI_T \cI_S^{-1})}{\|\cI_T \cI_S^{-1}\|_2} \asymp \frac{d+r^{2}}{1+r^{2}}$, $\tilde{\kappa} = \frac{\Tr (\cI_S^{-1})}{\| \cI_S^{-1}\|_2} \asymp d$, $\alpha_1 = B_1 \| \cI_S^{-1}\|_2^{0.5} \asymp  \sqrt{d}$, $\alpha_2 = B_2 \| \cI_S^{-1}\|_2\asymp d$, $\alpha_3 = B_3 \| \cI_S^{-1}\|_2^{1.5} \asymp  (\sqrt{d} + r)^{3}$. Therefore we have when $n\geq \cO(N^{\star}\log \frac{d}{\delta})$,
\begin{align*}
R_{\beta^{\star}}(\beta_{\MLE}) 
\lesssim \frac{\Tr\left(\cI_T\cI^{-1}_S\right)\log \frac{d}{\delta}}{n} \asymp \frac{(d+r^{2})\log \frac{d}{\delta}}{n},
\end{align*}
where
\begin{align*}
N^{\star}&= (1 + \tilde{\kappa}/\kappa)^2 \cdot \max\left\{\tilde{\kappa}^{-1}\alpha_1^2\log^{2\gamma} \left((1+\tilde{\kappa}/\kappa)\tilde{\kappa}^{-1}\alpha^2_1\right),\  \alpha_2^2, \  
\tilde{\kappa}(1 + \|\cI_T^{\frac12} \cI_S^{-1}\cI_T^{\frac12}\|_2^{-2})\alpha_3^2 \right\} \\
& \asymp \left(1+\frac{d+r^2d}{d+r^2}\right)^2\cdot\max\left\{1, d^2, d (1 + (1+r^{2})^{-2}) (\sqrt{d}+r)^{6}\right\}\\
& = \left(1+\frac{d+r^2d}{d+r^2}\right)^2\cdot d (\sqrt{d}+r)^{6} .
\end{align*}

When $r \lesssim 1$, $N^{\star} \asymp d^{4}$. When $1 \lesssim r \lesssim \sqrt{d}$, $N^{\star} \asymp r^{4}d^{4}$. When $\sqrt{d} \lesssim r$, $N^{\star} \asymp r^{6}d^{3}$. 
\end{proof}

\subsubsection{Proofs for Lemma \ref{logistic_lemma1} and \ref{logistic_lemma2}}
The intuition of proving Lemma \ref{logistic_lemma1} and \ref{logistic_lemma2} is that,  when $d$ is large, distribution $\uni (\cS^{d-1}(\sqrt{d}))$ behaves similar to distribution $\cN (0, I_d)$ which has good properties (isotropic, independence of each entry, etc.)
\begin{proof}[Proof of Lemma \ref{logistic_lemma1}]
By definition,
\begin{align*}
   \cI_S := \bbE_{x\sim\uni (\cS^{d-1}(\sqrt{d}))}[\frac{xx^{T}}{2+\exp (\beta^{\star T}x) +\exp (-\beta^{\star T}x)}]    
\end{align*}
Let $z \sim \cN(0,I_d)$, then $x$ and $z\frac{\sqrt{d}}{\|z\|_{2}}$ have the same distribution. Therefore 
\begin{align*}
   \cI_S &= \bbE_{x\sim\uni (\cS^{d-1}(\sqrt{d}))}[\frac{xx^{T}}{2+\exp (\beta^{\star T}x) +\exp (-\beta^{\star T}x)}]   \\
   &= \bbE_{z\sim\cN (0,I_d)}[\frac{zz^{T}\frac{d}{\|z\|^{2}_{2}}}{2+\exp (\beta^{\star T}z \cdot \frac{\sqrt{d}}{\|z\|_{2}}) +\exp (-\beta^{\star T}z\cdot \frac{\sqrt{d}}{\|z\|_{2}})}]   \\
   &= \bbE_{z\sim\cN (0,I_d)}[\frac{(\beta^{\star}\beta^{\star T} + U_{\perp}U_{\perp}^{T})zz^{T}\frac{d}{\|z\|^{2}_{2}}}{2+\exp (\beta^{\star T}z \cdot \frac{\sqrt{d}}{\|z\|_{2}}) +\exp (-\beta^{\star T}z\cdot \frac{\sqrt{d}}{\|z\|_{2}})}]   
\end{align*}
where $[\beta^{\star},U_{\perp}] \in \bbR^{d \times d}$ is a orthogonal basis. 

With this expression, we first prove $\beta^{\star}$ is an eigenvector of $\cI_S$ with corresponding eigenvalue $\lambda_1$.
\begin{align*}
   \cI_S \beta^{\star} &= \bbE_{z\sim\cN (0,I_d)}[\frac{(\beta^{\star}\beta^{\star T} + U_{\perp}U_{\perp}^{T})zz^{T}\frac{d}{\|z\|^{2}_{2}}}{2+\exp (\beta^{\star T}z \cdot \frac{\sqrt{d}}{\|z\|_{2}}) +\exp (-\beta^{\star T}z\cdot \frac{\sqrt{d}}{\|z\|_{2}})}] \beta^{\star}  \\
   &= \bbE_{z\sim\cN (0,I_d)}[\frac{\beta^{\star}\beta^{\star T}zz^{T}\frac{d}{\|z\|^{2}_{2}}\beta^{\star}}{2+\exp (\beta^{\star T}z \cdot \frac{\sqrt{d}}{\|z\|_{2}}) +\exp (-\beta^{\star T}z\cdot \frac{\sqrt{d}}{\|z\|_{2}})}] \\
   &+ \bbE_{z\sim\cN (0,I_d)}[\frac{U_{\perp}U_{\perp}^{T}zz^{T}\frac{d}{\|z\|^{2}_{2}}\beta^{\star}}{2+\exp (\beta^{\star T}z \cdot \frac{\sqrt{d}}{\|z\|_{2}}) +\exp (-\beta^{\star T}z\cdot \frac{\sqrt{d}}{\|z\|_{2}})}] \\
   &= \bbE_{z\sim\cN (0,I_d)}[\frac{(\beta^{\star T}z)^{2}\frac{d}{\|z\|^{2}_{2}}}{2+\exp (\beta^{\star T}z \cdot \frac{\sqrt{d}}{\|z\|_{2}}) +\exp (-\beta^{\star T}z\cdot \frac{\sqrt{d}}{\|z\|_{2}})}] \beta^{\star} \\
   &+ \bbE_{z\sim\cN (0,I_d)}[\frac{U_{\perp}U_{\perp}^{T}zz^{T}\frac{d}{\|z\|^{2}_{2}}\beta^{\star}}{2+\exp (\beta^{\star T}z \cdot \frac{\sqrt{d}}{\|z\|_{2}}) +\exp (-\beta^{\star T}z\cdot \frac{\sqrt{d}}{\|z\|_{2}})}] \\
   &= \lambda_1 \beta^{\star} + \bbE_{z\sim\cN (0,I_d)}[\frac{U_{\perp}U_{\perp}^{T}zz^{T}\frac{d}{\|z\|^{2}_{2}}\beta^{\star}}{2+\exp (\beta^{\star T}z \cdot \frac{\sqrt{d}}{\|z\|_{2}}) +\exp (-\beta^{\star T}z\cdot \frac{\sqrt{d}}{\|z\|_{2}})}]. 
\end{align*}
Therefore we only need to prove 
\begin{align*}
    \bbE_{z\sim\cN (0,I_d)}[\frac{U_{\perp}U_{\perp}^{T}zz^{T}\frac{d}{\|z\|^{2}_{2}}\beta^{\star}}{2+\exp (\beta^{\star T}z \cdot \frac{\sqrt{d}}{\|z\|_{2}}) +\exp (-\beta^{\star T}z\cdot \frac{\sqrt{d}}{\|z\|_{2}})}]=0.
\end{align*}
In fact,
\begin{align*}
    & \qquad \bbE_{z\sim\cN (0,I_d)}[\frac{U_{\perp}^{T}zz^{T}\frac{d}{\|z\|^{2}_{2}}\beta^{\star}}{2+\exp (\beta^{\star T}z \cdot \frac{\sqrt{d}}{\|z\|_{2}}) +\exp (-\beta^{\star T}z\cdot \frac{\sqrt{d}}{\|z\|_{2}})} ] \\
    & = \bbE_{z\sim\cN (0,I_d)}[\frac{ \frac{d}{\|z\|^{2}_{2}} (U_{\perp}^{T}z)(z^{T}\beta^{\star})}{2+\exp (\beta^{\star T}z \cdot \frac{\sqrt{d}}{\|z\|_{2}}) +\exp (-\beta^{\star T}z\cdot \frac{\sqrt{d}}{\|z\|_{2}})} ] \\
    & = \bbE_{z\sim\cN (0,I_d)}[\frac{ \frac{d}{|A|^{2} + \|B\|^{2}} AB}{2+\exp (A \cdot \frac{\sqrt{d}}{\sqrt{|A|^{2} + \|B\|^{2}}}) +\exp (-A\cdot \frac{\sqrt{d}}{|A|^{2} + \|B\|^{2}})} ] 
\end{align*}
where we let $A:=z^{T}\beta^{\star}$, $B:=U_{\perp}^{T}z$. Notice that by the property of $z\sim \cN(0,I_d)$, $A$ and $B$ are independent. Also, $B$ is symmetric, i.e., $B$ and $-B$ have the same distribution. Therefore 
\begin{align*}
    & \qquad \bbE_{z\sim\cN (0,I_d)}[\frac{ \frac{d}{|A|^{2} + \|B\|^{2}} AB}{2+\exp (A \cdot \frac{\sqrt{d}}{\sqrt{|A|^{2} + \|B\|^{2}}}) +\exp (-A\cdot \frac{\sqrt{d}}{|A|^{2} + \|B\|^{2}})}] \\
    &\stackrel{\text{replace $B$ by $-B$ }}{=} \bbE_{z\sim\cN (0,I_d)}[\frac{ -\frac{d}{|A|^{2} + \|B\|^{2}} AB}{2+\exp (A \cdot \frac{\sqrt{d}}{\sqrt{|A|^{2} + \|B\|^{2}}}) +\exp (-A\cdot \frac{\sqrt{d}}{|A|^{2} + \|B\|^{2}})}] \\
    &= -\bbE_{z\sim\cN (0,I_d)}[\frac{ \frac{d}{|A|^{2} + \|B\|^{2}} AB}{2+\exp (A \cdot \frac{\sqrt{d}}{\sqrt{|A|^{2} + \|B\|^{2}}}) +\exp (-A\cdot \frac{\sqrt{d}}{|A|^{2} + \|B\|^{2}})}],
\end{align*}
which implies 
\begin{align*}
    \bbE_{z\sim\cN (0,I_d)}[\frac{U_{\perp}U_{\perp}^{T}zz^{T}\frac{d}{\|z\|^{2}_{2}}\beta^{\star}}{2+\exp (\beta^{\star T}z \cdot \frac{\sqrt{d}}{\|z\|_{2}}) +\exp (-\beta^{\star T}z\cdot \frac{\sqrt{d}}{\|z\|_{2}})}]=0.
\end{align*}

Next we will prove that for any $\beta_{\perp}$ such that $\|\beta_{\perp}\|_{2}=1$, $\beta^{\star T}\beta_{\perp}=0$, $\beta_{\perp}$ is an eigenvector of $\cI_{S}$ with corresponding eigenvalue $\lambda_2$.  Let $[\beta_{\perp}, U]$ be an orthogonal basis ($\beta^{\star}$ is the first column of $U$).

\begin{align*}
   \cI_S \beta_{\perp} &= \bbE_{z\sim\cN (0,I_d)}[\frac{(\beta_{\perp}\beta_{\perp}^{T} + UU^{T})zz^{T}\frac{d}{\|z\|^{2}_{2}}}{2+\exp (\beta^{\star T}z \cdot \frac{\sqrt{d}}{\|z\|_{2}}) +\exp (-\beta^{\star T}z\cdot \frac{\sqrt{d}}{\|z\|_{2}})}] \beta_{\perp}  \\
   &= \bbE_{z\sim\cN (0,I_d)}[\frac{\beta_{\perp}\beta_{\perp}^{T}zz^{T}\frac{d}{\|z\|^{2}_{2}}\beta_{\perp}}{2+\exp (\beta^{\star T}z \cdot \frac{\sqrt{d}}{\|z\|_{2}}) +\exp (-\beta^{\star T}z\cdot \frac{\sqrt{d}}{\|z\|_{2}})}] \\
   &+ \bbE_{z\sim\cN (0,I_d)}[\frac{UU^{T}zz^{T}\frac{d}{\|z\|^{2}_{2}}\beta_{\perp}}{2+\exp (\beta^{\star T}z \cdot \frac{\sqrt{d}}{\|z\|_{2}}) +\exp (-\beta^{\star T}z\cdot \frac{\sqrt{d}}{\|z\|_{2}})}] \\
   &= \bbE_{z\sim\cN (0,I_d)}[\frac{(\beta_{\perp}^{T}z)^{2}\frac{d}{\|z\|^{2}_{2}}}{2+\exp (\beta^{\star T}z \cdot \frac{\sqrt{d}}{\|z\|_{2}}) +\exp (-\beta^{\star T}z\cdot \frac{\sqrt{d}}{\|z\|_{2}})}] \beta_{\perp} \\
   &+ \bbE_{z\sim\cN (0,I_d)}[\frac{UU^{T}zz^{T}\frac{d}{\|z\|^{2}_{2}}\beta_{\perp}}{2+\exp (\beta^{\star T}z \cdot \frac{\sqrt{d}}{\|z\|_{2}}) +\exp (-\beta^{\star T}z\cdot \frac{\sqrt{d}}{\|z\|_{2}})}] \\
   &= \lambda_2 \beta_{\perp} + 0\\
   &= \lambda_2 \beta_{\perp}
\end{align*}
Here 
\begin{align*}
  \bbE_{z\sim\cN (0,I_d)}[\frac{UU^{T}zz^{T}\frac{d}{\|z\|^{2}_{2}}\beta_{\perp}}{2+\exp (\beta^{\star T}z \cdot \frac{\sqrt{d}}{\|z\|_{2}}) +\exp (-\beta^{\star T}z\cdot \frac{\sqrt{d}}{\|z\|_{2}})}]  =0 
\end{align*}
because of a similar reason as in the previous part.

For $\cI_T$, the proving strategy is similar.  For $x \sim\uni(\cS^{d-1}(\sqrt{d}))+v$ on the target domain, where $v= r\beta_{\perp}^{\star}$, let $w=x-v=x-r\beta_{\perp}^{\star}$, then $w \sim\uni(\cS^{d-1}(\sqrt{d}))$. Let $z \sim \cN(0,I_d)$, then $w$ and $z\frac{\sqrt{d}}{\|z\|_{2}}$ have the same distribution. We have
\begin{align*}
   \cI_T &= \bbE_{x\sim\uni (\cS^{d-1}(\sqrt{d})) +v}[\frac{xx^{T}}{2+\exp (\beta^{\star T}x) +\exp (-\beta^{\star T}x)}]   \\
   &=\bbE_{w\sim\uni (\cS^{d-1}(\sqrt{d}))}[\frac{(w+v)(w+v)^{T}}{2+\exp (\beta^{\star T}(w+v)) +\exp (-\beta^{\star T}(w+v))}]   \\
   &\stackrel{v^{T}\beta^{\star}=0}{=} \bbE_{w\sim\uni (\cS^{d-1}(\sqrt{d}))}[\frac{ww^{T}+wv^{T}+vw^{T}+vv^{T}}{2+\exp (\beta^{\star T}w) +\exp (-\beta^{\star T}w)}]   
\end{align*}
Therefore 
\begin{align*}
   \cI_T \beta^{\star} &=  \bbE_{w\sim\uni (\cS^{d-1}(\sqrt{d}))}[\frac{ww^{T}+wv^{T}+vw^{T}+vv^{T}}{2+\exp (\beta^{\star T}w) +\exp (-\beta^{\star T}w)}] \beta^{\star} \\
   & \stackrel{v^{T}\beta^{\star}=0}{=} \bbE_{w\sim\uni (\cS^{d-1}(\sqrt{d}))}[\frac{ww^{T}}{2+\exp (\beta^{\star T}w) +\exp (-\beta^{\star T}w)}] \beta^{\star} \\
   &= \cI_S \beta^{\star} \\
   &= \lambda_1 \beta^{\star},
\end{align*}
where the last line follows from the previous proofs. Similarly, for any $\tilde{\beta_{\perp}}$ such that $\|\tilde{\beta_{\perp}}\|_{2}=1$, $\beta_{\perp}^{\star T}\tilde{\beta_{\perp}}=0$, 
\begin{align*}
   \cI_T \tilde{\beta_{\perp}} &=  \bbE_{w\sim\uni (\cS^{d-1}(\sqrt{d}))}[\frac{ww^{T}+wv^{T}+vw^{T}+vv^{T}}{2+\exp (\beta^{\star T}w) +\exp (-\beta^{\star T}w)}] \tilde{\beta_{\perp}} \\
   & \stackrel{v^{T}\tilde{\beta_{\perp}}=0}{=} \bbE_{w\sim\uni (\cS^{d-1}(\sqrt{d}))}[\frac{ww^{T}}{2+\exp (\beta^{\star T}w) +\exp (-\beta^{\star T}w)}] \tilde{\beta_{\perp}} \\
   &= \cI_S \tilde{\beta_{\perp}} \\
   &= \lambda_2 \tilde{\beta_{\perp}}.
\end{align*}
For $\beta_{\perp}^{\star}$, 
\begin{align*}
   \cI_T \beta_{\perp}^{\star} &=  \bbE_{w\sim\uni (\cS^{d-1}(\sqrt{d}))}[\frac{ww^{T}+wv^{T}+vw^{T}+vv^{T}}{2+\exp (\beta^{\star T}w) +\exp (-\beta^{\star T}w)}] \beta_{\perp}^{\star} \\
   &=  \bbE_{w\sim\uni (\cS^{d-1}(\sqrt{d}))}[\frac{ww^{T}}{2+\exp (\beta^{\star T}w) +\exp (-\beta^{\star T}w)}] \beta_{\perp}^{\star} \\
   &+  \bbE_{w\sim\uni (\cS^{d-1}(\sqrt{d}))}[\frac{wv^{T}}{2+\exp (\beta^{\star T}w) +\exp (-\beta^{\star T}w)}] \beta_{\perp}^{\star} \\
   &+  \bbE_{w\sim\uni (\cS^{d-1}(\sqrt{d}))}[\frac{vw^{T}}{2+\exp (\beta^{\star T}w) +\exp (-\beta^{\star T}w)}] \beta_{\perp}^{\star} \\
   &+  \bbE_{w\sim\uni (\cS^{d-1}(\sqrt{d}))}[\frac{vv^{T}}{2+\exp (\beta^{\star T}w) +\exp (-\beta^{\star T}w)}] \beta_{\perp}^{\star} \\
   & := I_1+ I_2+ I_3 +I_4.
\end{align*}
As in the previous proofs, 
\begin{align*}
    I_1=\cI_S \beta_{\perp}^{\star} = \lambda_2 \beta_{\perp}^{\star}.
\end{align*}
\begin{align*}
   I_2 &=  \bbE_{w\sim\uni (\cS^{d-1}(\sqrt{d}))}[\frac{wv^{T}}{2+\exp (\beta^{\star T}w) +\exp (-\beta^{\star T}w)}] \beta_{\perp}^{\star} \\
   &\stackrel{v=r\beta_{\perp}^{\star}}{=} r \bbE_{w\sim\uni (\cS^{d-1}(\sqrt{d}))}[\frac{w \beta_{\perp}^{\star T} \beta_{\perp}^{\star}}{2+\exp (\beta^{\star T}w) +\exp (-\beta^{\star T}w)}]  \\
    &\stackrel{\|\beta_{\perp}^{\star}\|=1}{=} r \bbE_{w\sim\uni (\cS^{d-1}(\sqrt{d}))}[\frac{w}{2+\exp (\beta^{\star T}w) +\exp (-\beta^{\star T}w)}]  \\
    &=0.
\end{align*}
where the last lines follows from $w$ is symmetric and $\frac{w}{2+\exp (\beta^{\star T}w) +\exp (-\beta^{\star T}w)}$ is a odd function of $w$.
\begin{align*}
   I_3 &=  \bbE_{w\sim\uni (\cS^{d-1}(\sqrt{d}))}[\frac{vw^{T}}{2+\exp (\beta^{\star T}w) +\exp (-\beta^{\star T}w)}] \beta_{\perp}^{\star} \\
   &\stackrel{v=r\beta_{\perp}^{\star}}{=} r \bbE_{w\sim\uni (\cS^{d-1}(\sqrt{d}))}[\frac{\beta_{\perp}^{\star} w^{T}\beta_{\perp}^{\star}}{2+\exp (\beta^{\star T}w) +\exp (-\beta^{\star T}w)}]  \\
   &= r \bbE_{w\sim\uni (\cS^{d-1}(\sqrt{d}))}[\frac{ w^{T}\beta_{\perp}^{\star}}{2+\exp (\beta^{\star T}w) +\exp (-\beta^{\star T}w)}] \beta_{\perp}^{\star} \\
    &=0.
\end{align*}
where the last lines follows from $w$ is symmetric and $\frac{w^{T}\beta_{\perp}^{\star}}{2+\exp (\beta^{\star T}w) +\exp (-\beta^{\star T}w)}$ is a odd function of $w$.
\begin{align*}
   I_4 &=  \bbE_{w\sim\uni (\cS^{d-1}(\sqrt{d}))}[\frac{vv^{T}}{2+\exp (\beta^{\star T}w) +\exp (-\beta^{\star T}w)}] \beta_{\perp}^{\star} \\
   &\stackrel{v=r\beta_{\perp}^{\star}}{=} r^{2} \bbE_{w\sim\uni (\cS^{d-1}(\sqrt{d}))}[\frac{\beta_{\perp}^{\star}\beta_{\perp}^{\star T} \beta_{\perp}^{\star}}{2+\exp (\beta^{\star T}w) +\exp (-\beta^{\star T}w)}]  \\
    &\stackrel{\|\beta_{\perp}^{\star}\|=1}{=} r^{2} \bbE_{w\sim\uni (\cS^{d-1}(\sqrt{d}))}[\frac{1}{2+\exp (\beta^{\star T}w) +\exp (-\beta^{\star T}w)}] \beta_{\perp}^{\star}  \\
    &=r^{2}\lambda_3 \beta_{\perp}^{\star}.
\end{align*}
Combine the calculations of $I_1,I_2,I_3,I_4$, we have 
\begin{align*}
    \cI_{T}\beta_{\perp}^{\star} &= I_1 + I_2 + I_3 + I_4 \\
    &= \lambda_2 \beta_{\perp}^{\star} + r^{2} \lambda_3 \beta_{\perp}^{\star} \\
    &= (\lambda_2 + r^{2} \lambda_3) \beta_{\perp}^{\star}.
\end{align*}
In conclusion, we have $\cI_S=U\diag(\lambda_1,\lambda_2,\ldots,\lambda_2)U^{T}$ and $\cI_T=U\diag(\lambda_1,\lambda_2+r^2\lambda_3,\lambda_2,\ldots,\lambda_2)U^{T}$ for an orthonormal matrix $U$, where $U= [\beta^{\star}, \beta_{\perp}^{\star}, \cdots]$.
\end{proof}

\begin{proof}[Proof of Lemma \ref{logistic_lemma2}]
Recall the definition of $\lambda_1, \lambda_2, \lambda_3$:
\begin{align*}
    &\lambda_{1}:=\bbE_{x\sim\uni (\cS^{d-1}(\sqrt{d}))}[\frac{(\beta^{\star T}x)^{2}}{2+\exp (\beta^{\star T}x) +\exp (-\beta^{\star T}x)}]=\bbE_{z\sim\cN(0,I_d)}[\frac{\frac{d}{\|z\|_{2}^{2}}(\beta^{\star T}z)^{2}}{2+\exp (\frac{\sqrt{d}}{\|z\|_{2}}\beta^{\star T}z) +\exp (-\frac{\sqrt{d}}{\|z\|_{2}}\beta^{\star T}z)}],\\
    &\lambda_{2}:=\bbE_{x\sim\uni (\cS^{d-1}(\sqrt{d}))}[\frac{(\beta_{\perp}^{\star T}x)^{2}}{2+\exp (\beta^{\star T}x) +\exp (-\beta^{\star T}x)}]=\bbE_{z\sim\cN(0,I_d)}[\frac{\frac{d}{\|z\|_{2}^{2}}(\beta_{\perp}^{\star T}z)^{2}}{2+\exp (\frac{\sqrt{d}}{\|z\|_{2}}\beta^{\star T}z) +\exp (-\frac{\sqrt{d}}{\|z\|_{2}}\beta^{\star T}z)}],\\
    &\lambda_{3}:=\bbE_{x\sim\uni (\cS^{d-1}(\sqrt{d}))}[\frac{1}{2+\exp (\beta^{\star T}x) +\exp (-\beta^{\star T}x)}]=\bbE_{z\sim\cN(0,I_d)}[\frac{1}{2+\exp (\frac{\sqrt{d}}{\|z\|_{2}}\beta^{\star T}z) +\exp (-\frac{\sqrt{d}}{\|z\|_{2}}\beta^{\star T}z)}].
\end{align*}
Next we will show that there exists constants $c,C,c'>0$ such that when $d\geq c'$, we have $c \leq \lambda_1 \leq C$. The proofs for $\lambda_2$ and $\lambda_3$ are similar. Notice that, when $d$ is large, $\frac{d}{\|z\|_{2}^{2}}$ concentrates around $1$. If we replace $\frac{d}{\|z\|_{2}^{2}}$ by $1$ in the above expressions, we have 
\begin{align*}
    &\lambda_{1}\approx \bbE_{z\sim\cN(0,I_d)}[\frac{(\beta^{\star T}z)^{2}}{2+\exp (\beta^{\star T}z) +\exp (-\beta^{\star T}z)}]
\end{align*}
Since $\beta^{\star T}z \sim \cN(0,1)$ when $z \sim \cN(0,I_d)$ and $\|\beta^{\star}\|=1$,  we have 
\begin{align*}
   \bbE_{z\sim\cN(0,I_d)}[\frac{(\beta^{\star T}z)^{2}}{2+\exp (\beta^{\star T}z) +\exp (-\beta^{\star T}z)}] = \bbE_{y\sim\cN(0,1)}[\frac{y^{2}}{2+\exp (y) +\exp (-y)}]
\end{align*}
which is a absolute constant greater than zero and not related to $d$. Following this intuition, we can bound $\lambda_1$ as the following. We first state the concentration of the norm of $\cN (0,I_d)$. By \cite{vershynin2018high} (3.7), 
\begin{align} \label{gaussian_norm_concentration}
    \bbP(|\|z\|-\sqrt{d}| \geq t) \leq 2 e^{-4ct^{2}} 
\end{align}
for some absolute constant $c>0$. Take $t=\frac{\sqrt{d}}{2}$, we have
\begin{align*} 
    \bbP(\frac{\|z\|}{\sqrt{d}} \notin [\frac12, \frac32]  ) \leq 2e^{-cd}.
\end{align*}
With this concentration, we do the following truncation:
\begin{align*}
    \lambda_{1}&=\bbE_{z\sim\cN(0,I_d)}[\frac{\frac{d}{\|z\|_{2}^{2}}(\beta^{\star T}z)^{2}}{2+\exp (\frac{\sqrt{d}}{\|z\|_{2}}\beta^{\star T}z) +\exp (-\frac{\sqrt{d}}{\|z\|_{2}}\beta^{\star T}z)}] \\
    &= \bbE_{z\sim\cN(0,I_d)}[\frac{\frac{d}{\|z\|_{2}^{2}}(\beta^{\star T}z)^{2}}{2+\exp (\frac{\sqrt{d}}{\|z\|_{2}}\beta^{\star T}z) +\exp (-\frac{\sqrt{d}}{\|z\|_{2}}\beta^{\star T}z)}\bbI_{\frac{\|z\|}{\sqrt{d}} \in [\frac12, \frac32]}] \\
    &+ \bbE_{z\sim\cN(0,I_d)}[\frac{\frac{d}{\|z\|_{2}^{2}}(\beta^{\star T}z)^{2}}{2+\exp (\frac{\sqrt{d}}{\|z\|_{2}}\beta^{\star T}z) +\exp (-\frac{\sqrt{d}}{\|z\|_{2}}\beta^{\star T}z)}\bbI_{\frac{\|z\|}{\sqrt{d}} \notin [\frac12, \frac32]}] \\
    &:= J_1 + J_2.
\end{align*}
For $J_2$, it is obvious that
\begin{align}
    0\leq J_2 \leq \frac{d}{4} \bbP(\frac{\|z\|}{\sqrt{d}} \notin [\frac12, \frac32]) \leq \frac{d}{2} e^{-cd}.
\end{align}

For upper bound of $J_1$, 
\begin{align*}
    J_1 &= \bbE_{z\sim\cN(0,I_d)}[\frac{\frac{d}{\|z\|_{2}^{2}}(\beta^{\star T}z)^{2}}{2+\exp (\frac{\sqrt{d}}{\|z\|_{2}}\beta^{\star T}z) +\exp (-\frac{\sqrt{d}}{\|z\|_{2}}\beta^{\star T}z)}\bbI_{\frac{\|z\|}{\sqrt{d}} \in [\frac12, \frac32]}] \\
    & \leq \bbE_{z\sim\cN(0,I_d)} [\frac{4(\beta^{\star T}z)^{2}}{4}] =1.
\end{align*}
Therefore 
\begin{align*}
    \lambda_1 &= J_1 + J_2 \\
    & \leq 1 + \frac{d}{2}e^{-cd}.
\end{align*}
It's obvious that there exists an absolute constant $c'$ such that when $d \geq c'$, $\lambda_1 \leq 2$.

For lower bound of $J_1$, we have 
\begin{align*}
    J_1 &= \bbE_{z\sim\cN(0,I_d)}[\frac{\frac{d}{\|z\|_{2}^{2}}(\beta^{\star T}z)^{2}}{2+\exp (\frac{\sqrt{d}}{\|z\|_{2}}\beta^{\star T}z) +\exp (-\frac{\sqrt{d}}{\|z\|_{2}}\beta^{\star T}z)}\bbI_{\frac{\|z\|}{\sqrt{d}} \in [\frac12, \frac32]}] \\
    & \geq \bbE_{z\sim\cN(0,I_d)}[\frac{\frac49(\beta^{\star T}z)^{2}}{2+\exp (2\beta^{\star T}z) +\exp (-2\beta^{\star T}z)}\bbI_{\frac{\|z\|}{\sqrt{d}} \in [\frac12, \frac32]}] \\
    & = \bbE_{z\sim\cN(0,I_d)}[\frac{\frac49(\beta^{\star T}z)^{2}}{2+\exp (2\beta^{\star T}z) +\exp (-2\beta^{\star T}z)}]
    - \bbE_{z\sim\cN(0,I_d)}[\frac{\frac49(\beta^{\star T}z)^{2}}{2+\exp (2\beta^{\star T}z) +\exp (-2\beta^{\star T}z)}\bbI_{\frac{\|z\|}{\sqrt{d}} \notin [\frac12, \frac32]}] \\
    & \geq \bbE_{z\sim\cN(0,I_d)}[\frac{\frac49(\beta^{\star T}z)^{2}}{2+\exp (2\beta^{\star T}z) +\exp (-2\beta^{\star T}z)}]
    - \bbE_{z\sim\cN(0,I_d)}[\frac{\frac49(\beta^{\star T}z)^{2}}{4}\bbI_{\frac{\|z\|}{\sqrt{d}} \notin [\frac12, \frac32]}] \\
    & \geq \bbE_{z\sim\cN(0,I_d)}[\frac{\frac49(\beta^{\star T}z)^{2}}{2+\exp (2\beta^{\star T}z) +\exp (-2\beta^{\star T}z)}]
    - \bbE_{z\sim\cN(0,I_d)}[\frac{\|z\|_{2}^{2}}{9}\bbI_{\frac{\|z\|}{\sqrt{d}} \notin [\frac12, \frac32]}] \\
    &= \bbE_{y\sim\cN(0,1)}[\frac{\frac49y^{2}}{2+\exp (2y) +\exp (-2y)}]
    - \bbE_{z\sim\cN(0,I_d)}[\frac{\|z\|_{2}^{2}}{9}\bbI_{\frac{\|z\|}{\sqrt{d}} \notin [\frac12, \frac32]}] \\
    &:= c_1 - \bbE_{z\sim\cN(0,I_d)}[\frac{\|z\|_{2}^{2}}{9}\bbI_{\frac{\|z\|}{\sqrt{d}} \notin [\frac12, \frac32]}]
\end{align*}
Notice that here $c_1$ is a positive constant not related to $d$. For the second term,
\begin{align*}
    &\bbE_{z\sim\cN(0,I_d)}[\frac{\|z\|_{2}^{2}}{9}\bbI_{\frac{\|z\|}{\sqrt{d}} \notin [\frac12, \frac32]}] \\
    &=  \bbE_{z\sim\cN(0,I_d)}[\frac{\|z\|_{2}^{2}}{9}\bbI_{\frac{\|z\|}{\sqrt{d}} \leq \frac12}] + \bbE_{z\sim\cN(0,I_d)}[\frac{\|z\|_{2}^{2}}{9}\bbI_{\frac{\|z\|}{\sqrt{d}} \geq \frac32}] \\
    & \leq \frac{d}{36} \bbP(\frac{\|z\|}{\sqrt{d}} \leq \frac12) + \frac{1}{9}\int_{\frac94 d}^{\infty} \bbP(\|z\|_{2}^{2} \geq t) \mathrm{d} t + \frac19 \cdot \frac94 d \bbP(\|z\|_{2}^{2} \geq \frac94 d) \\
    & \stackrel{\text{by (\ref{gaussian_norm_concentration})}}{\leq } \frac{d}{36} 2e^{-cd} + \frac{1}{9}\int_{\frac94 d}^{\infty} \bbP(\|z\|_{2}^{2} \geq t) \mathrm{d} t + \frac{d}{4} 2e^{-cd}\\
    & \stackrel{t = d(y+1)^{2}}{\leq } de^{-cd} + \frac{1}{9}\int_{\frac12}^{\infty}2d(y+1) \bbP(\|z\|_{2}\geq \sqrt{d} +\sqrt{d}y ) \mathrm{d} y \\
    & \stackrel{\text{by (\ref{gaussian_norm_concentration})}}{\leq } de^{-cd} + \frac{1}{9}\int_{\frac12}^{\infty}2d(y+1) 2e^{-4cdy^{2}} \mathrm{d} y \\
    & \leq de^{-cd} + 2d\int_{\frac12}^{\infty}y e^{-4cdy^{2}} \mathrm{d} y \\
    & \leq de^{-cd} + \frac{1}{4c} e^{-cd}
\end{align*}
Combine this inequality and previous inequalities of $J_1$ and $J_2$, we have
\begin{align*}
    \lambda_1 &= J_1 + J_2 \\
    & \geq c_1 - de^{-cd} - \frac{1}{4c} e^{-cd}
\end{align*}
Therefore it's obvious that there exists an absolute constant $c'$ such that when $d \geq c'$, $\lambda_1 \geq \frac{c_1}{2}$. 

The proof for $\lambda_2$ is almost the same, the only difference is that in the numerator, we replace $\beta^{\star T}z$ by $\beta_{\perp}^{\star T} z$. The proof for $\lambda_3$ is even simpler. For upper bound, 
\begin{align*}
    \lambda_{3}&=\bbE_{z\sim\cN(0,I_d)}[\frac{1}{2+\exp (\frac{\sqrt{d}}{\|z\|_{2}}\beta^{\star T}z) +\exp (-\frac{\sqrt{d}}{\|z\|_{2}}\beta^{\star T}z)}] \leq \frac14.
\end{align*}
For lower bound, 
\begin{align*}
    \lambda_{3}&=\bbE_{z\sim\cN(0,I_d)}[\frac{1}{2+\exp (\frac{\sqrt{d}}{\|z\|_{2}}\beta^{\star T}z) +\exp (-\frac{\sqrt{d}}{\|z\|_{2}}\beta^{\star T}z)}]\\
    &\geq \bbE_{z\sim\cN(0,I_d)}[\frac{1}{2+\exp (\frac{\sqrt{d}}{\|z\|_{2}}\beta^{\star T}z) +\exp (-\frac{\sqrt{d}}{\|z\|_{2}}\beta^{\star T}z)}\bbI_{\frac{\|z\|}{\sqrt{d}} \in [\frac12, \frac32]}] \\
    &\geq \bbE_{z\sim\cN(0,I_d)}[\frac{1}{2+\exp (2\beta^{\star T}z) +\exp (-2\beta^{\star T}z)}\bbI_{\frac{\|z\|}{\sqrt{d}} \in [\frac12, \frac32]}] \\
    &= \bbE_{z\sim\cN(0,I_d)}[\frac{1}{2+\exp (2\beta^{\star T}z) +\exp (-2\beta^{\star T}z)}] -\bbE_{z\sim\cN(0,I_d)}[\frac{1}{2+\exp (2\beta^{\star T}z) +\exp (-2\beta^{\star T}z)}\bbI_{\frac{\|z\|}{\sqrt{d}} \notin [\frac12, \frac32]}] \\
    &= c_2 - \frac14 \bbP(\frac{\|z\|}{\sqrt{d}} \notin [\frac12, \frac32]) \\
    & \geq c_2 -\frac12 e^{-cd}.
\end{align*}
Therefore there exists constant $c'$ such that when $d \geq c'$, $\lambda_3 \leq \frac{c_2}{2}$.
\end{proof}

\subsection{Proofs for Theorem \ref{thm:phase}}
In this section, our objective is to establish the upper bound of MLE for the phase retrieval model. 
A direct application of Theorem \ref{thm:well_upper} is impractical, as Assumption \ref{assm3} is not met; notably, both $\beta^{\star}, -\beta^{\star}$ serve as global minimums of population loss. 
To circumvent the issue of non-unique global minimums, we employ a methodology similar to that used in proving Theorem \ref{thm:well_upper}, though with a slightly refined analysis.
\begin{proof}[Proof of Theorem \ref{thm:phase}] In the sequel, we will use the same notations as in the proof of Theorem \ref{thm:well_upper}. Even though the global minimum of population loss for the phase retrieval model isn't unique, meaning it could be either $\beta^{\star}$ or $-\beta^{\star}$, we can still show that the MLE falls into a small ball around either $\beta^{\star}$ or $-\beta^{\star}$.

\begin{lemma}\label{lem:mle_phase} Under the settings of Theorem \ref{thm:phase}, if $n\geq \cO(d^4 \log\frac{d}{\delta})$, then with probability at least $1-\delta$, we have
\begin{align*}
\min\{\|\beta_{\MLE}-\beta^{\star}\|_2,\|\beta_{\MLE}+\beta^{\star}\|_2\}
\lesssim \sqrt{\frac{d^{2}\log \frac{d}{\delta}}{n}}.
\end{align*}
\end{lemma}

Without loss of generality, in the sequel, we consider $n\geq \cO(d^4 \log\frac{d}{\delta})$ and assume 
\begin{align}\label{ineq:mle_ball}
\|\beta_{\MLE}-\beta^{\star}\|_2\lesssim \sqrt{\frac{d^{2}\log \frac{d}{\delta}}{n}},
\end{align}
which implies $\beta_{\MLE}\in\bbB_{\beta^{\star}}(1)$.


Recall that for the phase retrieval model,
\begin{align*}
\ell(x,y,\beta)=\frac12\log (2\pi)+\frac12 \left(y-(x^{T}\beta)^2\right)^2.
\end{align*}
It then holds that
\begin{align*}
&\nabla \ell(x,y,\beta)=2(x^{T}\beta)^3 x-2(x^{T}\beta)yx,\\
& \nabla^2 \ell(x,y,\beta)=6(x^{T}\beta)^2xx^{T}-2yxx^{T},\\
& \nabla^3 \ell(x,y,\beta)=12 (x^{T}\beta) x \otimes x \otimes x.
\end{align*}
Note that for $Y=(X^{T}\beta^{\star})^{2}+\varepsilon$, we have $\nabla\ell(X,Y,\beta^{\star})=-2(X^{T}\beta^{\star})X\varepsilon$. Therefore (recall that $\|\beta^{\star}\|=1$) $\|\nabla\ell(x_i,y_i,\beta^{\star})\|$ is $2d$-subgaussian, by Lemma \ref{lem:concentration_vec}, we have for any $\delta$, with probability at least $1- \delta$,
\begin{align} \label{ineq:pf:concentration1_phase}
    \|\cI_{S}^{-1}g\|_{2} \lesssim\sqrt{\frac{\Tr(\cI_{S}^{-1}) \log \frac{d}{\delta}}{n}}+ d\|\cI_{S}^{-1}\| \sqrt{\log \frac{d^2\|\cI_{S}^{-1}\|^2}{\Tr(\cI_{S}^{-1})}} \frac{ \log \frac{d}{\delta}}{n}.
\end{align}
Which can be viewed as setting $B_1= d$ and $\gamma=\frac{1}{2}$ in Assumption \ref{assm1}. Hence $\beta^{\star}+z=\beta^{\star}-\cI_{S}^{-1}g\in\bbB_{\beta^{\star}}(1)$ when $n\geq\cO(\max\{\Tr(\cI_{S}^{-1})\log \frac{d}{\delta}, d\|\cI_{S}^{-1}\|_2 \sqrt{\log \frac{d^2\|\cI_{S}^{-1}\|^2}{\Tr(\cI_{S}^{-1})}} \log \frac{d}{\delta}\})$.

We then show the concentration inequality for the Hessian matrix. Note that
\begin{align*}
\nabla^{2} \ell_{n}(\beta^{\star})=\frac1n \sum^n_{i=1}\nabla^2\ell(x_i,y_i,\beta^{\star})=\frac{4}{n} \sum^n_{i=1}(x^{T}_i\beta^{\star})^2x_ix^{T}_i-\frac{2}{n} \sum^n_{i=1}\varepsilon_ix_ix^{T}_i.
\end{align*}
Since $\|(x^{T}\beta^{\star})^2xx^{T}\|\leq d^{2}$, by matrix Hoeffding, with probability at least $1-\delta$, we have
\begin{align} \label{ineq:pf:concentration3_phase}
\bbE_{\bbP_S}[(x^{T}\beta^{\star})^2xx^{T}] - d^{2} \sqrt{\frac{8\log \frac{d}{\delta}}{n}} I_{d}   \preceq \frac{1}{n} \sum^n_{i=1}(x^{T}_i\beta^{\star})^2x_ix^{T}_i \preceq \bbE_{\bbP_S}[(x^{T}\beta^{\star})^2xx^{T}]  + d^{2} \sqrt{\frac{8\log \frac{d}{\delta}}{n}} I_{d}
\end{align}
Moreover, by matrix Chernoff bound, with probability at least $1-\delta$, we have
\begin{align} \label{ineq:pf:concentration4_phase}
- d\sqrt{\frac{8\log \frac{d}{\delta}}{n}} I_{d}   \preceq-\frac{1}{n} \sum^n_{i=1}\varepsilon_ix_ix^{T}_i\preceq d\sqrt{\frac{8\log \frac{d}{\delta}}{n}} I_{d}.
\end{align}
Combine \eqref{ineq:pf:concentration3_phase} and \eqref{ineq:pf:concentration4_phase}, we obtain
\begin{align} \label{ineq:pf:concentration5_phase}
\nabla^{2} \ell(\beta^{\star}) - 6d^{2} \sqrt{\frac{8\log \frac{d}{\delta}}{n}} I_{d}   \preceq \nabla^{2} \ell_{n}(\beta^{\star}) \preceq \nabla^{2} \ell(\beta^{\star}) + 6d^{2} \sqrt{\frac{8\log \frac{d}{\delta}}{n}} I_{d},
\end{align}
which can be viewed as setting $B_2=d^{2}$ in \eqref{ineq:pf:concentration3}.

For any $\beta\in\bbB_{\beta^{\star}}(1)$, we have
\begin{align*}
\|\nabla^3\ell(x,y,\beta)\|_2=12\|(x^{T}\beta)x\otimes x\otimes x\|\leq 24 (\sqrt{d}+r)^{4}.
\end{align*}
Thus, we can view as if this model satisfies $B_3 = (\sqrt{d}+r)^{4}$ in Assumption \ref{assm2}. 


Then same as \eqref{ineq:pf:taylor3} we have with probability $1-\delta$,
\begin{align} \label{ineq:pf:taylor1_phase}
\ell_{n}(\beta^{\star}+z)- \ell_{n}(\beta^{\star}) 
&\leq - \frac{1}{2}z^{T}\cI_{S}z+2c^2B_2 \Tr(\cI_{S}^{-1})(\frac{\log \frac{d}{\delta}}{n})^{1.5}+2B^2_1B_2\|\cI_{S}^{-1}\|_{2}^2\log (\tilde{\kappa}^{-1/2}\alpha_1)(\frac{\log \frac{d}{\delta}}{n})^{2.5}\notag\\
&\quad\quad + \frac{2}{3}c^3B_3 \Tr(\cI_{S}^{-1})^{1.5}(\frac{\log \frac{d}{\delta}}{n})^{1.5}+\frac{2}{3}B^3_1B_3\|\cI_{S}^{-1}\|_{2}^3\log^{1.5} (\tilde{\kappa}^{-1/2}\alpha_1)(\frac{\log \frac{d}{\delta}}{n})^{3},
\end{align}

By Lemma \ref{lem:mle_phase}, we have \eqref{ineq:mle_ball}. Then same as \eqref{ineq:pf:taylor4} we have with probability at least $1-\delta$, 
\begin{align} \label{ineq:pf:taylor2_phase}
\ell_{n}(\beta_{\MLE})-\ell_{n}(\beta^{\star}) &\geq \frac{1}{2}(\Delta_{\beta_{\MLE}}-z)^{T} \cI_{S} (\Delta_{\beta_{\MLE}}-z) - \frac{1}{2}z^{T}\cI_{S}z\notag\\
&\quad - \cO(B_2 d^2(\frac{\log \frac{d}{\delta}}{n})^{1.5} + B_3 d^3(\frac{\log \frac{d}{\delta}}{n})^{1.5}).
\end{align}

Consequently, by \eqref{ineq:pf:taylor1_phase}, \eqref{ineq:pf:taylor2_phase} and the fact that $\ell_{n}(\beta_{\MLE})-\ell_{n}(\beta^{\star}+z)\leq 0$, we have 
\begin{align*}
 (\Delta_{\beta_{\MLE}}-z)^{T} \cI_{S} (\Delta_{\beta_{\MLE}}-z)
 &\leq\cO\bigg(B_2 \Tr(\cI_{S}^{-1})(\frac{\log \frac{d}{\delta}}{n})^{1.5}+B^2_1B_2\|\cI_{S}^{-1}\|_{2}^2\log (\tilde{\kappa}^{-1/2}\alpha_1) (\frac{\log \frac{d}{\delta}}{n})^{2.5}\notag\\
&\quad\quad + B_3 \Tr(\cI_{S}^{-1})^{1.5}(\frac{\log \frac{d}{\delta}}{n})^{1.5}+B^3_1B_3\|\cI_{S}^{-1}\|_{2}^3(\log (\tilde{\kappa}^{-1/2}\alpha_1))^{1.5}(\frac{\log \frac{d}{\delta}}{n})^{3}\notag\\
&\quad\quad+ B_2 d^{2}(\frac{\log \frac{d}{\delta}}{n})^{1.5}+B_3 d^3 (\frac{\log \frac{d}{\delta}}{n})^{1.5}
\bigg)
\end{align*}
Then, same as the proof of Lemma \ref{claim2}, we further have for any $\delta$, with probability at least $1-2\delta$,
\begin{align}\label{ineq:claim2_phase}
&(\beta_{\MLE}-\beta^{\star})^{T}\cI_{T}(\beta_{\MLE}-\beta^{\star}) \notag\\
&\lesssim \frac{\Tr(\cI_{T} \cI_{S}^{-1}) \log \frac{d}{\delta}}{n} \notag\\
&+ \cO\bigg(B_2\|\cI_{T}^{\frac{1}{2}}\cI_{S}^{-\frac{1}{2}}\|_{2}^{2} \Tr(\cI_{S}^{-1})(\frac{\log \frac{d}{\delta}}{n})^{1.5}
+B^2_1B_2\|\cI_{T}^{\frac{1}{2}}\cI_{S}^{-\frac{1}{2}}\|_{2}^{2}\|\cI_{S}^{-1}\|_{2}^2\log (\tilde{\kappa}^{-1/2}\alpha_1) (\frac{\log \frac{d}{\delta}}{n})^{2.5}\notag\\
&\quad\quad + B_3 \|\cI_{T}^{\frac{1}{2}}\cI_{S}^{-\frac{1}{2}}\|_{2}^{2}\Tr(\cI_{S}^{-1})^{1.5}(\frac{\log \frac{d}{\delta}}{n})^{1.5}
+B^3_1B_3\|\cI_{T}^{\frac{1}{2}}\cI_{S}^{-\frac{1}{2}}\|_{2}^{2}\|\cI_{S}^{-1}\|_{2}^3(\log (\tilde{\kappa}^{-1/2}\alpha_1))^{1.5}(\frac{\log \frac{d}{\delta}}{n})^{3}\notag\\
&\quad\quad+ B_2\|\cI_{T}^{\frac{1}{2}}\cI_{S}^{-\frac{1}{2}}\|_{2}^{2} d^2 (\frac{\log \frac{d}{\delta}}{n})^{1.5}
+B_3\|\cI_{T}^{\frac{1}{2}}\cI_{S}^{-\frac{1}{2}}\|_{2}^{2} d^3 (\frac{\log \frac{d}{\delta}}{n})^{1.5}\notag\\
&\quad\quad+B^2_{1}\|\cI_{T}^{\frac{1}{2}}\cI_{S}^{-\frac{1}{2}}\|_{2}^{2}\|\cI_{S}^{-1}\|_{2}\log (\kappa^{-1/2}\alpha_1) (\frac{\log \frac{d}{\delta}}{n})^2
\bigg)\notag\\
&= \frac{\Tr(\cI_{T} \cI_{S}^{-1}) \log \frac{d}{\delta}}{n} \notag\\
&+ \cO\bigg(d^2\|\cI_{T}^{\frac{1}{2}}\cI_{S}^{-\frac{1}{2}}\|_{2}^{2} \Tr(\cI_{S}^{-1})(\frac{\log \frac{d}{\delta}}{n})^{1.5}
+d^4 \|\cI_{T}^{\frac{1}{2}}\cI_{S}^{-\frac{1}{2}}\|_{2}^{2}\|\cI_{S}^{-1}\|_{2}^2\log (\tilde{\kappa}^{-1/2}\alpha_1) (\frac{\log \frac{d}{\delta}}{n})^{2.5}\notag\\
&\quad\quad + (\sqrt{d}+r)^{4} \|\cI_{T}^{\frac{1}{2}}\cI_{S}^{-\frac{1}{2}}\|_{2}^{2}\Tr(\cI_{S}^{-1})^{1.5}(\frac{\log \frac{d}{\delta}}{n})^{1.5}
+d^3 (\sqrt{d}+r)^{4} \|\cI_{T}^{\frac{1}{2}}\cI_{S}^{-\frac{1}{2}}\|_{2}^{2}\|\cI_{S}^{-1}\|_{2}^3(\log (\tilde{\kappa}^{-1/2}\alpha_1))^{1.5}(\frac{\log \frac{d}{\delta}}{n})^{3}\notag\\
&\quad\quad+ d^4\|\cI_{T}^{\frac{1}{2}}\cI_{S}^{-\frac{1}{2}}\|_{2}^{2} (\frac{\log \frac{d}{\delta}}{n})^{1.5}
+d^3 (\sqrt{d}+r)^{4}\|\cI_{T}^{\frac{1}{2}}\cI_{S}^{-\frac{1}{2}}\|_{2}^{2} (\frac{\log \frac{d}{\delta}}{n})^{1.5}\notag\\
&\quad\quad+d^2\|\cI_{T}^{\frac{1}{2}}\cI_{S}^{-\frac{1}{2}}\|_{2}^{2}\|\cI_{S}^{-1}\|_{2}\log (\kappa^{-1/2}\alpha_1) (\frac{\log \frac{d}{\delta}}{n})^2
\bigg)\notag\\
\end{align}
To guarantee $\frac{\Tr(\cI_{T} \cI_{S}^{-1}) \log \frac{d}{\delta}}{n}$ is the leading term, we only need $n \geq \cO (N_1 \log \frac{d}{\delta})$, where 
\begin{align*}
N_1:=\max\bigg\{
&\left(\frac{d^2\|\cI_{T}^{\frac{1}{2}}\cI_{S}^{-\frac{1}{2}}\|_{2}^{2}\Tr(\cI_{S}^{-1})}{\Tr(\cI_{T} \cI_{S}^{-1}) }\right)^2 ,
\left(\frac{d^4\|\cI_{T}^{\frac{1}{2}}\cI_{S}^{-\frac{1}{2}}\|_{2}^{2}\|\cI_{S}^{-1}\|_{2}^2 \log (\tilde{\kappa}^{-1/2}\alpha_1)}{\Tr(\cI_{T} \cI_{S}^{-1}) }\right)^{\frac23},
\left(\frac{(\sqrt{d}+r)^{4}\|\cI_{T}^{\frac{1}{2}}\cI_{S}^{-\frac{1}{2}}\|_{2}^{2}\Tr(\cI_{S}^{-1})^{1.5}}{\Tr(\cI_{T} \cI_{S}^{-1}) }\right)^2,
\notag\\
&\left(\frac{d^3 (\sqrt{d}+r)^{4}\|\cI_{T}^{\frac{1}{2}}\cI_{S}^{-\frac{1}{2}}\|_{2}^{2}\|\cI_{S}^{-1}\|_{2}^3 (\log (\tilde{\kappa}^{-1/2}\alpha_1))^{1.5}}{\Tr(\cI_{T} \cI_{S}^{-1}) }\right)^{\frac12}, 
\left(\frac{d^4\|\cI_{T}^{\frac{1}{2}}\cI_{S}^{-\frac{1}{2}}\|_{2}^{2}}{\Tr(\cI_{T} \cI_{S}^{-1}) }\right)^2, \left(\frac{d^3 (\sqrt{d}+r)^{4}\|\cI_{T}^{\frac{1}{2}}\cI_{S}^{-\frac{1}{2}}\|_{2}^{2}}{\Tr(\cI_{T} \cI_{S}^{-1}) }\right)^2, \notag \\
&\frac{d^2 \|\cI_{T}^{\frac{1}{2}}\cI_{S}^{-\frac{1}{2}}\|_{2}^{2}\|\cI_{S}^{-1}\|_{2}\log (\kappa^{-1/2}\alpha_1)}{\Tr(\cI_{T} \cI_{S}^{-1}) }
\bigg\}.
\end{align*}

That is, for any $\delta$, when $n \geq \cO (\max\{d^4,\Tr(\cI_{S}^{-1}), d\|\cI_{S}^{-1}\|_2 \log^{0.5}(\tilde{\kappa}^{-\frac12}\alpha_1), N_1\}\log \frac{d}{\delta})$,  with probability $1-2\delta$,
\begin{align*}
(\beta_{\MLE}-\beta^{\star})^{T}\cI_{T}(\beta_{\MLE}-\beta^{\star}) \lesssim \frac{\Tr(\cI_{T} \cI_{S}^{-1}) \log \frac{d}{\delta}}{n}.
\end{align*}
Then following the proof of Theorem \ref{thm:well_upper}, do Taylor expansion w.r.t. $\beta$ as the following:
\begin{align*}
R_{\beta^{\star}}(\beta_{\MLE})&=  
\bbE_{\substack{x\sim\bbP_T(X)\\y|x\sim f(y|x;\beta^{\star})}}\left[\ell(x,y,\beta_{\MLE})-\ell(x,y,\beta^{\star})\right] \notag \\ &\leq  \bbE_{\substack{x\sim\bbP_T(X)\\y|x\sim f(y|x;\beta^{\star})}} [\nabla \ell(x,y,\beta^{\star})]^{T}(\beta_{\MLE}-\beta^{\star}) \notag \\ 
&+ \frac{1}{2}(\beta_{\MLE}-\beta^{\star})^{T}\cI_{T}(\beta_{\MLE}-\beta^{\star}) +\frac{B_{3}}{6}\|\beta_{\MLE}-\beta^{\star}\|_{2}^{3}. \notag \\
&\leq \frac{c}{2}\frac{\Tr(\cI_{T} \cI_{S}^{-1}) \log \frac{d}{\delta}}{n} + \frac{c^3}{6}d^3 (\sqrt{d}+r)^{4} (\frac{\log \frac{d}{\delta}}{n})^{1.5}.
\end{align*}
with probability at least $1-2 \delta$. If we further assume $n\geq \cO((\frac{d^3 (\sqrt{d}+r)^{4}}{\Tr(\cI_{T} \cI_{S}^{-1}) })^{2} \log \frac{d}{\delta})$, it then holds that
\begin{align*} 
R_{\beta^{\star}}(\beta_{\MLE})  \leq c \frac{\Tr(\cI_{T} \cI_{S}^{-1}) \log \frac{d}{\delta}}{n}.
\end{align*}
Therefore we conclude that for any $\delta$, when $n \geq \cO (N \log \frac{d}{\delta}) $,  with probability at least $1-2\delta$, 
\begin{align*} 
R_{\beta^{\star}}(\beta_{\MLE})  \leq c \frac{\Tr(\cI_{T} \cI_{S}^{-1}) \log \frac{d}{\delta}}{n},
\end{align*} where
\begin{align*}
N&:=\max \{d^4, \Tr(\cI_{S}^{-1}), d\|\cI_{S}^{-1}\|_2\log^{0.5}(\tilde{\kappa}^{-\frac12}\alpha_1), N_1, (\frac{d^3 (\sqrt{d}+r)^{4}}{\Tr(\cI_{T} \cI_{S}^{-1}) })^{2}\} \\
&=\max\bigg\{ \left(\frac{d^2\|\cI_{T}^{\frac{1}{2}}\cI_{S}^{-\frac{1}{2}}\|_{2}^{2}\Tr(\cI_{S}^{-1})}{\Tr(\cI_{T} \cI_{S}^{-1}) }\right)^2,
\left(\frac{d^4\|\cI_{T}^{\frac{1}{2}}\cI_{S}^{-\frac{1}{2}}\|_{2}^{2}\|\cI_{S}^{-1}\|_{2}^2\log (\tilde{\kappa}^{-1/2}\alpha_1)}{\Tr(\cI_{T} \cI_{S}^{-1}) }\right)^{\frac23},
\left(\frac{(\sqrt{d}+r)^{4}\|\cI_{T}^{\frac{1}{2}}\cI_{S}^{-\frac{1}{2}}\|_{2}^{2}\Tr(\cI_{S}^{-1})^{1.5}}{\Tr(\cI_{T} \cI_{S}^{-1}) }\right)^2,
\notag\\
&\left(\frac{d^3 (\sqrt{d}+r)^{4}\|\cI_{T}^{\frac{1}{2}}\cI_{S}^{-\frac{1}{2}}\|_{2}^{2}\|\cI_{S}^{-1}\|_{2}^3(\log (\tilde{\kappa}^{-1/2}\alpha_1))^{1.5}}{\Tr(\cI_{T} \cI_{S}^{-1}) }\right)^{\frac12}, 
\left(\frac{d^4\|\cI_{T}^{\frac{1}{2}}\cI_{S}^{-\frac{1}{2}}\|_{2}^{2}}{\Tr(\cI_{T} \cI_{S}^{-1}) }\right)^2, \left(\frac{d^3 (\sqrt{d}+r)^{4}\|\cI_{T}^{\frac{1}{2}}\cI_{S}^{-\frac{1}{2}}\|_{2}^{2}}{\Tr(\cI_{T} \cI_{S}^{-1}) }\right)^2, \notag \\
&\frac{d^2 \|\cI_{T}^{\frac{1}{2}}\cI_{S}^{-\frac{1}{2}}\|_{2}^{2}\|\cI_{S}^{-1}\|_{2}\log (\kappa^{-1/2}\alpha_1)}{\Tr(\cI_{T} \cI_{S}^{-1})}, d^{4}, \Tr(\cI_{S}^{-1}), d\|\cI_{S}^{-1}\|_2\log^{0.5}(\tilde{\kappa}^{-\frac12}\alpha_1), (\frac{d^3 (\sqrt{d}+r)^{4}}{\Tr(\cI_{T} \cI_{S}^{-1}) })^{2}
\bigg\}.
\end{align*}
Now it remains to calculate $N$ and $\Tr (\cI_T \cI_S^{-1})$.  Similar to logistic regression (see Lemma \ref{logistic_lemma1} and \ref{logistic_lemma2}), we have the following two lemmas that characterize $\cI_S$ and $\cI_T$.

\begin{lemma} \label{phase_lemma1}
Under the conditions of Theorem \ref{thm:phase}, we have $\cI_S=U\diag(\lambda_1,\lambda_2,\ldots,\lambda_2)U^{T}$ and $\cI_T=U\diag(\lambda_1,\lambda_2+r^2\lambda_3,\lambda_2,\ldots,\lambda_2)U^{T}$ for an orthonormal matrix $U$. Where 
\begin{align*}
        &\lambda_{1}:=4\bbE_{x\sim\uni (\cS^{d-1}(\sqrt{d}))}[(\beta^{\star T}x)^{4}],\notag\\
        &\lambda_{2}:=4\bbE_{x\sim\uni (\cS^{d-1}(\sqrt{d}))}[(\beta^{\star T}x)^{2}(\beta_{\perp}^{\star T}x)^{2}],\notag\\
        &\lambda_{3}:=4\bbE_{x\sim\uni (\cS^{d-1}(\sqrt{d}))}[(\beta^{\star T}x)^{2}].
\end{align*}
\end{lemma}

\begin{lemma} \label{phase_lemma2}
Under the conditions of Theorem \ref{thm:phase}, there exist absolute constants $c,C, c'>0$ such that $c<\lambda_1, \lambda_2, \lambda_3 <C$, for $d \geq c'$.
\end{lemma}

The proofs for these two lemmas are in the next section. With Lemma \ref{phase_lemma1}, we have $\cI_T \cI_S^{-1} = U\diag(1,1+ r^{2}\frac{\lambda_3}{\lambda_2},\ldots,1)U^{T}$, $\cI_S^{-1} = U\diag(\frac{1}{\lambda_1},\frac{1}{\lambda_2},\ldots,\frac{1}{\lambda_2})U^{T}$. By Lemma \ref{phase_lemma2}, since $\lambda_1, \lambda_2, \lambda_3 = O(1)$, we have $\Tr (\cI_T \cI_S^{-1}) = d + r^{2} \frac{\lambda_3}{\lambda_2} \asymp d + r^{2}$, $\|\cI_T \cI_S^{-1}\|_2 = 1 + r^{2} \frac{\lambda_3}{\lambda_2} \asymp 1 + r^{2}$. Similarly  $\Tr (\cI_S^{-1}) = \lambda_1^{-1} + (d-1) \lambda_2^{-1} \asymp d $, $\| \cI_S^{-1}\|_2 = \max\{\lambda_1^{-1},\lambda_2^{-1}\} \asymp 1$, $\alpha_1=B_1\| \cI_S^{-1}\|_2^{\frac12} \asymp d$. Plug in these quantities, recall
\begin{align*}
     \kappa:=\frac{\Tr(\cI_{T} \cI_{S}^{-1}) }{\|\cI_T^{\frac12} \cI_S^{-1}\cI_T^{\frac12}\|_2} \asymp \frac{d+r^2}{1+r^2}
\end{align*}
we have 
\begin{align*}
N&=\max\bigg\{ d^6 \kappa^{-2},
d^{\frac83} \kappa^{-\frac23} \log^{\frac23} (\tilde{\kappa}^{-1/2}\alpha_1), d^3 (\sqrt{d}+r)^8 \kappa^{-2}, d^{\frac32} (\sqrt{d}+r)^2 \kappa^{-\frac12}\log^{\frac34} (\tilde{\kappa}^{-1/2}\alpha_1), d^8 \kappa^{-2}, d^6 (\sqrt{d}+r)^8 \kappa^{-2}, \notag \\
&d^2 \kappa^{-1}\log (\kappa^{-1/2}\alpha_1), d^{4}, d,d\log^{\frac12} (\tilde{\kappa}^{-1/2}\alpha_1), d^6 (\sqrt{d}+r)^8 \kappa^{-2} \|\cI_T \cI_S^{-1}\|^{-2}
\bigg\} \notag \\
&\stackrel{1 \leq \kappa \leq d}{=} \max\bigg\{d^6 (\sqrt{d}+r)^8 \kappa^{-2}, d^6 (\sqrt{d}+r)^8 \kappa^{-2} \|\cI_T \cI_S^{-1}\|^{-2}
\bigg\} \notag \\
&\stackrel{\|\cI_T \cI_S^{-1}\| \asymp 1+r^2 \geq 1}{=} d^6 (\sqrt{d}+r)^8 \kappa^{-2} \notag \\
& \asymp \frac{d^6 (\sqrt{d}+r)^8 (1+r^2)^2}{(d+r^2)^2}
\asymp d^6 (d+r^2)^2 (1+r^2)^2
\end{align*} 
We can see that when $r \leq 1$, $N \asymp d^8$. When $1 \leq r \leq \sqrt{d}$, $N \asymp d^8 r^4$.  When $ r \geq \sqrt{d}$, $N \asymp d^6 r^8$.

\end{proof}
\subsubsection{Proof of Lemma \ref{lem:mle_phase}}
In the following, we prove Lemma \ref{lem:mle_phase}. The intuition is that, although $\ell$ is not convex in $\beta$,  $\ell$ is quadratic in $M:=\beta\beta^{T}$.
\begin{proof}[Proof of Lemma \ref{lem:mle_phase}]
With a little bit abuse of notation, for matrix $M\in\bbR^{d\times d}$, we denote 
\begin{align*}
\ell(x,y,M):=\frac12(y-\langle xx^{T}, M\rangle)^2.
\end{align*}
Under the case where $M=\beta\beta^{T}$, we have
\begin{align*}
 \ell(x,y,M):=\frac12(y-\langle xx^{T}, \beta\beta^{T}\rangle)^2   =\frac12(y-(x^{T}\beta)^2)^2=\ell(x,y,\beta).
\end{align*}
We further denote
\begin{align*}
\ell_n(M):=\frac1n \sum^n_{i=1}\ell(x_i,y_i,M)=\frac{1}{2n}\sum^n_{i=1}(y_i-\langle x_ix_i^{T}, M\rangle)^2.
\end{align*}
and $M^{\star}:=\beta^{\star}\beta^{\star T}$.

It then holds that
\begin{align*}
\nabla\ell_n(M^{\star})=-\frac1n \sum^n_{i=1}\vec(x_i x_i^{T})\varepsilon_i,\,\, \nabla^2\ell_n(M^{\star})=\frac1n \sum^n_{i=1}\vec(x_i x_i^{T})\vec(x_i x_i^{T})^{T},\,\,
\nabla^3\ell_n(M)=0.
\end{align*}
Denote $\Sigma_S:=\bbE_{x\sim\bbP_S(X)}[\vec(xx^{T})\vec(xx^{T})^{T}]$, then by Lemma \ref{lem:concentration_vec} with $V=\Tr(\Sigma_S)$, $\alpha=2$, $B_{u}^{\alpha}=cd$ for some absolute constants $c,c'$, we have with probability at least $1-\delta$,
\begin{align}\label{phase:c1}
 \|\nabla\ell_n(M^{\star})\|_2   \leq c'\left(\sqrt{\frac{\Tr(\Sigma_S)\log \frac{d}{\delta}}{n}}+d (\log \frac{c^2 d^2}{\Tr(\Sigma_S)})^{\frac12}
 \frac{\log \frac{d}{\delta}}{n}\right).
\end{align}
By matrix Hoeffding, we have with probability at least $1-\delta$,
\begin{align}\label{phase:c2}
\Sigma_S -d^2\sqrt{\frac{8\log \frac{d}{\delta}}{n}} I_{d}   \preceq \nabla^{2} \ell_{n}(M^{\star}) \preceq \Sigma_S +d^2\sqrt{\frac{8\log \frac{d}{\delta}}{n}} I_{d}.
\end{align}

Before conducting further analysis, we need some characterizations of $\Sigma_S$. By the definition of $\Sigma_S$, we can see that the $((i,j),(k,l))$ entry of $\Sigma_S$ is $\bbE_{X \sim \bbP_S (X)}[X_i X_j X_k X_l]$. Since $X$ is symmetric and isotropic, we have
\begin{equation*}
    \bbE_{X \sim \bbP_S (X)}[X_i X_j X_k X_l] = 
    \begin{cases}
                \bbE_{X \sim \bbP_S (X)}[X_i^2 X_k^2] &\text{if $i=j, k=l$ and $i \neq k$} \\
                \bbE_{X \sim \bbP_S (X)}[X_i^2 X_j^2] 
                & \text{if $\{i,j\} = \{k,l\}$ and $i \neq j$} \\
                \bbE_{X \sim \bbP_S (X)}[X_i^4] 
                & \text{if $i=j=k=l$}  \\
                0 &\text{Otherwise}
    \end{cases}
\end{equation*}
For the calculation of moments, using (3a) in \cite{Uniformsphere} with $a=(1,0,\cdots, 0)^{T}$ and $\epsilon= \frac{1}{\sqrt{d}} X$,  we have $\bbE_{X \sim \bbP_S (X)}[X_1^4]=\frac{3d}{d+2}$, $\bbE_{X \sim \bbP_S (X)}[X_1^2 X_2^2]=\frac{d}{d+2}$. Since $X$ is isotropic, we have  
\begin{equation} \label{Sigma_S}
    (\Sigma_S)_{((i,j),(k,l))} = 
    \begin{cases}
                \frac{d}{d+2} &\text{if $i=j, k=l$ and $i \neq k$} \\
                \frac{d}{d+2} 
                & \text{if $\{i,j\} = \{k,l\}$ and $i \neq j$} \\
                \frac{3d}{d+2}
                & \text{if $i=j=k=l$}  \\
                0 &\text{Otherwise}
    \end{cases}
\end{equation}
Therefore 
\begin{align} \label{Trace_Sigma_S}
  \Tr(\Sigma_S)=\sum_{i,j}\bbE[X_i^2 X_j^2]=d(d-1)\frac{d}{d+2} + d\frac{3d}{d+2}=d^2.  
\end{align}

The following lemma characterizes the "minimum eigenvalue" of $\Sigma_S$ on a special subspace, which will be useful in our analysis.

\begin{lemma} \label{minimum_eigenvalue_Sigma_S}
    For any vector $a=(a_{ij})_{(i,j) \in [d] \times [d]}\in \bbR^{d^2}$ satisfies $a_{ij}$ = $a_{ji}$,
    \begin{align*}
        a^{T} \Sigma_S a \geq \frac{2d}{d+2} \|a\|_2^2.
    \end{align*}
\end{lemma}
\begin{proof}
   \begin{align*}
       a^{T} \Sigma_S a &=\sum_{i,j,k,l}a_{ij}a_{kl} (\Sigma_S)_{((i,j),(k,l))} \\
       &\stackrel{\text{by (\ref{Sigma_S})}}{=} \frac{d}{d+2} (\sum_{i \neq j}a_{ij}^2 + \sum_{i \neq j}a_{ij}a_{ji}+ \sum_{i \neq j}a_{ii}a_{jj} + 3\sum_{i}a_{ii}^2) \\
       &\stackrel{a_{ij}=a_{ji}}{=} \frac{d}{d+2} (2\sum_{i \neq j}a_{ij}^2 + \sum_{i \neq j}a_{ii}a_{jj} + 3\sum_{i}a_{ii}^2) \\
       &= \frac{d}{d+2} (2(\sum_{i \neq j}a_{ij}^2+\sum_{i}a_{ii}^2) + (\sum_{i \neq j}a_{ii}a_{jj} +\sum_{i}a_{ii}^2)) \\
       &= \frac{d}{d+2}(2\|a\|_2^2 + (\sum_{i}a_{ii})^2) \\
       & \geq \frac{2d}{d+2} \|a\|_2^2.
   \end{align*} 
\end{proof}
With Lemma \ref{minimum_eigenvalue_Sigma_S} and (\ref{Trace_Sigma_S}), we are now able to prove Lemma \ref{lem:mle_phase}. By Taylor expansion, we have for $M=\beta \beta^{T}$, $M^{\star}=\beta^{\star} \beta^{\star T}$, with probability at least $1-\delta$, 
\begin{align*}
\ell_n(M)-\ell_n(M^{\star})
&\stackrel{\nabla^3\ell_n \equiv 0}{=}\vec(M-M^{\star})^T\nabla\ell_n(M^{\star})+\frac12 \vec(M-M^{\star})^T\nabla^{2} \ell_{n}(M^{\star})\vec(M-M^{\star})\notag\\
&\stackrel{\text{by \eqref{phase:c1}}, \eqref{phase:c2}}{\geq}-c'\|M-M^{\star}\|_{F}\left(\sqrt{\frac{\Tr(\Sigma_S)\log \frac{d}{\delta}}{n}}+d (\log \frac{c^2 d^2}{\Tr(\Sigma_S)})^{\frac12}
 \frac{\log \frac{d}{\delta}}{n}\right)\notag\\
&\quad+\frac12 \vec(M-M^{\star})^T\Sigma_S\vec(M-M^{\star})-\|M-M^{\star}\|^2_{F}d^2\sqrt{\frac{8\log \frac{d}{\delta}}{n}}\notag\\
&\stackrel{\text{by Lemma \ref{minimum_eigenvalue_Sigma_S} and  (\ref{Trace_Sigma_S})}}{\geq} \left(\frac{d}{d+2}-d^2\sqrt{\frac{8\log \frac{d}{\delta}}{n}}\right)\|M-M^{\star}\|^2_{F}-c''\left(\sqrt{\frac{d^2\log \frac{d}{\delta}}{n}}+d\frac{\log \frac{d}{\delta}}{n}\right)\|M-M^{\star}\|_{F}\notag\\
&\geq \frac12 \|M-M^{\star}\|^2_{F}-c''\left(\sqrt{\frac{d^2\log \frac{d}{\delta}}{n}}+d\frac{\log \frac{d}{\delta}}{n}\right)\|M-M^{\star}\|_{F}
\end{align*}
when $n\geq \cO(d^4 \log\frac{d}{\delta})$.

We denote $M_{\MLE}:=\beta_{\MLE}\beta_{\MLE}^{T}$. Note that $\ell_n(M_{\MLE})-\ell_n(M^{\star})=\ell_n(\beta_{\MLE})-\ell_n(\beta^{\star})\leq 0$. Thus we have
\begin{align*}
   \frac12 \|M_{\MLE}-M^{\star}\|^2_{F}-c''\left(\sqrt{\frac{d^2\log \frac{d}{\delta}}{n}}+d\frac{\log \frac{d}{\delta}}{n}\right)\|M_{\MLE}-M^{\star}\|_{F}\leq 0 , 
\end{align*}
which implies
\begin{align*}
\|M_{\MLE}-M^{\star}\|_{F}\lesssim    \left(\sqrt{\frac{d^2\log \frac{d}{\delta}}{n}}+d\frac{\log \frac{d}{\delta}}{n}\right)\lesssim  \sqrt{\frac{d^2\log \frac{d}{\delta}}{n}}.
\end{align*}

Thus so far we have shown, if $n\geq \cO(d^4 \log\frac{d}{\delta})$, then with probability at least $1-\delta$, we have
\begin{align*}
 \|M_{\MLE}-M^{\star}\|_{F}\lesssim    \sqrt{\frac{d^2\log \frac{d}{\delta}}{n}}.
\end{align*}
By Lemma 6 in \cite{ge2017no}, we further have
\begin{align*}
\min\{\|\beta_{\MLE}-\beta^{\star}\|_2,\|\beta_{\MLE}+\beta^{\star}\|_2\}\lesssim\frac{1}{\|\beta^{\star}\|_2}    \|M_{\MLE}-M^{\star}\|_{F}
\lesssim \sqrt{\frac{d^2\log \frac{d}{\delta}}{n}}.
\end{align*}
\end{proof}

\subsubsection{Proofs for Lemma \ref{phase_lemma1} and \ref{phase_lemma2}}
The proofs for Lemma \ref{phase_lemma1} and \ref{phase_lemma2} are similar to proofs for Lemma \ref{logistic_lemma1} and \ref{logistic_lemma2}. 

\begin{proof}[Proof of Lemma \ref{phase_lemma1}]
By definition,
\begin{align*}
   \cI_S := 4\bbE_{x\sim\uni (\cS^{d-1}(\sqrt{d}))}[xx^{T}(x^T\beta^{\star})^2]    
\end{align*}
Let $z \sim \cN(0,I_d)$, then $x$ and $z\frac{\sqrt{d}}{\|z\|_{2}}$ have the same distribution. Therefore 
\begin{align*}
   \cI_S &= 4\bbE_{x\sim\uni (\cS^{d-1}(\sqrt{d}))}[xx^{T}(x^T\beta^{\star})^2]   \\
   &= 4\bbE_{z\sim\cN (0,I_d)}[zz^{T}\frac{d}{\|z\|^{2}_{2}} (\beta^{\star T}z \cdot \frac{\sqrt{d}}{\|z\|_{2}})^2]   \\
   &= 4\bbE_{z\sim\cN (0,I_d)}[(\beta^{\star}\beta^{\star T} + U_{\perp}U_{\perp}^{T})zz^{T} (\beta^{\star T}z)^2 \frac{d^2}{\|z\|^{4}_{2}}]   
\end{align*}
where $[\beta^{\star},U_{\perp}] \in \bbR^{d \times d}$ is a orthogonal basis. 

With this expression, we first prove $\beta^{\star}$ is an eigenvector of $\cI_S$ with corresponding eigenvalue $\lambda_1$.
\begin{align*}
   \cI_S \beta^{\star} &= 4\bbE_{z\sim\cN (0,I_d)}[(\beta^{\star}\beta^{\star T} + U_{\perp}U_{\perp}^{T})zz^{T} (\beta^{\star T}z)^2 \frac{d^2}{\|z\|^{4}_{2}}]  \beta^{\star}  \\
   &= 4\bbE_{z\sim\cN (0,I_d)}[\beta^{\star}\beta^{\star T}zz^{T} (\beta^{\star T}z)^2 \frac{d^2}{\|z\|^{4}_{2}}\beta^{\star}]  \\
   &+ 4\bbE_{z\sim\cN (0,I_d)}[ U_{\perp}U_{\perp}^{T}zz^{T} (\beta^{\star T}z)^2 \frac{d^2}{\|z\|^{4}_{2}}\beta^{\star}]  \\
   &= 4\bbE_{z\sim\cN (0,I_d)}[(\beta^{\star T}z)^4 \frac{d^2}{\|z\|^{4}_{2}}] \beta^{\star} \\
   &+ 4\bbE_{z\sim\cN (0,I_d)}[ U_{\perp}U_{\perp}^{T}zz^{T} (\beta^{\star T}z)^2 \frac{d^2}{\|z\|^{4}_{2}}\beta^{\star}]\\
   &= \lambda_1 \beta^{\star} + 4\bbE_{z\sim\cN (0,I_d)}[ U_{\perp}U_{\perp}^{T}zz^{T} (\beta^{\star T}z)^2 \frac{d^2}{\|z\|^{4}_{2}}\beta^{\star}]. 
\end{align*}
Therefore we only need to prove 
\begin{align*}
    \bbE_{z\sim\cN (0,I_d)}[ U_{\perp}U_{\perp}^{T}zz^{T} (\beta^{\star T}z)^2 \frac{d^2}{\|z\|^{4}_{2}}\beta^{\star}]=0.
\end{align*}
In fact,
\begin{align*}
    & \qquad \bbE_{z\sim\cN (0,I_d)}[ U_{\perp}^{T}zz^{T} (\beta^{\star T}z)^2 \frac{d^2}{\|z\|^{4}_{2}}\beta^{\star}] \\
    & = \bbE_{z\sim\cN (0,I_d)}[ (U_{\perp}^{T}z) (\beta^{\star T}z)^3 \frac{d^2}{\|z\|^{4}_{2}}] \\
    & = \bbE_{z\sim\cN (0,I_d)}[(\frac{d}{|A|^{2} + \|B\|^{2}})^2 A^3 B ] 
\end{align*}
where we let $A:=z^{T}\beta^{\star}$, $B:=U_{\perp}^{T}z$. Notice that by the property of $z\sim \cN(0,I_d)$, $A$ and $B$ are independent. Also, $B$ is symmetric, i.e., $B$ and $-B$ have the same distribution. Therefore 
\begin{align*}
    \bbE_{z\sim\cN (0,I_d)}[ U_{\perp}U_{\perp}^{T}zz^{T} (\beta^{\star T}z)^2 \frac{d^2}{\|z\|^{4}_{2}}\beta^{\star}]=\bbE_{z\sim\cN (0,I_d)}[(\frac{d}{|A|^{2} + \|B\|^{2}})^2 A^3 B ]=0.
\end{align*}

Next we will prove that for any $\beta_{\perp}$ such that $\|\beta_{\perp}\|_{2}=1$, $\beta^{\star T}\beta_{\perp}=0$, $\beta_{\perp}$ is an eigenvector of $\cI_{S}$ with corresponding eigenvalue $\lambda_2$.  Let $[\beta_{\perp}, U]$ be an orthogonal basis ($\beta^{\star}$ is the first column of $U$).

\begin{align*}
   \cI_S \beta_{\perp} &= 4\bbE_{z\sim\cN (0,I_d)}[(\beta_{\perp}\beta_{\perp}^T + UU^{T})zz^{T} (\beta^{\star T}z)^2 \frac{d^2}{\|z\|^{4}_{2}}]  \beta_{\perp}  \\
   &= 4\bbE_{z\sim\cN (0,I_d)}[\beta_{\perp}\beta_{\perp}^Tzz^{T} (\beta^{\star T}z)^2 \frac{d^2}{\|z\|^{4}_{2}}\beta_{\perp}]  \\
   &+ 4\bbE_{z\sim\cN (0,I_d)}[ UU^{T}zz^{T} (\beta^{\star T}z)^2 \frac{d^2}{\|z\|^{4}_{2}}\beta_{\perp}]  \\
   &= 4\bbE_{z\sim\cN (0,I_d)}[(\beta_{\perp}^T z)^2 (\beta^{\star T}z)^2 \frac{d^2}{\|z\|^{4}_{2}}] \beta_{\perp}\\
   &+ 4\bbE_{z\sim\cN (0,I_d)}[ UU^{T}zz^{T} (\beta^{\star T}z)^2 \frac{d^2}{\|z\|^{4}_{2}}\beta_{\perp}]  \\
   &= \lambda_2 \beta_{\perp} + 0\\
   &= \lambda_2 \beta_{\perp}
\end{align*}
Here 
\begin{align*}
  4\bbE_{z\sim\cN (0,I_d)}[ UU^{T}zz^{T} (\beta^{\star T}z)^2 \frac{d^2}{\|z\|^{4}_{2}}\beta_{\perp}]   =0 
\end{align*}
because of a similar reason as in the previous part.

For $\cI_T$, the proving strategy is similar.  For $x \sim\uni(\cS^{d-1}(\sqrt{d}))+v$ on the target domain, where $v= r\beta_{\perp}^{\star}$, let $w=x-v=x-r\beta_{\perp}^{\star}$, then $w \sim\uni(\cS^{d-1}(\sqrt{d}))$. Let $z \sim \cN(0,I_d)$, then $w$ and $z\frac{\sqrt{d}}{\|z\|_{2}}$ have the same distribution. We have
\begin{align*}
    \cI_T &= 4\bbE_{x\sim\uni (\cS^{d-1}(\sqrt{d})) +v}[xx^{T}(x^T\beta^{\star})^2]  \\
   &=4\bbE_{w\sim\uni (\cS^{d-1}(\sqrt{d})) }[(w+v)(w+v)^{T}((w+v)^T\beta^{\star})^2]  \\
   &\stackrel{v^{T}\beta^{\star}=0}{=} 4\bbE_{w\sim\uni (\cS^{d-1}(\sqrt{d}))}[(ww^{T}+wv^{T}+vw^{T}+vv^{T})(w^{T} \beta^{\star})^2]   
\end{align*}
Therefore 
\begin{align*}
   \cI_T \beta^{\star} &= 4 \bbE_{w\sim\uni (\cS^{d-1}(\sqrt{d}))}[(ww^{T}+wv^{T}+vw^{T}+vv^{T})(w^{T} \beta^{\star})^2]   \beta^{\star} \\
   & \stackrel{v^{T}\beta^{\star}=0}{=} 4\bbE_{w\sim\uni (\cS^{d-1}(\sqrt{d}))}[ww^{T}(w^{T} \beta^{\star})^2]   \beta^{\star} \\
   &= \cI_S \beta^{\star} \\
   &= \lambda_1 \beta^{\star},
\end{align*}
where the last line follows from the previous proofs. Similarly, for any $\tilde{\beta_{\perp}}$ such that $\|\tilde{\beta_{\perp}}\|_{2}=1$, $\beta_{\perp}^{\star T}\tilde{\beta_{\perp}}=0$, 
\begin{align*}
   \cI_T \tilde{\beta_{\perp}} &= 4 \bbE_{w\sim\uni (\cS^{d-1}(\sqrt{d}))}[(ww^{T}+wv^{T}+vw^{T}+vv^{T})(w^{T} \beta^{\star})^2]   \tilde{\beta_{\perp}} \\
   & \stackrel{v^{T}\tilde{\beta_{\perp}}=0}{=} 4\bbE_{w\sim\uni (\cS^{d-1}(\sqrt{d}))}[ww^{T}(w^{T} \beta^{\star})^2]\tilde{\beta_{\perp}} \\
   &= \cI_S \tilde{\beta_{\perp}} \\
   &= \lambda_2 \tilde{\beta_{\perp}}.
\end{align*}
For $\beta_{\perp}^{\star}$, 
\begin{align*}
   \cI_T \beta_{\perp}^{\star} &=  4 \bbE_{w\sim\uni (\cS^{d-1}(\sqrt{d}))}[(ww^{T}+wv^{T}+vw^{T}+vv^{T})(w^{T} \beta^{\star})^2]   \beta_{\perp}^{\star} \\
   &=  4 \bbE_{w\sim\uni (\cS^{d-1}(\sqrt{d}))}[ww^{T}(w^{T} \beta^{\star})^2]    \beta_{\perp}^{\star} +  4 \bbE_{w\sim\uni (\cS^{d-1}(\sqrt{d}))}[wv^{T}(w^{T} \beta^{\star})^2]   \beta_{\perp}^{\star} \\
   &+  4 \bbE_{w\sim\uni (\cS^{d-1}(\sqrt{d}))}[vw^{T}(w^{T} \beta^{\star})^2]    \beta_{\perp}^{\star} +  4 \bbE_{w\sim\uni (\cS^{d-1}(\sqrt{d}))}[vv^{T}(w^{T} \beta^{\star})^2]   \beta_{\perp}^{\star} \\
   & := I_1+ I_2+ I_3 +I_4.
\end{align*}
As in the previous proofs, 
\begin{align*}
    I_1=\cI_S \beta_{\perp}^{\star} = \lambda_2 \beta_{\perp}^{\star}.
\end{align*}
\begin{align*}
   I_2 &=  4 \bbE_{w\sim\uni (\cS^{d-1}(\sqrt{d}))}[wv^{T}(w^{T} \beta^{\star})^2]   \beta_{\perp}^{\star} \\
   &\stackrel{v=r\beta_{\perp}^{\star}}{=} 4r \bbE_{w\sim\uni (\cS^{d-1}(\sqrt{d}))}[w(\beta_{\perp}^{\star T}\beta_{\perp}^{\star})(w^{T} \beta^{\star})^2]    \\
    &\stackrel{\|\beta_{\perp}^{\star}\|=1}{=} 4r \bbE_{w\sim\uni (\cS^{d-1}(\sqrt{d}))}[w(w^{T} \beta^{\star})^2]  \\
    &=0.
\end{align*}
where the last lines follows from $w$ is symmetric and $w(w^{T} \beta^{\star})^2$ is a odd function of $w$.
\begin{align*}
   I_3 &= 4 \bbE_{w\sim\uni (\cS^{d-1}(\sqrt{d}))}[vw^{T}(w^{T} \beta^{\star})^2]    \beta_{\perp}^{\star} \\
   &\stackrel{v=r\beta_{\perp}^{\star}}{=} 4r \bbE_{w\sim\uni (\cS^{d-1}(\sqrt{d}))}[\beta_{\perp}^{\star} w^{T}\beta_{\perp}^{\star}(w^{T} \beta^{\star})^2]  \\
   &= 4r \bbE_{w\sim\uni (\cS^{d-1}(\sqrt{d}))}[ (w^{T}\beta_{\perp}^{\star})(w^{T} \beta^{\star})^2] \beta_{\perp}^{\star} \\
    &=0.
\end{align*}
where the last lines follows from $w$ is symmetric and $ (w^{T}\beta_{\perp}^{\star})(w^{T} \beta^{\star})^2$ is a odd function of $w$.
\begin{align*}
   I_4 &=  4 \bbE_{w\sim\uni (\cS^{d-1}(\sqrt{d}))}[vv^{T}(w^{T} \beta^{\star})^2]   \beta_{\perp}^{\star}\\
   &\stackrel{v=r\beta_{\perp}^{\star}}{=} 4 r^{2} \bbE_{w\sim\uni (\cS^{d-1}(\sqrt{d}))}[\beta_{\perp}^{\star}\beta_{\perp}^{\star T} \beta_{\perp}^{\star} (w^{T} \beta^{\star})^2]  \\
    &\stackrel{\|\beta_{\perp}^{\star}\|=1}{=} 4 r^{2} \bbE_{w\sim\uni (\cS^{d-1}(\sqrt{d}))}[\beta_{\perp}^{\star}(w^{T} \beta^{\star})^2]  \\
    &=r^{2}\lambda_3 \beta_{\perp}^{\star}.
\end{align*}
Combine the calculations of $I_1,I_2,I_3,I_4$, we have 
\begin{align*}
    \cI_{T}\beta_{\perp}^{\star} &= I_1 + I_2 + I_3 + I_4 \\
    &= \lambda_2 \beta_{\perp}^{\star} + r^{2} \lambda_3 \beta_{\perp}^{\star} \\
    &= (\lambda_2 + r^{2} \lambda_3) \beta_{\perp}^{\star}.
\end{align*}
In conclusion, we have $\cI_S=U\diag(\lambda_1,\lambda_2,\ldots,\lambda_2)U^{T}$ and $\cI_T=U\diag(\lambda_1,\lambda_2+r^2\lambda_3,\lambda_2,\ldots,\lambda_2)U^{T}$ for an orthonormal matrix $U$, where $U= [\beta^{\star}, \beta_{\perp}^{\star}, \cdots]$.
\end{proof}

\begin{proof}[Proof of Lemma \ref{phase_lemma2}]
Recall the definition of $\lambda_1, \lambda_2, \lambda_3$:
\begin{align*}
        &\lambda_{1}:=4\bbE_{x\sim\uni (\cS^{d-1}(\sqrt{d}))}[(\beta^{\star T}x)^{4}] = 4\bbE_{z\sim\cN (0,I_d)}[(\beta^{\star T}z)^4 \frac{d^2}{\|z\|^{4}_{2}}],\notag\\
        &\lambda_{2}:=4\bbE_{x\sim\uni (\cS^{d-1}(\sqrt{d}))}[(\beta^{\star T}x)^{2}(\beta_{\perp}^{\star T}x)^{2}]= 4\bbE_{z\sim\cN (0,I_d)}[(\beta^{\star T}z)^2 (\beta_{\perp}^{\star T} z)^2 \frac{d^2}{\|z\|^{4}_{2}}] ,\notag\\
        &\lambda_{3}:=4\bbE_{x\sim\uni (\cS^{d-1}(\sqrt{d}))}[(\beta^{\star T}x)^{2}]= 4\bbE_{z\sim\cN (0,I_d)}[(\beta^{\star T}z)^2 \frac{d}{\|z\|^{2}_{2}}].
\end{align*}
Next we will show that there exists constants $c,C,c'>0$ such that when $d\geq c'$, we have $c \leq \lambda_1 \leq C$. The proofs for $\lambda_2$ and $\lambda_3$ are similar. 
\ref{gaussian_norm_concentration}
With this concentration, we do the following truncation:
\begin{align*}
    \frac14 \lambda_{1}&= \bbE_{z\sim\cN (0,I_d)}[(\beta^{\star T}z)^4 \frac{d^2}{\|z\|^{4}_{2}}] \\
    &=\bbE_{z\sim\cN (0,I_d)}[(\beta^{\star T}z)^4 \frac{d^2}{\|z\|^{4}_{2}}\bbI_{\frac{\|z\|}{\sqrt{d}} \in [\frac12, \frac32]}] + \bbE_{z\sim\cN (0,I_d)}[(\beta^{\star T}z)^4 \frac{d^2}{\|z\|^{4}_{2}}\bbI_{\frac{\|z\|}{\sqrt{d}} \notin [\frac12, \frac32]}] \\
    &:= J_1 + J_2.
\end{align*}
For $J_2$, it is obvious that
\begin{align} \label{J_2}
    0\leq J_2 \leq d^2 \bbP(\frac{\|z\|}{\sqrt{d}} \notin [\frac12, \frac32]) \leq 2d^2 e^{-cd}.
\end{align}

For upper bound of $J_1$, 
\begin{align*}
    J_1 &= \bbE_{z\sim\cN (0,I_d)}[(\beta^{\star T}z)^4 \frac{d^2}{\|z\|^{4}_{2}}\bbI_{\frac{\|z\|}{\sqrt{d}} \in [\frac12, \frac32]}] \\
    & \leq \bbE_{z\sim\cN(0,I_d)} [16(\beta^{\star T}z)^{4}] =48.
\end{align*}
Therefore 
\begin{align*}
    \frac14 \lambda_1 = J_1 + J_2 \leq 48 + 2d^2 e^{-cd}.
\end{align*}
It's obvious that there exists an absolute constant $c'$ such that when $d \geq c'$, $\frac14 \lambda_1 \leq 50$.

For lower bound of $J_1$, we have 
\begin{align*}
    J_1 &= \bbE_{z\sim\cN (0,I_d)}[(\beta^{\star T}z)^4 \frac{d^2}{\|z\|^{4}_{2}}\bbI_{\frac{\|z\|}{\sqrt{d}} \in [\frac12, \frac32]}] \\
    & \geq \bbE_{z\sim\cN(0,I_d)} [(\frac23)^4(\beta^{\star T}z)^{4}] =(\frac23)^4 \cdot 3.
\end{align*}
Therefore 
\begin{align*}
    \frac14 \lambda_1 = J_1 + J_2 \geq (\frac23)^4 \cdot 3
\end{align*}
Therefore it's obvious that there exists an absolute constant $c'$ such that when $d \geq c'$, $\frac14 \lambda_1 \geq \frac12$. The proofs for $\lambda_2$ and $\lambda_3$ are almost the same. 

\end{proof}

%% file: pf_misspecified.tex
\section{Proofs for Section \ref{mis-specified}}

\subsection{Poofs for Proposition \ref{prop:mle_consistent}}
\begin{proof}
We consider the case where $Y=X^2+\varepsilon$, $\varepsilon\sim\cN(0,1)$, $\varepsilon\indep X$, and we have $X\sim\cN(-10,1)$ on the source domain and $X\sim\cN(10,1)$ on the target domain. Then the optimal linear fit on the target is given by 
\begin{align*}
\beta^{\star} = \argmin_{\beta\in\bbR}\bbE_{(x,y)\sim\bbP_T(X,Y)}\left[(y-x\beta)^2\right]=\left(\bbE_{x\sim\cN(10,1)}[x^2]\right)^{-1}\bbE_{x\sim\cN(10,1)}[x^3]>0.
\end{align*}
However, the linear fit learned via classical MLE asymptotically behaves as
\begin{align*}
\beta_{\MLE}
&=\argmin_{\beta\in\bbR}\frac{1}{2n} \sum^n_{i=1}(y_i-x_i\beta)^2
=\left(\frac1n\sum^n_{i=1}x^2_i\right)^{-1}\left(\frac1n\sum^n_{i=1}x_i y_i\right)\\
&\xrightarrow[]{n\rightarrow\infty}
\left(\bbE_{x\sim\cN(-10,1)}[x^2]\right)^{-1}\bbE_{x\sim\cN(-10,1)}[x^3]<0.
\end{align*}
Hence, the classical MLE losses consistency.
For MWLE, we have
\begin{align*}
\beta_{\MWLE}
&=\argmin_{\beta\in\bbR}\frac{1}{2n}\sum^n_{i=1}w(x_i) (y_i-x_i\beta)^2\\
&=\left(\frac1n\sum^n_{i=1}w(x_i) x^2_i\right)^{-1}\left(\frac1n\sum^n_{i=1}w(x_i) x_i y_i\right)\xrightarrow[]{n\rightarrow\infty}\beta^{\star},
\end{align*}
which asymptotically provides a good estimator. 
\end{proof}

\subsection{Proofs for Theorem \ref{thm:mis_upper}}
The detailed version of Theorem \ref{thm:mis_upper} is stated as the following.
\begin{theorem}\label{thm_exact:mis_upper}
Suppose the function class $\cF$ satisfies Assumption \ref{assm:mis_upper}. Let $G_w := G_w(M)$ and $H_w := H_w(M)$. For any $\delta\in (0,1)$, if $n\geq c \max\{N^{\star}\log(d/\delta), N(\delta), N'(\delta)\}$, then with probability at least $1-3\delta$, we have
\begin{align*}
&R_{M}(\beta_{\MWLE}) 
\leq c\frac{\Tr\left(G_w H^{-1}_w\right)\log \frac{d}{\delta}}{n}
\end{align*}
for an absolute constant $c$. Here 
\begin{align*}
N^{\star}
:= W^2 \cdot\max\{\lambda^{-1}\tilde{\alpha}_1^{2}\log^{2\gamma}(W^2\lambda^{-1}\tilde{\alpha}^2_1),\tilde{\alpha}_2^{2},\lambda\tilde{\alpha}_3^{2}\},
\end{align*}
where $\tilde{\alpha}_1:=B_1\|H^{-1}_w\|_2^{0.5}$, $\tilde{\alpha}_2:=B_2\|H^{-1}_w\|_2$, $\tilde{\alpha}_3:=B_3\|H^{-1}_w\|_2^{1.5}$, and $\lambda:=\Tr(G_w H_w^{-2})/\|H^{-1}_w\|_2$.
\end{theorem}

The proofs for Theorem \ref{thm_exact:mis_upper} is similar to proofs for Theorem \ref{thm_exact:well_upper}. For notation simplicity, through out the proofs for Theorem \ref{thm_exact:mis_upper}, let $\beta^{\star}:=\beta^{\star}(M)$, $H_w:= H_w(M)$, $  G_w:= G_w(M)$. 
We first state two main lemmas, which capture the distance between $\beta_{\MWLE}$ and $\beta^{\star}$ under different measurements.

\begin{lemma}\label{claim1_mis}
Suppose Assumption \ref{assm:mis_upper} holds. For any $\delta\in (0,1)$ and any $n\geq c\max\{N_1\log(d/\delta), N(\delta), N'(\delta)\}$, with probability at least $1-2\delta$, we have $\beta_{\MWLE}\in\bbB_{\beta^{\star}}(c\sqrt{\frac{\Tr(G_w H^{-2}_w)\log \frac{d}{\delta}}{n}})$ for some absolute constant $c$.
Here 
\begin{align*}
N_1:=\max\bigg\{
&W^2 B_{2}^2\|H_w^{-1}\|_{2}^2,
W^2 B_{3}^2\Tr(G_w H_w^{-2})\|H_w^{-1}\|_{2}^2,
\left(\frac{W^3 B^2_1B_2\|H_w^{-1}\|_{2}^3\log^{2\gamma}(W\lambda^{-1/2}\tilde{\alpha}_1)}{\Tr(G_w H_w^{-2})}\right)^{\frac23},\notag\\
&\left(\frac{W^4B^3_1B_3\|H_w^{-1}\|_{2}^4\log^{3\gamma}(W\lambda^{-1/2}\tilde{\alpha}_1)}{\Tr(G_w H_w^{-2})}\right)^{\frac12},
\frac{ W^2 B_{1}^2\|H_w^{-1}\|_{2}^2\log^{2\gamma}(W\lambda^{-1/2}\tilde{\alpha}_1)}{\Tr(G_w H_w^{-2})}
\bigg\}.
\end{align*}
\end{lemma}

\begin{lemma}\label{claim2_mis}
Suppose Assumption \ref{assm:mis_upper} holds. For any $\delta\in (0,1)$ and any $n\geq c\max\{N_1\log(d/\delta), N_2\log(d/\delta), N(\delta), N'(\delta)\}$, with probability at least $1-3\delta$, we have
\begin{align*}
 \|H_w^{\frac{1}{2}}(\beta_{\MWLE}-\beta^{\star})\|_{2}^{2}\leq  c \frac{\Tr(G_w H_w^{-1}) \log \frac{d}{\delta}}{n} .
\end{align*}
for some absolute constant $c$. Here $N_1$ is defined in Lemma \ref{claim1_mis} and
\begin{align*}
N_2:=\max\bigg\{
&\left(\frac{W B_2\Tr(G_w H_w^{-2})}{\Tr(G_w H_w^{-1}) }\right)^{2},
\left(\frac{W B_3\Tr(G_w H_w^{-2})^{1.5}}{\Tr(G_w H_w^{-1}) }\right)^{2},
\left(\frac{ W^3 B^2_1B_2\|H_w^{-1}\|_{2}^2\log^{2\gamma}(W\lambda^{-1/2}\tilde{\alpha}_1)}{\Tr(G_w H_w^{-1})}\right)^{\frac23},\notag\\
&\left(\frac{W^4B^3_1B_3\|H_w^{-1}\|_{2}^3\log^{3\gamma}(W\lambda^{-1/2}\tilde{\alpha}_1)}{\Tr(G_w H_w^{-1}) }\right)^{\frac12},
\frac{W^2 B_{1}^2\|H_w^{-1}\|_{2}\log^{2\gamma}(W\lambda^{-1/2}\tilde{\alpha}_1)}{\Tr(G_w H_w^{-1}) }
\bigg\}.
\end{align*}
\end{lemma}

The proofs for Lemma \ref{claim1_mis} and \ref{claim2_mis} are delayed to the end of this subsection. With these two lemmas, we can now state the proof for Theorem \ref{thm_exact:mis_upper}.
\begin{proof}[Proof of Theorem \ref{thm_exact:mis_upper}]
By Assumption \ref{assm6} and \ref{assm2_mis}, we can do Taylor expansion w.r.t. $\beta$ as the following:
\begin{align*}
R_{M}(\beta_{\MWLE})&=  
\bbE_{(x,y) \sim \bbP_{T}(x,y)}\left[\ell(x,y,\beta_{\MWLE})-\ell(x,y,\beta^{\star})\right] \notag \\ &\leq  \bbE_{(x,y) \sim \bbP_{T}(x,y)} [\nabla \ell(x,y,\beta^{\star})]^{T}(\beta_{\MWLE}-\beta^{\star}) \notag \\ 
&+ \frac{1}{2}(\beta_{\MWLE}-\beta^{\star})^{T}H_w(\beta_{\MWLE}-\beta^{\star}) +\frac{WB_{3}}{6}\|\beta_{\MWLE}-\beta^{\star}\|_{2}^{3}. 
\end{align*}
Applying Lemma \ref{claim1_mis} and \ref{claim2_mis}, we know for any $\delta$ and any $n\geq c\max\{N_1\log(d/\delta), N_2\log(d/\delta), N(\delta), N'(\delta)\}$, with probability at least $1-3\delta$, we have
\begin{align*}
(\beta_{\MWLE}-\beta^{\star})^{T}H_w(\beta_{\MWLE}-\beta^{\star})\leq   c\frac{\Tr(G_w H_w^{-1}) \log \frac{d}{\delta}}{n}
\end{align*} and
\begin{align*}
   \|\beta_{\MWLE}-\beta^{\star}\|_{2} \leq c\sqrt{\frac{ \Tr(G_w H_w^{-2})\log \frac{d}{\delta}}{n}}.
\end{align*}
Also notice that, $\bbE_{(x,y) \sim \bbP_{T}(x,y)} [\nabla \ell(x,y,\beta^{\star})]=0$. Therefore, with probability at least $1- 3\delta $, we have
\begin{align*} 
R_{M}(\beta_{\MWLE})  &\leq \frac{c}{2}\frac{\Tr(G_w H_w^{-1}) \log \frac{d}{\delta}}{n} + \frac{c^3}{6}WB_3\Tr(G_w H_w^{-2})^{1.5}(\frac{ \log \frac{d}{\delta}}{n})^{1.5}.
\end{align*}
If we further have $n\geq c(\frac{WB_3\Tr(G_w H_w^{-2})^{1.5}}{\Tr(G_w H_w^{-1})})^{2}\log (d/\delta)$, it then holds that
\begin{align*} 
R_{M}(\beta_{\MWLE})  &\leq c\frac{\Tr(G_w H_w^{-1}) \log \frac{d}{\delta}}{n} .
\end{align*}
Note that
\begin{align*}
&\max\bigg\{N_1,N_2,
\left(\frac{WB_3\Tr(G_w H_w^{-2})^{1.5}}{\Tr(G_w H_w^{-1})}\right)^{2}
\bigg\}\notag\\
&=\max\bigg\{W^2 B_{2}^2\|H_w^{-1}\|_{2}^2,\,
W^2 B_{3}^2\Tr(G_w H_w^{-2})\|H_w^{-1}\|_{2}^2,\,
\left(\frac{W^3 B^2_1B_2\|H_w^{-1}\|_{2}^3\log^{2\gamma}(W\lambda^{-1/2}\tilde{\alpha}_1)}{\Tr(G_w H_w^{-2})}\right)^{\frac23},\notag\\
&\left(\frac{W^4B^3_1B_3\|H_w^{-1}\|_{2}^4\log^{3\gamma}(W\lambda^{-1/2}\tilde{\alpha}_1)}{\Tr(G_w H_w^{-2})}\right)^{\frac12},\,
\frac{ W^2 B_{1}^2\|H_w^{-1}\|_{2}^2\log^{2\gamma}(W\lambda^{-1/2}\tilde{\alpha}_1)}{\Tr(G_w H_w^{-2})}\bigg\}\notag\\
&=W^2\cdot\max\{\tilde{\alpha}_2^{2},\lambda\tilde{\alpha}_3^{2}, \tilde{\alpha}_1^{4/3}\tilde{\alpha}_2^{2/3}\lambda^{-2/3}\log^{4\gamma/3}(W\lambda^{-1/2}\tilde{\alpha}_1),
\tilde{\alpha}_1^{3/2}\tilde{\alpha}_3^{1/2}\lambda^{-1/2}\log^{3\gamma/2}(W\lambda^{-1/2}\tilde{\alpha}_1),
\lambda^{-1}\tilde{\alpha}_1^{2}\log^{2\gamma}(W\lambda^{-1/2}\tilde{\alpha}_1)\}\notag\\
&\leq W^2 \cdot\max\{\lambda^{-1}\tilde{\alpha}_1^{2}\log^{2\gamma}(W^2\lambda^{-1}\tilde{\alpha}^2_1),\tilde{\alpha}_2^{2},\lambda\tilde{\alpha}_3^{2}\}\notag\\
&=:N^{\star}.
\end{align*}
Here the first equation follows from the fact that
\begin{align*}
\Tr(G_w H_w^{-2})=\Tr(H_w^{-1/2}G_w H_w^{-1/2} H_w^{-1})\leq \|H_w^{-1}\|_2\Tr(H_w^{-1/2}G_w H_w^{-1/2})=\|H_w^{-1}\|_2\Tr(G_w H_w^{-1}).    
\end{align*}
To summarize, for any $\delta\in (0,1)$ and any $n\geq c\max\{N^{\star}\log(d/\delta),N(\delta), N'(\delta)\}$, with probability at least $1- 3\delta $, we have
\begin{align*} 
R_{M}(\beta_{\MWLE})  &\leq c\frac{\Tr(G_w H_w^{-1}) \log \frac{d}{\delta}}{n} .
\end{align*}
\end{proof}

In the following, we prove Lemma \ref{claim1_mis} and \ref{claim2_mis}. 

\paragraph{Proof of Lemma \ref{claim1_mis}}
\begin{proof}[Proof of Lemma \ref{claim1_mis}]
For notation simplicity, we denote $g:= \nabla \ell_{n}^{w}(\beta^{\star}) - \bbE_{\bbP_S} [\nabla \ell_{n}^{w}(\beta^{\star})]$. Note that 
\begin{align*}
V &
= n \cdot \mathbb{E} [\|A(\nabla\ell_n^{w}(\beta^{\star})-\bbE[\nabla\ell_n^{w}(\beta^{\star})])\|_2^2]\notag\\
&=n\cdot\bbE[\nabla\ell_n^{w}(\beta^{\star})^{T}A^TA\nabla\ell_n^{w}(\beta^{\star})]\notag\\
&=n\cdot\bbE[\Tr(A\nabla\ell_n^{w}(\beta^{\star})\nabla\ell_n^{w}(\beta^{\star})^{T}A^T)]\notag\\
&=\Tr(A G_w A^{T}).
\end{align*}
By taking $A= H_w^{-1}$ in Assumption \ref{assm1_mis}, for any $\delta$ and any $n>N(\delta)$, we have with probability at least $1-\delta$:
\begin{align} \label{ineq:pf:concentration1_mis}
    &\|H_w^{-1}g\|_{2} \leq  c\sqrt{\frac{\Tr(G_w H_w^{-2}) \log \frac{d}{\delta}}{n}}+ W B_{1}\|H_w^{-1}\|_{2}\log^{\gamma}\left(\frac{W B_{1}\|H_w^{-1}\|_{2}}{\sqrt{\Tr(G_w H_w^{-2})}}\right) \frac{\log \frac{d}{\delta}}{n}\notag\\
    &\qquad \qquad= c\sqrt{\frac{\Tr(G_w H_w^{-2}) \log \frac{d}{\delta}}{n}}+ W B_{1}\|H_w^{-1}\|_{2}\log^{\gamma}(W\lambda^{-1/2}\tilde{\alpha}_1) \frac{\log \frac{d}{\delta}}{n} \\\label{ineq:pf:concentration3_mis}
    &\left\|\nabla^{2} \ell_{n}^{w}(\beta^{\star})-\bbE[\nabla^{2} \ell_{n}^{w}(\beta^{\star})]\right\|_2 
\leq W B_2 \sqrt{\frac{\log \frac{d}{\delta}}{n}}.
\end{align}
Let event $\tilde{A}:=\{\eqref{ineq:pf:concentration1_mis}, \eqref{ineq:pf:concentration3_mis} \text{ holds}\}$ and $\tilde{A}':=\{\text{$\ell^{w}_n(\cdot)$ has a unique local minimum, which is also global minimum}\}$. By Assumption \ref{assm1_mis} and Assumption \ref{assm3_mis}, it then holds for any $\delta$ and any $n\geq \max\{N(\delta),N'(\delta)\}$ that $\bbP(\tilde{A}\cap \tilde{A}')\geq 1-2\delta$. Under the event $\tilde{A}\cap \tilde{A}'$, we have the following Taylor expansion:
\begin{align} \label{ineq:pf:taylor1_mis}
\ell_{n}^{w}(\beta) - \ell_{n}^{w}(\beta^{\star}) 
& \stackrel{\text{by Assumption \ref{assm6}, \ref{assm2_mis}}}{\leq} (\beta - \beta^{\star})^{T} \nabla \ell_{n}^{w} (\beta^{\star}) +  \frac{1}{2} (\beta - \beta^{\star})^{T} \nabla^{2} \ell_{n}^{w} (\beta^{\star}) (\beta - \beta^{\star}) + \frac{W B_{3}}{6} \|\beta-\beta^{\star}\|_{2}^{3} \notag \\
& \stackrel{\bbE_{\bbP_S} [\nabla \ell_n^{w}(\beta^{\star})]=0}{=} (\beta - \beta^{\star})^{T} g  +  \frac{1}{2} (\beta - \beta^{\star})^{T} \nabla^{2} \ell_{n}^{w} (\beta^{\star}) (\beta - \beta^{\star}) + \frac{W B_{3}}{6} \|\beta-\beta^{\star}\|_{2}^{3} \notag \\
& \stackrel{\text{by (\ref{ineq:pf:concentration3_mis})}}{\leq} (\beta - \beta^{\star})^{T} g + \frac{1}{2} (\beta - \beta^{\star})^{T} H_w (\beta - \beta^{\star}) + W B_{2}\sqrt{\frac{\log \frac{d}{\delta}}{n}} \|\beta-\beta^{\star}\|_{2}^{2} + \frac{W B_{3}}{6} \|\beta-\beta^{\star}\|_{2}^{3}  \notag \\
& \stackrel{\Delta_{\beta}:=\beta-\beta^{\star}}{=} \Delta_{\beta}^{T}g + \frac{1}{2}\Delta_{\beta}^{T} H_w \Delta_{\beta} + W B_{2}\sqrt{\frac{\log \frac{d}{\delta}}{n}} \|\Delta_{\beta}\|_{2}^{2} + \frac{W B_{3}}{6} \|\Delta_{\beta}\|_{2}^{3} \notag \\
&= \frac{1}{2}(\Delta_{\beta}-z)^{T} H_w (\Delta_{\beta}-z) - \frac{1}{2}z^{T}H_wz + W B_{2}\sqrt{\frac{\log \frac{d}{\delta}}{n}} \|\Delta_{\beta}\|_{2}^{2} + \frac{W B_{3}}{6} \|\Delta_{\beta}\|_{2}^{3}
\end{align} 
where $z:=-H_w^{-1}g$. Similarly
\begin{align} \label{ineq:pf:taylor2_mis}
\ell_{n}^{w}(\beta) - \ell_{n}^{w}(\beta^{\star}) \geq \frac{1}{2}(\Delta_{\beta}-z)^{T} H_w (\Delta_{\beta}-z) - \frac{1}{2}z^{T}H_w z - W B_{2}\sqrt{\frac{\log \frac{d}{\delta}}{n}} \|\Delta_{\beta}\|_{2}^{2} - \frac{W B_{3}}{6} \|\Delta_{\beta}\|_{2}^{3}.
\end{align}

Notice that $\Delta_{\beta^{\star}+z} = z$, by (\ref{ineq:pf:concentration1_mis}) and (\ref{ineq:pf:taylor1_mis}), we have
\begin{align} \label{ineq:pf:taylor3_mis}
&\ell_{n}^{w}(\beta^{\star}+z)- \ell_{n}^{w}(\beta^{\star}) \notag\\
&\leq - \frac{1}{2}z^{T} H_w z+ WB_{2}\sqrt{\frac{\log \frac{d}{\delta}}{n}}\left( c\sqrt{\frac{\Tr(G_w H_w^{-2}) \log \frac{d}{\delta}}{n}}+ W B_{1}\|H_w^{-1}\|_{2}\log^{\gamma}(W\lambda^{-1/2}\tilde{\alpha}_1) \frac{\log \frac{d}{\delta}}{n}\right)^{2} \notag\\
&\quad+ \frac{W B_{3}}{6}\left( c\sqrt{\frac{\Tr(G_w H_w^{-2}) \log \frac{d}{\delta}}{n}}+ W B_{1}\|H_w^{-1}\|_{2}\log^{\gamma}(W\lambda^{-1/2}\tilde{\alpha}_1) \frac{\log \frac{d}{\delta}}{n}\right)^{3}\notag\\
&\leq - \frac{1}{2}z^{T} H_w z+2c^2W B_2\Tr(G_w H_w^{-2})(\frac{\log \frac{d}{\delta}}{n})^{1.5}+2W
^3 B^2_1B_2\|H_w^{-1}\|_{2}^2\log^{2\gamma}(W\lambda^{-1/2}\tilde{\alpha}_1) (\frac{\log \frac{d}{\delta}}{n})^{2.5}\notag\\
&\quad\quad +\frac23 c^3W B_3 \Tr(G_w H_w^{-2})^{1.5}(\frac{\log \frac{d}{\delta}}{n})^{1.5}+\frac23 W^4B^3_1B_3\|H_w^{-1}\|_{2}^3\log^{3\gamma}(W\lambda^{-1/2}\tilde{\alpha}_1)(\frac{\log \frac{d}{\delta}}{n})^{3}.
\end{align}

For any $\beta \in \bbB_{\beta^{\star}}(3c\sqrt{\frac{\Tr(G_w H_w^{-2})\log\frac{d}{\delta}}{n}})$, by (\ref{ineq:pf:taylor2_mis}), we have
\begin{align} \label{ineq:pf:taylor4_mis}
\ell_{n}^{w}(\beta)-\ell_{n}^{w}(\beta^{\star}) 
&\geq \frac{1}{2}(\Delta_{\beta}-z)^{T} H_w (\Delta_{\beta}-z) - \frac{1}{2}z^{T}H_w z \notag\\
&-9c^2W B_2\Tr(G_w H_w^{-2})(\frac{\log \frac{d}{\delta}}{n})^{1.5}-\frac92 c^3W B_3\Tr(G_w H_w^{-2})^{1.5}(\frac{\log \frac{d}{\delta}}{n})^{1.5}.
\end{align}
(\ref{ineq:pf:taylor4_mis}) - (\ref{ineq:pf:taylor3_mis}) gives
\begin{align} \label{ineq:pf:taylor5_mis}
&\ell_{n}^{w}(\beta)-\ell_{n}^{w}(\beta^{\star}+z)  \notag\\
&\geq \frac{1}{2}(\Delta_{\beta}-z)^{T} H_w (\Delta_{\beta}-z) \notag \\
&-\bigg(11c^2W B_2\Tr(G_w H_w^{-2})(\frac{\log \frac{d}{\delta}}{n})^{1.5}+\frac{31}{6}c^3W B_3\Tr(G_w H_w^{-2})^{1.5}(\frac{\log \frac{d}{\delta}}{n})^{1.5}\notag\\
&\quad\quad+ 2W
^3 B^2_1B_2\|H_w^{-1}\|_{2}^2\log^{2\gamma}(W\lambda^{-1/2}\tilde{\alpha}_1) (\frac{\log \frac{d}{\delta}}{n})^{2.5}
+\frac23 W^4B^3_1B_3\|H_w^{-1}\|_{2}^3\log^{3\gamma}(W\lambda^{-1/2}\tilde{\alpha}_1)(\frac{\log \frac{d}{\delta}}{n})^{3}
\bigg)
\end{align}
Consider the ellipsoid 
\begin{align*}
\cD:=\bigg\{\beta \in \bbR^{d} \,\bigg|\, \frac{1}{2}&(\Delta_{\beta}-z)^{T} H_w (\Delta_{\beta}-z) \notag \\
&\leq11c^2W B_2\Tr(G_w H_w^{-2})(\frac{\log \frac{d}{\delta}}{n})^{1.5}+\frac{31}{6}c^3W B_3\Tr(G_w H_w^{-2})^{1.5}(\frac{\log \frac{d}{\delta}}{n})^{1.5}\notag\\
&\quad\quad+ 2W
^3 B^2_1B_2\|H_w^{-1}\|_{2}^2\log^{2\gamma}(W\lambda^{-1/2}\tilde{\alpha}_1) (\frac{\log \frac{d}{\delta}}{n})^{2.5}\notag\\
&\quad\quad+\frac23 W^4B^3_1B_3\|H_w^{-1}\|_{2}^3\log^{3\gamma}(W\lambda^{-1/2}\tilde{\alpha}_1)(\frac{\log \frac{d}{\delta}}{n})^{3}
\bigg\}
\end{align*}
Then by (\ref{ineq:pf:taylor5_mis}), for any $\beta \in \bbB_{\beta^{\star}}(3c\sqrt{\frac{\Tr(G_w H_w^{-2})\log\frac{d}{\delta}}{n}}) \cap \cD^{C}$, we have
\begin{align} \label{ineq:pf:ellipsoid_mis}
\ell_{n}^{w}(\beta)-\ell_{n}^{w}(\beta^{\star}+z) > 0.    
\end{align}
Notice that by the definition of $\cD$, using $\lambda_{\min}^{-1}(H_w)= \|H_w^{-1}\|_{2}$, we have for any $\beta \in \cD$,
\begin{align*}
\|\Delta_{\beta}-z\|_{2}^{2} 
&\leq 22c^2\|H_w^{-1}\|_{2}W B_2\Tr(G_w H_w^{-2})(\frac{\log \frac{d}{\delta}}{n})^{1.5}+\frac{31}{3}c^3\|H_w^{-1}\|_{2}W B_3\Tr(G_w H_w^{-2})^{1.5}(\frac{\log \frac{d}{\delta}}{n})^{1.5}\notag\\
&\quad\quad
+4\|H_w^{-1}\|_{2}W^3 B^2_1B_2\|H_w^{-1}\|_{2}^2\log^{2\gamma}(W\lambda^{-1/2}\tilde{\alpha}_1) (\frac{\log \frac{d}{\delta}}{n})^{2.5}\notag\\
&\quad\quad+\frac43\|H_w^{-1}\|_{2}W^4B^3_1B_3\|H_w^{-1}\|_{2}^3\log^{3\gamma}(W\lambda^{-1/2}\tilde{\alpha}_1)(\frac{\log \frac{d}{\delta}}{n})^{3}.
\end{align*}
Thus for any $\beta \in \cD$, 
\begin{align*}
\|\Delta_{\beta}\|_{2}^{2} &\leq 2(\|\Delta_{\beta}-z\|_{2}^{2}+\|z\|_{2}^{2})    \\
& \stackrel{\text{by} (\ref{ineq:pf:concentration1_mis})}{\leq} 
44c^2\|H_w^{-1}\|_{2}W B_2\Tr(G_w H_w^{-2})(\frac{\log \frac{d}{\delta}}{n})^{1.5}+\frac{62}{3}c^3\|H_w^{-1}\|_{2}W B_3\Tr(G_w H_w^{-2})^{1.5}(\frac{\log \frac{d}{\delta}}{n})^{1.5}\notag\\
&\quad\quad
+ 8\|H_w^{-1}\|_{2}W^3 B^2_1B_2\|H_w^{-1}\|_{2}^2\log^{2\gamma}(W\lambda^{-1/2}\tilde{\alpha}_1) (\frac{\log \frac{d}{\delta}}{n})^{2.5}\notag\\
&\quad\quad +\frac83 \|H_w^{-1}\|_{2}W^4B^3_1B_3\|H_w^{-1}\|_{2}^3\log^{3\gamma}(W\lambda^{-1/2}\tilde{\alpha}_1)(\frac{\log \frac{d}{\delta}}{n})^{3}\notag\\
&\quad\quad +4c^2\Tr(G_w H_w^{-2}) \frac{ \log \frac{d}{\delta}}{n}+4 W^2 B_{1}^2\|H_w^{-1}\|_{2}^2 \log^{2\gamma}(W\lambda^{-1/2}\tilde{\alpha}_1)(\frac{\log \frac{d}{\delta}}{n})^2.
\end{align*}
To guarantee $\Tr(G_w H_w^{-2}) \frac{ \log \frac{d}{\delta}}{n}$ is the leading term, we only need $\Tr(G_w H_w^{-2}) \frac{ \log \frac{d}{\delta}}{n}$ to dominate the rest of the terms. Hence, if we further have $n\geq cN_1\log(d/\delta)$, it then holds that
\begin{align*}
 \|\Delta_{\beta}\|_{2}^{2}
 \leq 9c^2\Tr(G_w H_w^{-2}) \frac{ \log \frac{d}{\delta}}{n},
\end{align*}
i.e., $\beta \in \bbB_{\beta^{\star}}(3c\sqrt{\frac{\Tr(G_w H_w^{-2})\log \frac{d}{\delta}}{n}})$. 
Here 
\begin{align*}
N_1:=\max\bigg\{
&W^2 B_{2}^2\|H_w^{-1}\|_{2}^2,
W^2 B_{3}^2\Tr(G_w H_w^{-2})\|H_w^{-1}\|_{2}^2,
\left(\frac{W^3 B^2_1B_2\|H_w^{-1}\|_{2}^3\log^{2\gamma}(W\lambda^{-1/2}\tilde{\alpha}_1)}{\Tr(G_w H_w^{-2})}\right)^{\frac23},\notag\\
&\left(\frac{W^4B^3_1B_3\|H_w^{-1}\|_{2}^4\log^{3\gamma}(W\lambda^{-1/2}\tilde{\alpha}_1)}{\Tr(G_w H_w^{-2})}\right)^{\frac12},
\frac{ W^2 B_{1}^2\|H_w^{-1}\|_{2}^2\log^{2\gamma}(W\lambda^{-1/2}\tilde{\alpha}_1)}{\Tr(G_w H_w^{-2})}
\bigg\}.
\end{align*}
In other words, we show that $\cD \subset \bbB_{\beta^{\star}}(3c\sqrt{\frac{\Tr(G_w H_w^{-2})\log \frac{d}{\delta}}{n}})$. Recall that by (\ref{ineq:pf:ellipsoid_mis}), we know that for any $\beta \in \bbB_{\beta^{\star}}(3c\sqrt{\frac{\Tr(G_w H_w^{-2})\log \frac{d}{\delta}}{n}}) \cap \cD^{C}$, 
\begin{align*} 
\ell_{n}^{w}(\beta)-\ell_{n}^{w}(\beta^{\star}+z) > 0.
\end{align*}
Note that $\beta^{\star}+z\in\cD$.
Hence there is a local minimum of $\ell_{n}^{w}(\beta)$ in $\cD$. Under the event $\tilde{A}'$, we know that the global minimum of $\ell_{n}^{w}(\beta)$ is in $\cD$, i.e., 
\begin{align*}
    \beta_{\MWLE} \in \cD \subset \bbB_{\beta^{\star}}(3c\sqrt{\frac{\Tr(G_w H_w^{-2})\log \frac{d}{\delta}}{n}}).
\end{align*}
\end{proof}

\paragraph{Proof of Lemma \ref{claim2_mis}}
\begin{proof}[Proof of Lemma \ref{claim2_mis}]
Let $\tilde{E}:= \{ \beta_{\MWLE} \in \cD \subset \bbB_{\beta^{\star}}(\sqrt{\frac{\Tr(G_w H_w^{-2})\log \frac{d}{\delta}}{n}}) \}$. Then by the proof of Lemma \ref{claim1_mis}, for any $\delta\in (0,1)$ and any $n\geq c\max\{N_1\log(d/\delta), N(\delta), N'(\delta)\}$, we have $\bbP(\tilde{E}) \geq 1-2\delta$.

By taking $A= H_w^{-\frac{1}{2}}$ in Assumption \ref{assm1_mis}, for any $\delta\in (0,1)$ and any $n\geq N(\delta)$, with probability at least $1-\delta$, we have:
\begin{align} \label{ineq:pf:concentration2_mis}
    \|H_w^{-\frac{1}{2}}g\|_{2} &\leq c\sqrt{\frac{\Tr(G_w H_w^{-1}) \log \frac{d}{\delta}}{n}}+ W B_{1}\|H_w^{-\frac{1}{2}}\|_{2} \log^{\gamma}\left(\frac{W B_{1}\|H_w^{-\frac{1}{2}}\|_{2}}{\sqrt{\Tr(G_w H_w^{-1})}}\right) \frac{\log \frac{d}{\delta}}{n}\notag\\
    &\leq c\sqrt{\frac{\Tr(G_w H_w^{-1}) \log \frac{d}{\delta}}{n}}+ W B_{1}\|H_w^{-\frac{1}{2}}\|_{2} \log^{\gamma}\left(\frac{W B_{1}\|H_w^{-\frac{1}{2}}\|_{2}}{\sqrt{\Tr(G_w H_w^{-2})\|H_w^{-1}\|^{-1}_{2}}}\right) \frac{\log \frac{d}{\delta}}{n}\notag\\
    &= c\sqrt{\frac{\Tr(G_w H_w^{-1}) \log \frac{d}{\delta}}{n}}+ W B_{1}\|H_w^{-\frac{1}{2}}\|_{2} \log^{\gamma}(W\lambda^{-1/2}\tilde{\alpha}_1)\frac{\log \frac{d}{\delta}}{n}
\end{align}
We denote $\tilde{E}':=\{\eqref{ineq:pf:concentration2_mis} \text{ holds}\}$. Then for any $\delta$ and any $n\geq c\max\{N_1(M)\log(d/\delta), N(\delta), N'(\delta)\}$, we have $\bbP(\tilde{E}\cap \tilde{E}')\geq 1-3\delta$. 

Under $\tilde{E}\cap \tilde{E}'$, $\beta_{\MWLE} \in \cD$, i.e., 
\begin{align*}
&\frac{1}{2}(\Delta_{\beta_{\MWLE}}-z)^{T}H_w(\Delta_{\beta_{\MWLE}}-z) \\
&\leq 11c^2 W B_2\Tr(G_w H_w^{-2})(\frac{\log \frac{d}{\delta}}{n})^{1.5}+\frac{31}{6}c^3W B_3\Tr(G_w H_w^{-2})^{1.5}(\frac{\log \frac{d}{\delta}}{n})^{1.5}\notag\\
&\quad\quad+ 2W
^3 B^2_1B_2\|H_w^{-1}\|_{2}^2\log^{2\gamma}(W\lambda^{-1/2}\tilde{\alpha}_1) (\frac{\log \frac{d}{\delta}}{n})^{2.5}
+\frac23 W^4B^3_1B_3\|H_w^{-1}\|_{2}^3\log^{3\gamma}(W\lambda^{-1/2}\tilde{\alpha}_1)(\frac{\log \frac{d}{\delta}}{n})^{3}.
\end{align*}
In other words,
\begin{align} \label{ineq:pf:claim2_mis}
&\|H_w^{\frac{1}{2}}(\Delta_{\beta_{\MWLE}}-z)\|_{2}^{2}\notag\\
&\leq 22c^2 W B_2\Tr(G_w H_w^{-2})(\frac{\log \frac{d}{\delta}}{n})^{1.5}+\frac{31}{3}c^3W B_3\Tr(G_w H_w^{-2})^{1.5}(\frac{\log \frac{d}{\delta}}{n})^{1.5}\notag\\
&\quad\quad+ 4W
^3 B^2_1B_2\|H_w^{-1}\|_{2}^2\log^{2\gamma}(W\lambda^{-1/2}\tilde{\alpha}_1) (\frac{\log \frac{d}{\delta}}{n})^{2.5}
+\frac43 W^4B^3_1B_3\|H_w^{-1}\|_{2}^3\log^{3\gamma}(W\lambda^{-1/2}\tilde{\alpha}_1)(\frac{\log \frac{d}{\delta}}{n})^{3}.
\end{align}
Thus we have
\begin{align}
&\|H_w^{\frac{1}{2}}(\beta_{\MWLE}-\beta^{\star})\|_{2}^{2} \notag\\
& = \|H_w^{\frac{1}{2}}\Delta_{\beta_{\MWLE}}\|_{2}^{2} \notag \\
& = \|H_w^{\frac{1}{2}}(\Delta_{\beta_{\MWLE}}-z) +H_w^{\frac{1}{2}}z \|_{2}^{2} \notag \\
& \leq 2\|H_w^{\frac{1}{2}}(\Delta_{\beta_{\MWLE}}-z)\|_{2}^{2} + 2\|H_w^{\frac{1}{2}}z \|_{2}^{2} \notag \\
& = 2\|H_w^{\frac{1}{2}}
(\Delta_{\beta_{\MWLE}}-z))\|_{2}^{2} + 2\|H_w^{-\frac{1}{2}}g \|_{2}^{2} \notag \\
& \stackrel{\text{by} (\ref{ineq:pf:claim2_mis}) \text{and} (\ref{ineq:pf:concentration2_mis})}{\leq}4c^2 \frac{\Tr(G_w H_w^{-1}) \log \frac{d}{\delta}}{n}  \notag \\
&+
44c^2W B_2\Tr(G_w H_w^{-2})(\frac{\log \frac{d}{\delta}}{n})^{1.5}
+\frac{62}{3}c^3W B_3\Tr(G_w H_w^{-2})^{1.5}(\frac{\log \frac{d}{\delta}}{n})^{1.5}\notag\\
&\quad\quad
+ 8W^3 B^2_1B_2\|H_w^{-1}\|_{2}^2\log^{2\gamma}(W\lambda^{-1/2}\tilde{\alpha}_1) (\frac{\log \frac{d}{\delta}}{n})^{2.5}
+\frac83 W^4B^3_1B_3\|H_w^{-1}\|_{2}^3\log^{3\gamma}(W\lambda^{-1/2}\tilde{\alpha}_1)(\frac{\log \frac{d}{\delta}}{n})^{3}\notag\\
&\quad\quad 
+4W^2 B_{1}^2\|H_w^{-1}\|_{2} \log^{2\gamma}(W\lambda^{-1/2}\tilde{\alpha}_1) (\frac{\log \frac{d}{\delta}}{n})^2
\end{align}
To guarantee $ \frac{\Tr(G_w H_w^{-1}) \log \frac{d}{\delta}}{n}$ is the leading term, we only need $ \frac{\Tr(G_w H_w^{-1}) \log \frac{d}{\delta}}{n}$ to dominate the rest of the terms. Hence, if we further have $n\geq cN_2\log(d/\delta)$, we have
\begin{align*}
 \|H_w^{\frac{1}{2}}(\beta_{\MWLE}-\beta^{\star})\|_{2}^{2}\leq   9c^2\frac{\Tr(G_w H_w^{-1}) \log \frac{d}{\delta}}{n} .
\end{align*}
Here 
\begin{align*}
N_2:=\max\bigg\{
&\left(\frac{W B_2\Tr(G_w H_w^{-2})}{\Tr(G_w H_w^{-1}) }\right)^{2},
\left(\frac{W B_3\Tr(G_w H_w^{-2})^{1.5}}{\Tr(G_w H_w^{-1}) }\right)^{2},
\left(\frac{ W^3 B^2_1B_2\|H_w^{-1}\|_{2}^2\log^{2\gamma}(W\lambda^{-1/2}\tilde{\alpha}_1)}{\Tr(G_w H_w^{-1})}\right)^{\frac23},\notag\\
&\left(\frac{W^4B^3_1B_3\|H_w^{-1}\|_{2}^3\log^{3\gamma}(W\lambda^{-1/2}\tilde{\alpha}_1)}{\Tr(G_w H_w^{-1}) }\right)^{\frac12},
\frac{W^2 B_{1}^2\|H_w^{-1}\|_{2}\log^{2\gamma}(W\lambda^{-1/2}\tilde{\alpha}_1)}{\Tr(G_w H_w^{-1}) }
\bigg\}.
\end{align*}
To summarize, we show that for any $\delta$ and any $n\geq c\max\{N_1\log(d/\delta), N_2\log(d/\delta), N(\delta), N'(\delta)\}$, with probability at least $1-3\delta$, we have
\begin{align*}
 \|H_w^{\frac{1}{2}}(\beta_{\MWLE}-\beta^{\star})\|_{2}^{2}\leq   9c^2\frac{\Tr(G_w H_w^{-1}) \log \frac{d}{\delta}}{n} .
\end{align*}
\end{proof}

\subsection{Proofs for Theorem \ref{thm:mis_optimal}}

\begin{proof}[Proof of Theorem \ref{thm:mis_optimal}]
For any $W > 1$, we construct $\bbP_S(X)$, $\bbP_T(X)$, $\cM$ and $\cF$ as follows. 
We define $\bbP_T(X):=\uni(\bbB(1))$ and $\bbP_S(X):=\uni(\bbB(W^{\frac{1}{d}}))$, where $\bbB(1)$ and $\bbB(W^{\frac{1}{d}})$ are $d$-dimensional balls centered around the original with radius $1$ and $W^{\frac{1}{d}}$, respectively. For notation simplicity, we denote $Q:=\bbB(1)$ and $P:=\bbB(W^{\frac{1}{d}})$ in the following. The density ratios is then given by
\begin{align*}
w(x):=\frac{d\bbP_T(x)}{d\bbP_S(x)}=
\begin{cases}
    W & x\in Q\\
    0 & x\notin Q
\end{cases},
\end{align*}
which is upper bounded by $W$. We further have 
\begin{align*}
 \cI_S(\beta)=\bbE_{x\sim\bbP_S(X)}[xx^T]= \frac{W^{\frac{2}{d}}}{3d}I_d\succ 0,\,\, \cI_T(\beta)=\bbE_{x\sim\bbP_T(X)}[xx^T]= \frac{1}{3d}I_d\succ 0.
\end{align*}


Let $\cF:=\{f(y\,|\,x;\beta)\,|\,\beta\in\bbR^d\}$ be the linear regression class, i.e., $-\log f(y\,|\,x;\beta)=(\log 2\pi)/2+(y-x^{T}\beta)^2/2$.
We assume the true conditional distribution belongs to a class $\cM$ that is defined as
\begin{align*}
\cM:=\left\{Y\,|\,X \,\,{\sf s.t } \,\,\,p(y\,|\,x)=f(y\,|\,x;\beta^{\star}_1)\mathbf{1}_{\{x\in Q\}}+f(y\,|\,x;\beta^{\star}_2)\mathbf{1}_{\{x\in P\setminus Q \}}, \beta^{\star}_1,\beta^{\star}_2\in\bbB_{\beta_0}(B)\right\}
\end{align*}
for some $\beta_0\in\bbR^d$ and $B>0$.
We utilize the function class $\cF$ to approximate the true conditional density function, which subsequently results in model mis-specification.
In the sequel, we will show the lower bound of excess risk for any estimators under this model class $\cM$.

Fix any ground truth model $M\in\cM$, that is, we are assuming the true conditional distribution follows the form:
\begin{align*}
p(y\,|\,x)=f(y\,|\,x;\beta^{\star}_1)\mathbf{1}_{\{x\in Q\}}+f(y\,|\,x;\beta^{\star}_2)\mathbf{1}_{\{x\in P\setminus Q \}},    
\end{align*}
where $\beta^{\star}_1$ and $\beta^{\star}_2$ are arbitrarily chosen fixed points from $\bbB_{\beta_0}(B)$. Note that the model is actually well-specified on the target domain. Hence the optimal fit on the target is given by
\begin{align*}
\beta^{\star}(M)=\argmin_{\beta}\bbE_{(x,y)\sim\bbP_T(X,Y)}[\ell(x,y,\beta)]=\beta^{\star}_1.
\end{align*}

For linear regression, it is easy to verify that Assumption \ref{assm4}, \ref{assm5} and \ref{assm7} hold. Let $R_0$ and $R_1$ be the parameters chosen by Lemma \ref{lem:fisher_inform}. Then similar to the proofs of Theorem \ref{thm:well_lower}, we have
\begin{align}\label{ineq:pf:mis_optimal1}
&\inf_{\hat\beta}\sup_{M\in\cM}\bbE_{(x_i,y_i)\sim\bbP_S(X,Y)}\left[R_M(\hat\beta)\right]\notag\\
&=\inf_{\hat\beta}\sup_{\beta^{\star}_1,\beta^{\star}_2\in\bbB_{\beta_0}(B)}\bbE_{(x_i,y_i)\sim\bbP_S(X,Y)}\left[R_{\beta^{\star}_1}(\hat\beta)\right]\notag\\
&\geq \inf_{\hat\beta}\sup_{\beta^{\star}_1,\beta^{\star}_2\in\bbB_{\beta_0}(R_1)}\bbE_{(x_i,y_i)\sim\bbP_S(X,Y)}\left[R_{\beta^{\star}_1}(\hat\beta)\right]\notag\\
&\geq \inf_{\hat\beta\in\bbB_{\beta_0}(R_0)}\sup_{\beta^{\star}_1,\beta^{\star}_2\in\bbB_{\beta_0}(R_1)}\bbE_{(x_i,y_i)\sim\bbP_S(X,Y)}\left[R_{\beta^{\star}_1}(\hat\beta)\right]\notag\\
&\geq \frac14 \inf_{\hat\beta\in\bbB_{\beta_0}(R_0)}\sup_{\beta^{\star}_1,\beta^{\star}_2\in\bbB_{\beta_0}(R_1)}\bbE_{(x_i,y_i)\sim\bbP_S(X,Y)}\left[(\hat\beta-\beta^{\star}_1)^{T}\cI_T(\beta_0)(\hat\beta-\beta^{\star}_1)\right]\notag\\
&\geq \frac14 \inf_{\hat\beta\in\bbB_{\beta_0}(R_0)}\sup_{\beta^{\star}_1,\beta^{\star}_2\in C_{\beta_0}(\frac{R_1}{\sqrt{d}})}\bbE_{(x_i,y_i)\sim\bbP_S(X,Y)}\left[(\hat\beta-\beta^{\star}_1)^{T}\cI_T(\beta_0)(\hat\beta-\beta^{\star}_1)\right]\notag\\
&= \frac14 \inf_{\hat\beta\in\bbB_{\beta_0}(R_0)}\sup_{[\beta^{\star T}_1,\beta^{\star T}_2]\in C_{[\beta^{T}_0,\beta^{T}_0]}(\frac{R_1}{\sqrt{d}})}\bbE_{(x_i,y_i)\sim\bbP_S(X,Y)}\left[(\hat\beta-\beta^{\star}_1)^{T}\cI_T(\beta_0)(\hat\beta-\beta^{\star}_1)\right]
\end{align}
By Theorem 1 in \cite{gill1995applications} (multivariate van Trees inequality) with $\psi(\beta_1^{\star},\beta_2^{\star})=\beta_1^{\star}$, $C(\beta_1^{\star},\beta_2^{\star})\equiv C:=[WI_d,0]\in\bbR^{d\times 2d}$ and $B(\beta_1^{\star},\beta_2^{\star})\equiv B:=\cI_T^{-1}(\beta_0)$, we have for any estimator $\hat\beta$ and good prior density $\lambda$ that supported on $C_{[\beta^{T}_0,\beta^{T}_0]}(\frac{R_1}{\sqrt{d}})$,
\begin{align*}
\bbE_{[\beta^{\star T}_1,\beta^{\star T}_2]\sim\lambda}\bbE_{(x_i,y_i)\sim\bbP_S(X,Y)}\left[(\hat\beta-\beta^{\star}_1)^{T}\cI_T(\beta_0)(\hat\beta-\beta^{\star}_1)\right]\geq \frac{(Wd)^2}{2nWd+\tilde{\cI}(\lambda)},
\end{align*}
where
\begin{align*}
\tilde{\cI}(\lambda) = \int_{C_{[\beta^{T}_0,\beta^{T}_0]}(\frac{R_1}{\sqrt{d}})}\left(\sum_{i,j,k,\ell}B_{ij}C_{ik}C_{j\ell}\frac{\partial}{\partial\tilde{\beta}_k}\lambda(\tilde{\beta})\frac{\partial}{\partial\tilde{\beta}_{\ell}}\lambda(\tilde{\beta})\right)  \frac{1}{\lambda(\tilde{\beta})} d\tilde{\beta}.
\end{align*}
Let $\tilde{\beta}_0=[\beta_{0,1},\ldots,\beta_{0,d},\beta_{0,1},\ldots,\beta_{0,d}]^{T}$, $\tilde{\beta}=[\beta_{1},\ldots,\beta_{2d}]^{T}$ and
\begin{align*}
    f_i(x) := \frac{\pi\sqrt{d}}{4R_1}\cos\left(\frac{\pi\sqrt{d}}{2R_1}(x-\tilde{\beta}_{0,i})\right),\, i=1,\ldots,2d.
\end{align*}
We define the prior density as
\begin{align*}
\lambda(\tilde{\beta}):=
\begin{cases} 
      \Pi^{2d}_{i=1}f_i(\beta_i) & \tilde{\beta}\in C_{[\beta^{T}_0,\beta^{T}_0]}(\frac{R_1}{\sqrt{d}})\\
      0 & \tilde{\beta}\notin C_{[\beta^{T}_0,\beta^{T}_0]}(\frac{R_1}{\sqrt{d}})
\end{cases}.
\end{align*}
Then following the same argument as in the proof of Lemma \ref{lem:van_trees}, we have
\begin{align*}
 \tilde{\cI}(\lambda)=\frac{\pi^2d}{R_1^2}\Tr(BCC^{T})=\frac{\pi^2W^2 d}{R_1^2}\Tr(\cI^{-1}_T(\beta_0)).
\end{align*}
As a result, for any estimator $\hat\beta$, we have
\begin{align*}
&\bbE_{[\beta^{\star T}_1,\beta^{\star T}_2]\sim\lambda}\bbE_{(x_i,y_i)\sim\bbP_S(X,Y)}\left[(\hat\beta-\beta^{\star}_1)^{T}\cI_T(\beta_0)(\hat\beta-\beta^{\star}_1)\right]\notag\\
&\quad\geq \frac{(Wd)^2}{2nWd+\frac{\pi^2W^2 d}{R_1^2}\Tr(\cI^{-1}_T(\beta_0))},
\end{align*}
which implies
\begin{align}\label{ineq:pf:mis_optimal2}
&\sup_{[\beta^{\star T}_1,\beta^{\star T}_2]\in C_{[\beta^{T}_0,\beta^{T}_0]}(\frac{R_1}{\sqrt{d}})}\bbE_{(x_i,y_i)\sim\bbP_S(X,Y)}\left[(\hat\beta-\beta^{\star}_1)^{T}\cI_T(\beta_0)(\hat\beta-\beta^{\star}_1)\right]\notag\\
&\geq \bbE_{[\beta^{\star T}_1,\beta^{\star T}_2]\sim\lambda}\bbE_{(x_i,y_i)\sim\bbP_S(X,Y)}\left[(\hat\beta-\beta^{\star})^{T}\cI_T(\beta_0)(\hat\beta-\beta^{\star})\right]\notag\\
&\geq \frac{(Wd)^2}{2nWd+\frac{\pi^2W^2 d}{R_1}\Tr(\cI^{-1}_T(\beta_0))}.
\end{align}
Combine \eqref{ineq:pf:mis_optimal1} and \eqref{ineq:pf:mis_optimal2}, we have
\begin{align*}
 \inf_{\hat\beta}\sup_{M\in\cM}\bbE_{(x_i,y_i)\sim\bbP_S(X,Y)}\left[R_M(\hat\beta)\right]
 \geq \frac14 \cdot \frac{(Wd)^2}{2nWd+\frac{\pi^2W^2 d}{R_1}\Tr(\cI^{-1}_T(\beta_0))}
 \gtrsim \frac{Wd}{n}
\end{align*}
when $n$ is sufficiently large.

Recall that \begin{align*}
H_w(M)=\bbE_{(x,y)\sim\bbP_T(X,Y)}\left[\nabla^2\ell(x,y,\beta^{\star}(M))\right]=\bbE_{(x,y)\sim\bbP_T(X,Y)}\left[\nabla^2\ell(x,y,\beta^{\star}_1)\right]=\cI_T(\beta^{\star}_1).
\end{align*}
and by the definition of $w(x)$, we further have
\begin{align*}
G_w(M)
&=\bbE_{(x,y)\sim\bbP_S(X,Y)}\left[w(x)^2\nabla\ell(x,y,\beta^{\star}(M))\nabla\ell(x,y,\beta^{\star}(M))^{T}\right]  \\
&=\bbE_{(x,y)\sim\bbP_T(X,Y)}\left[w(x)\nabla\ell(x,y,\beta^{\star}(M))\nabla\ell(x,y,\beta^{\star}(M))^{T}\right]\\
&=W\bbE_{(x,y)\sim\bbP_T(X,Y)}\left[\nabla\ell(x,y,\beta^{\star}(M))\nabla\ell(x,y,\beta^{\star}(M))^{T}\right]\\
&=W\bbE_{(x,y)\sim\bbP_T(X,Y)}\left[\nabla\ell(x,y,\beta^{\star}_1)\nabla\ell(x,y,\beta^{\star}_1)^{T}\right]\\
&=W\cI_T(\beta^{\star}_1).
\end{align*}
Therefore $\Tr (G_w(M)H_w(M)^{-1})= Wd$, which gives the desired result. What remains is to verify that $\cM$ satisfies Assumption \ref{assm6}, \ref{assm1_mis}, \ref{assm2_mis} and \ref{assm3_mis}. Assumption \ref{assm6} is trivially satisfied. For Assumption \ref{assm1_mis} and \ref{assm2_mis}, notice that 
\begin{align*}
&\nabla\ell(x,y,\beta)
=-x(y-x^{T}\beta),\\
&\nabla^2\ell(x,y,\beta)
=xx^T,\\
&\nabla^3\ell(x,y,\beta)
=0.
\end{align*}
and
\begin{align*}
w(x):=\frac{d\bbP_T(x)}{d\bbP_S(x)}=
\begin{cases}
    W & x\in Q\\
    0 & x\notin Q
\end{cases},
\end{align*}
By the definition of $\cM$, we can write the distribution of $y$ as 
\begin{align*}
y_i=
\begin{cases}
    x_i^{T} \beta_1^{\star} + \epsilon_i & x_i\in Q\\
    x_i^{T} \beta_2^{\star} + \epsilon_i & x_i\notin Q
\end{cases},
\end{align*}
where $\epsilon_i$ is a $\cN(0,1)$ noise independent of all $x_i$'s. Therefore let $u_i := A w(x_i)\nabla\ell(x_i,y_i,\beta^{\star}(M))$,
we have 
\begin{align*}
u_i=
\begin{cases}
    -WA x_i \epsilon_i & x_i\in Q\\
    0 & x_i\notin Q
\end{cases},
\end{align*}
which indicates that $\|u_i\|$ is $\|A\|W$-subgaussian. Therefore by Lemma \ref{lem:concentration_vec}, the vector concentration in Assumption \ref{assm1_mis} is satisfied with $\gamma=0.5$, $B_1=1$. For the matrix concentration, notice that 
\begin{align*}
w(x_i)\nabla^2\ell(x_i,y_i,\beta^{\star}(M))=
\begin{cases}
    W x_i x_i^{T} & x_i\in Q\\
    0 & x_i\notin Q
\end{cases},
\end{align*}
therefore my matrix Hoeffding, $\|w(x_i)\nabla^2\ell(x_i,y_i,\beta^{\star}(M))\|_2 \leq W$, 
thus the matrix concentration in Assumption \ref{assm1_mis} is satisfied with $B_2=1$. Further more, $N(\delta)=0$ is enough for satisfying Assumption \ref{assm1_mis}.

Assumption \ref{assm2_mis} is satisfied with $B_3=0$ since 
$\nabla^3\ell(x,y,\beta)
=0$.

For Assumption \ref{assm3_mis}, we can prove that it is satisfied with $N'(\delta)= \max \{8W \log \frac{1}{\delta}, 2dW\}$. This is because, 
\begin{align*}
    \bbP(\nabla^{2}\ell_{n}^{w}(\beta) \succ 0 \text{ for all $\beta$}) &= \bbP(\frac{W}{n}\sum_{i=1}^{n}x_i x_i^{T} \bbI_{x_i \in Q } \succ 0) \\
    &\geq \bbP(\#\{x_i \in Q\} > d) \\
    &= 1- \bbP(\#\{x_i \in Q\} \leq d) \\
    &\stackrel{\text{by Chernoff bound}}{\geq}  1- \exp (-\frac{\mu}{2}(1-\frac{d}{\mu})^2) \\
    & \geq 1- \delta,
\end{align*}
where $\mu:=\frac{n}{W}$, and the last inequality hold when $n\geq N'(\delta)$. Therefore when $n\geq N'(\delta)$, with probability at least $1-\delta$, $\ell_{n}^{w}$ is strictly convex, therefore has a unique local minimum which is also the global minimum.
\end{proof}

%% file: pf_auxiliary.tex
\section{Auxiliaries}
In this section, we present several auxiliary lemmas and propositions.

\subsection{Concentration for gradient and Hessian}
The following lemma gives a generic version of Bernstein inequality for vectors.
\begin{lemma}\label{lem:concentration_vec}
Let $u, u_{1},\cdots,u_{n}$ be i.i.d. mean-zero random vectors. We denote $ V = \bbE[\|u\|^{2}_{2}]$ and 
\begin{align*}
B^{(\alpha)}_{u}:=\inf\{t>0: \bbE[\exp(\|u\|^{\alpha}/t^{\alpha})]\leq 2\},\quad \alpha\geq 1.
\end{align*}
Suppose $B^{(\alpha)}_{u}<\infty$ for some $\alpha\geq 1$. Then there exists an absolute constant $c>0$ such that for all $\delta\in (0,1)$, with probability at least $1-\delta$:
\begin{align*}
     \left\|\frac{1}{n}\sum_{i=1}^{n}u_{i}\right\|_{2} \leq c  \left(\sqrt{\frac{V\log \frac{d}{\delta}}{n}}+B^{(\alpha)}_{u}\left(\log\frac{B^{(\alpha)}_{u}}{\sqrt{V}}\right)^{1/\alpha}\frac{\log \frac{d}{\delta}}{n}\right).
\end{align*}
\end{lemma}
\begin{proof}
See Proposition 2 in \cite{koltchinskii2011nuclear} for the proof.
\end{proof}

The following proposition shows that when gradient and Hessian are bounded or sub-Gaussian (sub-exponential), Assumption \ref{assm1} is naturally satisfied.
\begin{proposition} \label{prop:concentration}
If $\|\nabla \ell (x_i,y_i,\beta^{\star})\|_2\leq b_1$ for all $i\in [n]$, then the vector concentration \eqref{assm1:ineq1} is satisfied with $B_1=b_1$ and $\gamma=0$. Alternatively, if $\|\nabla \ell (x_i,y_i,\beta^{\star})\|_2$ is $b_1$-subgaussian, then \eqref{assm1:ineq1} is satisfied with $B_1=b_1$ and $\gamma=1/2$. When $\|\nabla \ell (x_i,y_i,\beta^{\star})\|_2$ is $b_1$-subexponential, then \eqref{assm1:ineq1} is satisfied with $B_1=b_1$ and $\gamma=1$. For the Hessian concenntration, if $\|\nabla^2 \ell (x_i,y_i,\beta^{\star})\|_2\leq b_2$ for all $i\in [n]$, then \eqref{assm1:ineq2} is satisfied with $B_2=b_2$.
\end{proposition}
\begin{proof}
The vector concentration \eqref{assm1:ineq1} is a direct proposition of Lemma \ref{lem:concentration_vec}. The Hessian concentration \eqref{assm1:ineq2} is a direct consequence of matrix Hoeffiding inequality. 
\end{proof}